\documentclass[11pt]{article}
\usepackage{graphicx}
\usepackage{amsmath, amssymb, enumerate}
\usepackage{natbib} 
\usepackage{cancel}
\usepackage{algorithm}
\usepackage{subcaption}
\usepackage{algorithmic}
\usepackage{booktabs}
\usepackage{pgfplots}
\usepackage{appendix}
\usepackage{tikz}

\usepackage[applemac]{inputenc}
\usepackage[T1]{fontenc}
\usepackage{amsfonts}
\usepackage{float}
\usepackage{tikz,authblk}
\usetikzlibrary{shapes.geometric,arrows}
\usepackage{mathtools}

\usepackage{xcolor}
\newif\iftrackchanges
\trackchangesfalse  

\newcommand{\chg}[1]{\iftrackchanges{\color{blue}#1}\else#1\fi}

\usepackage{parskip}
\usepackage{multirow}
\usepackage{amsthm}
\newtheorem{theorem}{Theorem}

\newtheorem{corollary}{Corollary}
\newtheorem{definition}{Definition}

\newtheorem{remark}{Remark}
\pgfplotsset{compat=1.18}
\graphicspath{ {./images/} }
\title{Nonlinear Model-Based Sequential Decision-Making in Agriculture}
\author[1]{Sakshi Arya\thanks{*Corresponding author. Email: sxa1351@case.edu}}
\author[1,2]{Wentao Lin\thanks{Email: wentao.lin223@gmail.com (Part of this work was done when the author was a student at Case Western Reserve University.)}}

\affil[1]{Department of Mathematics, Applied Mathematics and Statistics, Case Western Reserve University, Cleveland, OH 44106, USA}
\affil[2]{CTrees, Pasadena, CA 91105, USA}

\date{}

\begin{document}

\maketitle

\begin{abstract}
\chg{Agricultural decision-making faces a dual challenge: sustaining high yields to meet global food security needs while reducing the environmental impacts of input use, including fertilizer losses (e.g., nutrient runoff and leaching) and other agrochemical applications such as herbicides, insecticides, and fungicides.} Nitrogen inputs are central to this tension. They are indispensable for crop growth yet also major drivers of greenhouse gas emissions, nutrient runoff, and escalating production costs. Addressing these intertwined pressures requires adaptive decision-support tools that are not only statistically principled but also economically sustainable and interpretable for practitioners.

We develop nonlinear model-based bandit algorithms as a framework for adaptive fertilizer management under uncertainty. Building on classical mechanistic yield-response models including Mitscherlich, Michaelis-Menten, quadratic plateau, and logistic functions, our approach links algorithmic exploration-exploitation strategies directly to interpretable biological processes such as maximum yield and nutrient efficiency. This grounding makes recommendations transparent for practitioners while supporting cost-effective and sustainable input use. Methodologically, we establish regret and sample complexity results for the well-specified nonlinear case, examine robustness under misspecification, and \chg{evaluate the proposed methods through extensive profit-oriented simulations and an offline replay case study on publicly available multi-site corn nitrogen field trials from the U.S. Midwest}. \chg{Overall, the results show that incorporating biologically meaningful mechanistic structure enables faster learning and higher profit as evidence accumulates, with flexible nonparametric baselines providing a competitive alternative in pooled and heterogeneous settings.
 Our findings illustrate how interpretable, uncertainty-aware sequential decision rules can support economically sustainable fertilizer recommendations and contribute to more efficient agricultural input use. }
\end{abstract}


\section{Introduction}
Identifying the most effective crop management practice for a given site is a recurring challenge in agriculture. While field trials and expert recommendations provide useful starting points, the actual performance of management strategies such as different fertilizer rates, planting densities, or irrigation schedules can be highly variable due to differences in soil properties, previous management, and other biological or environmental factors \citep{hochman2011emerging, tilman2002agricultural}. As a result, farmers and researchers must repeatedly decide which management actions to implement each season, often with limited prior information about their true effects. This problem exemplifies \emph{sequential decision-making under uncertainty}, where the goal is to efficiently learn and optimize management choices over time in a way that supports both productivity and sustainability.

In the statistical and machine learning literature, such problems are naturally framed as \emph{multi-armed bandit} problems, dating back to \citet{robbins1952}. They have been widely applied in areas such as precision medicine \citep{lu2021bandit}, recommendation systems \citep{li2010contextual}, and, increasingly, agriculture \citep{gautron2022reinforcement}. The term originates from the analogy of a gambler choosing between several slot machines (``one-armed bandits''), each with unknown payout rates. At each round, the decision-maker selects an \emph{action/arm} (e.g., a specific management practice) and observes its outcome (e.g., crop yield or profit). The central challenge is to balance \emph{exploration} (trying new or uncertain options to learn about their effects) with \emph{exploitation} (repeating choices that have performed well so far).  Bandit algorithms provide principled, data-driven strategies to manage this tradeoff, making them well-suited to resource optimization tasks in agriculture and environmental management. \chg{A standard way to evaluate such algorithms is through minimizing \emph{regret} (see Table~\ref{tab:key_concepts} for formal definition), which measures the cumulative loss incurred while learning compared to an ideal benchmark that always selects the best fixed action in hindsight.} 

Most existing bandit algorithms rely on linear/GLM structure or generic black-box models, which can miss the mechanistic nonlinear dose-response patterns common in agronomy.
We develop a family of bandit algorithms that embeds agronomy-standard nonlinear response models directly in the decision loop, yielding interpretable recommendations and improved sample efficiency in the small-sample regimes typical of field experimentation, where decisions carry meaningful economic and environmental stakes.

Nonlinear models provide parsimonious and interpretable representations of domain-specific processes. Unlike generic polynomial or linear models, their parameters often map directly to biological quantities, such as maximum yield potential or nutrient efficiency, making the models both scientifically meaningful and statistically efficient. In agriculture, such models \citep{miguez2018nonlinear} capture complex responses like crop yield to fertilizer, temperature effects on growth, or pest dynamics, where the relationship follows saturating, logistic, or exponential patterns rather than a simple linear trend. Their adoption is motivated not only by statistical accuracy but also by the need for transparent recommendations in sustainable resource management 

\subsection{Background and related work}
Multi-armed bandit algorithms \citep{berry1985bandit, lattimore2018bandit} provide a core statistical framework for sequential decision-making under uncertainty, where the key challenge is to balance exploration and exploitation. In the original bandit formulation, each arm was associated with an unknown fixed reward distribution, and repeated pulls of the same arm were assumed to yield independent and identically distributed (i.i.d.) rewards. Subsequent work extended this framework to incorporate contextual information, leading to the development of contextual bandit models \citep{slivkins2014contextual, chu2011contextual}. In these settings, regression models such as linear \citep{abbasi2011improved}, generalized linear models (GLMs) \citep{li2010contextual, filippi2010parametric}, and more flexible nonparametric approaches, including kernels, nearest neighbors, and neural networks \citep{yang2002randomized, rigollet2010nonparametric, srinivas2010gaussian, valko2013finite}, enable adaptive decision-making tailored to evolving arm/covariate information.

Classical nonlinear regression models, such as Mitscherlich \citep{dhanoa2022overview}, Michaelis-Menten \citep{lopez2000generalized}, quadratic plateau models \citep{belanger2000comparison}, and logistic response functions \citep{sepaskhah2011logistic}, are routinely used to capture crop yield response to fertilizer and other management variables. These models offer interpretable parameters tied to underlying biological processes, which is important for transparent and actionable recommendations \citep{de2008soybean, miguez2018nonlinear}.

Machine learning approaches, including random forests, neural networks, and regression trees, have become increasingly popular for crop yield prediction and resource optimization \citep{shahhosseini2019maize, khaki2020cnn, khairunniza2014application, jabed2024crop}. While powerful in large-data settings, these methods are typically supervised and static, lacking the sequential and adaptive capabilities needed for decision-making under uncertainty. Concerns also remain about their interpretability and suitability for low-data environments \citep{dobermann2022responsible}.

Within agriculture, adaptive experimentation and real-time learning strategies have emerged as promising approaches for site-specific management and resource efficiency \citep{gautron2022reinforcement}. Recent work has explored the use of bandit algorithms \citep{saikai2020machine,saikai2018multi,huang2025precision} for adaptive plot selection and management optimization, but most existing studies focus on linear models or generic black-box methods, seldom leveraging domain knowledge of nonlinear response behavior, especially valuable for smallholders and in data-scarce scenarios \citep{smith2018getting}. \chg{
Alongside bandit and reinforcement-learning approaches, a growing precision-agriculture literature uses Gaussian-process surrogate modeling, Bayesian optimization, and Bayesian experimental design to adaptively select treatment levels and sampling schemes in on-farm trials, often with an explicit focus on estimating the economic optimum nitrogen rate (EONR). For example, \citet{ng2022bayesian} develop Bayesian optimal dynamic sampling procedures for on-farm experimentation, and \citet{matavel2025bayesian} propose Bayesian-optimized experimental designs for EONR estimation using a model-averaging approach. Related work develops fully Bayesian economically optimal designs in spatial and multi-year settings \citep{poursina2024fully}, uses Gaussian-process modeling across multiple on-farm experiments to quantify uncertainty and profitability of fertilizer strategies \citep{Mia02072024}, and advances multi-environment evaluation of nitrogen recommendation tools, highlighting the need for decision rules that can update across seasons \citep{Abdipourchenarestansofla2025,Ransom2020tools}. Bayesian optimization has also been used to automate calibration of mechanistic crop models, illustrating its role in decision-support pipelines \citep{Akhavizadegan2021}.
Our work is complementary to these Bayesian optimization and design approaches. While they primarily target efficient experimentation and estimation (e.g., learning EONR with minimal trials) via a learned surrogate and an acquisition rule, we study bandit decision rules that explicitly balance exploration and exploitation to optimize \emph{cumulative} performance (profit and/or yield) over repeated decisions. Moreover, surrogate/acquisition formulations can be less directly interpretable for agronomic decision-making, whereas our policies are built around mechanistic dose--response models whose parameters map to familiar agronomic quantities (e.g., plateau level, response rate), making the learned recommendations more transparent in data-limited settings.}

It is well-known that model-based bandit and reinforcement learning approaches can be far more sample-efficient than model-free methods when the underlying reward mechanisms are known \citep{osband2014model, sun2019model}. However, such approaches have received limited attention in agricultural decision-making from a precision-agriculture and sustainability perspective, where established nonlinear models are often available but underutilized.
Our work addresses this gap by integrating domain-driven nonlinear models directly into the bandit framework. This enables adaptive, interpretable, and sample-efficient resource management, advancing the practical and theoretical toolkit for sustainable and cost-effective fertilizer use.

\chg{Our primary objective is to adapt nitrogen-rate recommendations over repeated seasons by balancing exploration and exploitation to maximize overall profit while learning the profit-maximizing rate (EONR / maximum return to nitrogen) over a feasible fertilizer-rate grid.
\paragraph{Our Contributions.}
This paper makes the following contributions:
\begin{itemize}
\item \textbf{Mechanistic nonlinear bandits for fertilizer-rate decisions:}
We formulate nitrogen-rate selection as a sequential decision problem in which the mean reward is modeled using agronomy-standard mechanistic nonlinear dose--response families (Mitscherlich, quadratic plateau, Michaelis--Menten, and logistic). {To our knowledge, this is the first bandit-based framework in agronomic decision-making that provides a unified, interpretable template for {mechanistic nonlinear} response models rather than restricting to linear/GLM rewards or fully model-free black-box learners, and it can be adapted to different agronomic objectives by redefining the reward (e.g., profit, yield, or other management utilities).}

    \item \textbf{Nonlinear model-based $\epsilon$-greedy, nonlinear-UCB, and \texttt{ViOlin} baselines:}
    We present a unified framework and pseudocode templates for three nonlinear model-based strategies, $\epsilon$-greedy, UCB, and \texttt{ViOlin}, that can be instantiated with the above mechanistic response models, providing an interpretable alternative to linear and model-free baselines for data-limited agronomic decision-making.


    \item \textbf{Profit-oriented empirical evaluation in simulated and real trials:}
    Under the profit objective (learning the EONR), we conduct an extensive simulation study under both well-specified and misspecified regimes, and we additionally provide a reproducible real-data case study using multi-site corn nitrogen trials from \citet{ransom_data} via an offline replay protocol. These experiments quantify when mechanistic nonlinear modeling yields large gains and how performance degrades under model mismatch, relative to linear methods (e.g., LinUCB) and fully model-free methods (e.g., kNN-UCB), particularly in low-sample regimes.

\item \textbf{Theory-to-practice connection via available complexity and sample-efficiency results:}
    We connect the nonlinear model classes considered here to existing sequential learning theory by presenting a sequential complexity bound for bounded function classes (with an illustrative visualization for the four nonlinear families studied) and and by summarizing available sample-efficiency results for curvature-guided methods such as \texttt{ViOlin} from \citet{dong2021provable}, discussing its implications for data-limited agricultural experimentation.
\end{itemize}}
We begin by setting up the problem in Section \ref{sec: problem_setup}. In Section \ref{sec: algorithms}, we lay out the proposed framework for nonlinear model-based bandits, in which we study three model-based algorithms: $\epsilon$-greedy, Upper Confidence Bound (UCB), and the \texttt{ViOlin} algorithm. Section \ref{sec: theory} presents known results for parametric bandits in both well-specified and misspecified settings, and derives sample complexity bounds for the specific nonlinear models considered by the \texttt{ViOlin} algorithm. In Section \ref{sec: simulation}, we then conduct extensive simulation experiments emulating real-world yield production and fertilizer optimization scenarios, with the goal of resource-efficient and cost-effective management in sample-limited settings. Then, in Section \ref{sec:realdata} we conduct a reproducible real-data case study based on multi-site corn nitrogen trials in the US midwest to illustrate the advantage of the proposed methodology. Finally, Section \ref{sec:conclusion} concludes the paper.

\section{Problem Setup: Adaptive Input Decisions with Nonlinear Models}
\label{sec: problem_setup}

Formally, we consider a sequential decision-making scenario with a time horizon of $T$ rounds. Each round can represent, for example, a growing season or field trial, where decisions are made at planting and outcomes observed at harvest. At each round $t = 1, \dots, T$, the decision-maker (agent) chooses an action (input level) $x_t \in \mathcal{X}$ from a feasible set $\mathcal{X}$, such as nitrogen application rates. After choosing action $x_t$, the agent observes a noisy reward $y_t$, representing yield or profit, modeled as:
\begin{align}
    y_t = f(x_t; \beta^*) + \chg{\eta_t}, \label{eq: model}
\end{align}
where $f: \mathcal{X} \to \mathbb{R}$ is a known, structured nonlinear function parameterized by an unknown vector $\beta^* \in \mathbb{R}^p$, and $\eta_t$ is a mean-zero noise term. We assume $\eta_t$ is \emph{sub-Gaussian}, i.e., there exists $\sigma > 0$ such that for all $\lambda \in \mathbb{R}$,
\begin{align}
    \mathbb{E}\left[e^{\lambda \eta_t}\right] \leq \exp\left(\frac{\lambda^2 \sigma^2}{2}\right), \label{eq: subGaussian}
\end{align}
where $\mathbb{E}(\cdot)$ denotes the expectated value with respect to the random noise.
This assumption reflects the variability of agricultural experiments driven by weather, soil heterogeneity, and biological processes, while ruling out extreme outliers that would make learning impractical.

The decision-maker seeks a policy/algorithm $\pi$ that selects actions based on historical observations to maximize cumulative expected reward. Performance can equivalently be measured through the cumulative regret, which quantifies the total loss relative to always applying the best treatment:
\begin{align}
    R_T(\pi) = \sum_{t=1}^T \left\{ f(x^*; \beta^*) - f(x_t; \beta^*) \right\},
\end{align}
where $x^* = \arg\max_{x \in \mathcal{X}} f(x; \beta^*)$ denotes the optimal treatment/arm. In agricultural terms, regret measures the cumulative yield or profit lost due to not choosing the best possible input at each season. The objective is to devise algorithms achieving \emph{sublinear regret}, meaning $R_T(\pi) = o_P(T)$, so that the average regret per round diminishes over time. Such a property guarantees that the policy rapidly identifies near-optimal input levels, minimizing wasted resources and experimental costs. Other performance metrics more suitable for risk-averse decision-making have also been proposed, such as the Conditional Value-at-Risk (CVaR) of the regret, which emphasizes control of the worst-case outcomes in the tail of the regret distribution \citep{baudry2021optimal, wang2023near}.
\chg{\begin{remark}[Scope]
For clarity, we study a stationary setting where the underlying mean response does not change across rounds, so the optimal benchmark rate $x^*$ is fixed. Accounting for season-to-season shifts (e.g., weather or soil changes) would require a nonstationary model and is left for future work.
\end{remark}}
\chg{
\begin{table}[H]
\centering
\caption{\chg{Key concepts and notation for sequential fertilizer-rate decision-making.}}
\label{tab:key_concepts}
\begin{tabular}{p{0.30\textwidth} p{0.50\textwidth}}
\toprule
\textbf{Term} & \textbf{Meaning in this paper} \\
\midrule
Round ($t=1,\dots,T$) & One decision opportunity (e.g., a season, site-year, or trial iteration) in which a nitrogen rate is selected and an outcome is observed. \\
Action / arm ($x_t \in \mathcal{X}$) & The nitrogen fertilizer rate applied at round $t$; $\mathcal{X}$ is a feasible set (often a discrete grid). \\
Reward ($y_t$) & Observed outcome from choosing $x_t$ (yield or profit, depending on the objective). \\
Policy / algorithm ($\pi$) & The rule the decision-maker follows each round: it looks at all past fertilizer choices and outcomes
$\{(x_i,y_i)\}_{i=1}^{t-1}$ and then recommends the next nitrogen rate $x_t$. \\
Bandit (partial) feedback 
& Only the outcome corresponding to the chosen rate $x_t$ is observed at each round (not outcomes for all rates). \\
True mean reward ($f(x;\beta^*)$) & Expected outcome at rate $x$ under the unknown true parameter $\beta^*$. \\
Estimated mean reward ($f(x;\hat\beta_t)$) & Plug-in prediction used by the algorithm at round $t$ based on data collected so far. \\
Exploration vs.\ exploitation & Trying uncertain rates to learn (exploration) versus choosing the currently best-performing rate (exploitation). \\
Economic objective ($\Pi(x)$) & Profit at rate $x$, typically $\Pi(x)=p_y Y(x)-p_x x$, with crop price $p_y$ and input cost $p_x$. \\
Optimal rate ($x^*$) & The best fixed rate in $\mathcal{X}$ under the objective (e.g., the profit-maximizing rate / EONR). \\
Warm start ($n_0$) 
& Initial rounds of forced exploration (uniformly random choices) to stabilize early model fitting before using model-based rules. \\
Per-round regret ($r_t$) 
& Loss relative to the benchmark: $r_t=\Pi(x^*)-\Pi(x_t)$ in simulations; in offline replay, $x^*$ is the within-round oracle based on observed treatment means. \\
Cumulative regret ($R_t$) and average regret ($R_t/t$) 
& $R_t=\sum_{s=1}^t r_s$ summarizes total loss; $R_t/t$ summarizes average loss per decision (used for ``decreasing'' curves). \\
\bottomrule
\end{tabular}
\end{table}}
\chg{\subsection{From yield to profit: an economic objective.}\label{sec:profit_objective}
In fertilizer management, maximizing yield alone is not necessarily the right decision objective. Most yield--nitrogen response curves exhibit diminishing returns: yield increases with fertilizer up to a point and then plateaus. Because fertilizer is a costly input, applying nitrogen beyond the economically efficient range can increase cost without proportional yield gains. For this reason, we focus on \emph{profit} (net return) rather than yield as the primary optimization target.

Let $p_y$ denote the price per unit yield (e.g., \$/bushel) and let $p_x$ denote the price per unit nitrogen (e.g., \$/lb N). Suppose the mean yield response is modeled by a parametric function $f(x;\theta)$, where $x$ is the nitrogen rate and $\theta$ denotes agronomic parameters. We define the profit function
\begin{equation}
    \Pi(x) = p_y\, f(x;\theta) - p_x\, x. \label{eq:profit_generic}
\end{equation}
The economically optimal fertilizer rate is then
\begin{equation}
    x^\star (\theta) = \arg\max_{x \in \mathcal{X}} \Pi(x)
    = \arg\max_{x \in \mathcal{X}} \left[p_y f(x;\theta) - p_x x\right], \label{eq:profit_optimal}
\end{equation}
where $\mathcal{X}$ is the feasible set of fertilizer rates (e.g., a grid from 0 to 250 lb N/ac). Depending on the form of $f$, this maximization can be carried out in closed form or numerically. In our experiments, we use closed-form expressions for $x^\star$ for the yield-response models considered (summarized in Table~\ref{tab:profit_optima}), which makes clear how the optimal decision depends jointly on agronomic parameters $\theta$ and economic parameters $(p_y,p_x)$ that must be learned from data. Accordingly, we define (expected) profit regret as $\Pi(x^\star)-\Pi(x_t)$ and evaluate cumulative regret $\sum_{t=1}^T\{\Pi(x^\star)-\Pi(x_t)\}$, where $x_t$ is the fertilizer rate selected by the algorithm at time $t$. For the nonlinear yield functions studied in our simulations, closed-form expressions for $x^\star$ are summarized in Table~\ref{tab:profit_optima}.}

\chg{\paragraph{Why bandits?}
Although $x^\star(\theta)$ admits closed-form expressions (Table~\ref{tab:profit_optima}), $\theta$ is unknown and must be learned online from limited data; the bandit algorithms couple plug-in optimization with explicit exploration to avoid premature convergence under early estimation error.}

To illustrate the framework, we focus on fertilizer-yield relationships and investigate four classical nonlinear models widely used in agricultural research:
\begin{enumerate}
    \item \textbf{Mitscherlich Model}: $f(x; A, b, d) = d + A \left(1 - e^{-b x}\right)$, capturing saturating yield response to fertilizer inputs.
    \item \textbf{Michaelis-Menten Model}: $f(x; a, b, d) = d + \frac{a x}{b + x}$, describing nutrient uptake or growth responses with diminishing returns.
    \item \textbf{Quadratic Plateau Model}: \[
     f(x; a, b, c, x_0) = 
        \begin{cases}
            a + b x + c x^2 & \text{if } x \leq x_0 \\
            a + b x_0 + c x_0^2 & \text{if } x > x_0
        \end{cases},\] modeling responses that rise and then stabilize at a plateau, common in fertilizer trials.
    \item \textbf{Logistic Dose-Response Model}: $f(x; A, B, C, d) = d + \frac{A}{1 + \exp\left(-B(x - C)\right)}$, often used to capture threshold or inflection-point behavior.
\end{enumerate}
These models are parsimonious, interpretable, and directly tied to biological processes. For example, in the Mitscherlich model, $A$ represents the maximum additional yield achievable with fertilizer, $b$ reflects the efficiency of nutrient uptake, and $d$ is the baseline yield. For example, a larger value of $b$ indicates that yield saturates more rapidly with increasing input, whereas a smaller $b$ corresponds to a more gradual response.  Such parameters are routinely used in agronomic practice and make the resulting recommendations not only statistically grounded but also transparent and actionable.

Although we emphasize nitrogen fertilizer management as a motivating example, the framework generalizes to a wide range of agricultural and environmental problems where nonlinear dose-response relationships are well established, including irrigation efficiency, biomass accumulation, and nutrient loss in response to rainfall \citep{miguez2018nonlinear}. We provide detailed parameterizations and simulation examples in Section~\ref{sec: simulation}.

\section{Methods: Algorithms for nonlinear model-based bandits} \label{sec: algorithms}
In agricultural trials, each season offers only limited opportunities for experimentation.  Algorithms must therefore use past data efficiently while remaining robust to uncertainty. 
In this section, we describe a family of strategies for nonlinear model-based bandit problems, where the expected reward is modeled as a nonlinear function of the chosen input.
  We first provide a general algorithmic template, followed by detailed descriptions of three model-based strategies, namely, $\epsilon$-greedy, Upper Confidence Bound (UCB), and \texttt{ViOlin}, adapted to classical agronomic response models. Note that all the three algorithms have been studied in various bandit problems before \citep{auer2010ucb,arya2020randomized,dong2021provable}, here we specialize them to mechanistic nonlinear settings. 

\subsection{General framework}
Recall, $f(x; \theta)$ denotes a known nonlinear model (e.g., Michaelis-Menten, logistic, quadratic) parameterized by $\theta$, where $x \in \mathcal{X}$ is the action (e.g., N-rate) and the expected reward is $\mathbb{E}[Y \mid x] = f(x; \theta)$. At each time $t$, the agent selects an action $x_t$, then receives a reward $y_t$.

\begin{algorithm}[H]
\caption{General Nonlinear Model-based Bandit Framework}
\label{alg: general_layout_online_Nonlinear}
\begin{algorithmic}[1]
\STATE \textbf{Input:} Action set $\mathcal{X}$,  nonlinear model class $f(\cdot; \theta)$, time horizon: $T$
\STATE \textbf{Initialize: } Pull each arm at least once and receive corresponding rewards until round $t_0$
\FOR{$t = t_0 + 1$ to $T$}
  \STATE Estimate model parameters $\hat{\theta}_t$ using past data $\{(x_s, y_s): s < t\}$
  \STATE Choose action $x_t$ using a strategy/policy $\pi$ utilizing $f(x; \hat{\theta}_t)$
  \STATE Observe reward $y_t$
  \STATE Update dataset with $(x_t, y_t)$
\ENDFOR
\end{algorithmic}
\end{algorithm}
\chg{\paragraph{Illustrative example: }
To make Algorithm~\ref{alg: general_layout_online_Nonlinear} concrete, consider a small pilot dataset with $5$ nitrogen rates (e.g., $\mathcal{X} = \{0,50,100,150,200\}$) and $2$ independent replicated plots per rate, giving $t_0=10$ observed pairs $\{(x_i,y_i)\}_{i=1}^{10}$.
At the first adaptive decision round (i.e., $t=11$), Line~4 fits the nonlinear model using \emph{all} available observations, so the nonlinear least-squares objective uses $10$ data points and returns $\hat{\theta}_{11}$.
Line~5 then selects the next rate $x_{11}$ from $\mathcal{X}$ using the chosen policy/bandit algorithm applied to $f(\cdot;\hat{\theta}_{11})$.
After observing $y_{11}$ in Line~6, Line~7 appends $(x_{11},y_{11})$, so the next update uses $11$ observations.
In general, at round $t$, $\hat{\theta}_t$ is estimated from the full history $\{(x_s,y_s)\}_{s=1}^{t-1}$ (i.e., $t-1$ observations), after which the bandit policy (algorithm) $\pi$ selects the next action $x_t$ based on $f(\cdot;\hat{\theta}_t)$. 

In our implementation, $\hat{\theta}_t$ is computed via nonlinear least squares (e.g., using \texttt{curve\_fit} in \texttt{scipy.optimize}) by minimizing a squared-error objective of the form
\[
\hat{\theta}_t = \arg\min_{\theta}\sum_{s<t}\bigl(y_s - f(x_s;\theta)\bigr)^2.
\]
 After observing $y_t$, the pair $(x_t,y_t)$ is appended to the dataset and used in the next update.}

\begin{remark}[Bandit (partial) feedback]
Unlike classical supervised learning or factorial field experiments that provide outcomes for every treatment, bandit settings reveal only the reward for the action actually chosen at each round. This \emph{partial feedback} captures the reality of sequential agronomic decision-making where only tested nitrogen rates yield observable outcomes. \chg{As a consequence, even when a closed-form optimizer $x^\star(\theta)$ exists for a parametric response family, the parameters $\theta$ cannot be learned reliably without deliberate exploration, since unchosen rates provide no counterfactual information.}
\end{remark}
Next, we describe three model-based algorithms that follow the general structure of Algorithm \ref{alg: general_layout_online_Nonlinear} but differ in how they balance the exploration-exploitation trade-off.
\subsection{Epsilon-Greedy}
The $\epsilon$-greedy algorithm (Algorithm \ref{alg:eps_greedy}) \citep{langford2008epoch, yang2002randomized} is one of the simplest approaches yet effective for balancing exploration and exploitation. At each round, the current nonlinear model is fitted using past observations, and the input that maximizes predicted profit is selected with high probability $1-\epsilon_t$ (\textit{exploitation}). With small probability $\epsilon_t$, the algorithm explores by selecting a random alternative from the feasible input grid (\textit{exploration}).  This mechanism ensures continued learning, preventing the algorithm from becoming stuck on suboptimal choices. In agricultural contexts, this method reflects the intuitive practice of trying occasional alternative fertilizer rates even when one rate appears best.

Several variants of $\epsilon$-greedy exist; the most popular is the \textit{annealed} or \textit{decaying} $\epsilon$-greedy, where the exploration probability $\epsilon_t$ is a non-increasing sequence that tends to zero as $t \rightarrow \infty$. The intuition is that, as more data are collected and the model becomes more certain, the algorithm increasingly favors exploitation over exploration. \chg{Note that the user specifies an exploration \emph{schedule} $\{\epsilon_t\}_{t\ge1}$ (typically through an initial value and a decay rule), which controls the exploration--exploitation trade-off and directly affects cumulative regret.} 

Epsilon-greedy is simple, robust, and easy to implement. It is particularly useful when the model is misspecified or when simple, interpretable algorithms are desired.

\begin{algorithm}[H]
\caption{Nonlinear model-based Epsilon-Greedy}
\label{alg:eps_greedy}
\begin{algorithmic}[1]
\STATE \textbf{Input:} Action set $\mathcal{X}$, exploration schedule $\{\epsilon_t\}_{t\ge1}$ (e.g., $\epsilon_t=\min\{1,\epsilon_0 t^{-\gamma}\}$), non-linear model class $f(\cdot; \theta)$, horizon $T$
\STATE \textbf{Initialize: } Random pull arms and receive corresponding rewards until round $t_0$
\FOR{$t = t_0 + 1$ to $T$}
  \STATE Estimate $\hat{\theta}_t$ from previous data: $\{x_1, y_1, \hdots, x_{t-1}, y_{t-1}\}$
  \STATE Select action $x_t$ according to:
    \[
      x_t =
      \begin{cases}
        \text{sample uniformly from } \mathcal{X} & \text{with probability } \epsilon_t \\
        \arg\max_{x \in \mathcal{X}} f(x; \hat{\theta}_t) & \text{with probability } 1-\epsilon_t
      \end{cases}
    \]
  \STATE Observe $y_t$, update data
\ENDFOR
\end{algorithmic}
\end{algorithm}

If $\mathcal{X}$ is a finite set of $K$ elements, ``sample uniformly from $\mathcal{X}$'' means selecting an element at random with equal probability, so that each arm is chosen with probability $1/K$. In the $\epsilon$-greedy algorithm, during the exploration step, each arm is selected with probability $\epsilon/K$. If $\mathcal{X}$ is a continuous interval or region, this means sampling from the continuous uniform distribution over $\mathcal{X}$, i.e., $x_t \sim \mathrm{Unif}(\mathcal{X})$.

\subsection{Nonlinear model-based UCB}
\paragraph{Overview:} The UCB (Upper Confidence Bound) algorithm \citep{auer2010ucb, chu2011contextual, zhou2020neural} embodies the principle of ``optimism in the face of uncertainty'' by selecting actions that maximize an upper confidence bound on the expected reward.  Here, the nonlinear model is repeatedly fitted, and for each input level the algorithm constructs an upper confidence bound on the predicted reward. The next action is chosen to maximize this upper bound. This mechanism naturally prioritizes inputs that are either promising in terms of mean yield or remain highly uncertain, thereby encouraging exploration in scientifically justified directions. Applied to fertilizer management, UCB corresponds to testing input levels where either predicted yield is high or parameter uncertainty remains large.

In model-based nonlinear regression, the uncertainty $\mathrm{Unc}_t(x)$ is often based on the variability of the estimated reward due to finite data, as quantified via the standard error of the model prediction. Specifically, the uncertainty term is given by
\begin{equation} \label{eq:uncertainty}
\mathrm{Unc}_t(x) = \sqrt{ \nabla_\theta f(x; \hat{\theta}_t)^\top\, \widehat{\mathrm{Cov}}(\hat{\theta}_t)\, \nabla_\theta f(x; \hat{\theta}_t) },
\end{equation}
where $\nabla_\theta f(x; \hat{\theta}_t)$ denotes the gradient of the reward model with respect to the parameters, evaluated at the current estimate $\hat{\theta}_t$, and $\widehat{\mathrm{Cov}}(\hat{\theta}_t)$ is the estimated covariance matrix of $\hat{\theta}_t$. \chg{We emphasize that $\mathrm{Unc}_t(x)$ is a first-order error-propagation (delta-method) proxy and is not claimed to be a certified nonasymptotic confidence radius.
This is nonetheless appropriate for our algorithmic use: the UCB rule only needs a {data-adaptive ranking of uncertainty across candidate rates} to avoid premature commitment when $T$ is small.
In early rounds, the nonlinear fit is typically weakly identified and $\widehat{\mathrm{Cov}}(\hat\theta_t)$ is larger (and can be ill-conditioned), which increases $\mathrm{Unc}_t(x)$ and naturally promotes additional exploration.
To stabilize this regime we (i) warm-start with uniform exploration for the first $n_0$ observations and (ii) compute $\widehat{\mathrm{Cov}}(\hat\theta_t)$ using ridge/regularized inversion of the observed information, preventing numerical degeneracy in very small samples.
We use~\eqref{eq:uncertainty} as a practical and interpretable exploration score for data-limited agronomic settings, and evaluate its behavior empirically in our simulations and offline real-data replay experiments.}

The UCB action selection rule then takes the form
\begin{equation} \label{eq:ucb}
x_t = \arg\max_{x \in \mathcal{X}}\, f(x; \hat{\theta}_t) + \alpha \cdot \mathrm{Unc}_t(x)
\end{equation}
where $\alpha > 0$ is a user-specified confidence parameter that controls the degree of exploration. 

UCB algorithms provide strong theoretical guarantees for balancing exploration and exploitation, and are particularly effective when reliable, model-based estimates of prediction uncertainty are available (see, e.g., \citep{zhou2020neural}).  It is particularly advantageous in settings where data collection is expensive or risky, as it prioritizes actions that could yield high rewards or that have not been thoroughly explored. Model-based UCB (Algorithm \ref{alg:ucb}) thus combines statistical rigor with the agronomic interpretability of nonlinear functions, making it a strong candidate for adaptive on-farm experimentation.

\begin{algorithm}[H]
\caption{Nonlinear model-based UCB}
\label{alg: ucbsketch}
\begin{algorithmic}[1]
\STATE \textbf{Input:} Action set $\mathcal{X}$, confidence parameter $\alpha > 0$
\FOR{$t = 1$ to $T$}
  \STATE Estimate $\hat{\theta}_t$ from the previous data 
  \FOR{each $x \in \mathcal{X}$}
    \STATE Compute predicted reward $f(x, z_t; \hat{\theta}_t)$ and uncertainty $\mathrm{Unc}_t(x)$ as defined in \eqref{eq:uncertainty} 
    \STATE Calculate $\text{UCB}_t(x) = f(x, z_t; \hat{\theta}_t) + \alpha \cdot \mathrm{Unc}_t(x)$
  \ENDFOR
  \STATE Choose $x_t = \arg\max_{x \in \mathcal{X}} \text{UCB}_t(x)$
  \STATE Observe $y_t$, update data
\ENDFOR
\end{algorithmic}
\end{algorithm}

\subsection{\texttt{ViOlin} (Virtual Ascent with Online Model Learner)}
The \texttt{ViOlin} algorithm (Algorithm~\ref{alg:ViOlin})~\citep{dong2021provable} is a model-based bandit strategy designed for efficient learning in nonlinear settings. Unlike $\epsilon$-greedy, which explores actions uniformly at random, or UCB, which prioritizes actions with high uncertainty, \texttt{ViOlin} is a greedy method that leverages both the current model estimate and local geometric information about the reward surface (i.e., gradient and curvature). \chg{At each round, the algorithm selects the action predicted to be best by the current estimated model, while using local geometric information (slope and curvature) of the \emph{fitted} reward surface to guide the search. In our implementation, this geometry enters through the action-selection score in Algorithm~\ref{alg:ViOlin}, and model parameters are updated from accumulated noisy rewards via standard parametric fitting. This geometry-guided greedy strategy is motivated by the ViOlin framework of \citet{dong2021provable}, which provides sample-efficiency guarantees when local gradient/curvature (Hessian) information about the reward function can be leveraged; here we use analytic derivatives of the fitted mechanistic model as a practical proxy under noisy observations.}
\chg{\begin{algorithm}[H]
\caption{\texttt{ViOlin}: Virtual Ascent with Online Model Learner \citep{dong2021provable}}
\label{alg:ViOlin}
\begin{algorithmic}[1]
\STATE \textbf{Input:} Model class $\mathcal{F}=\{f(\cdot;\theta):\theta\in\Theta\}$, action set $\mathcal{X}$, initial guess $\hat{\theta}_1$,
learner/estimator $\mathcal{O}$, total rounds $T$, curvature weights $\kappa_1,\kappa_2\ge 0$, minimum fit size $m$
\FOR{$t = 1$ to $T$}
    \IF{$t \le m$}
        \STATE Choose $x_t$ uniformly at random from $\mathcal{X}$
        \STATE Observe noisy reward $y_t$
    \ELSE
        \STATE Fit/update $\hat{\theta}_t$ using $\{(x_s,y_s)\}_{s=1}^{t-1}$ via $\mathcal{O}$
        \STATE \textbf{Action selection:} for each $x\in\mathcal{X}$ compute
        \[
        \hat{\mu}_t(x)=f(x;\hat{\theta}_t),\qquad
        \hat{g}_t(x)=\partial_x f(x;\hat{\theta}_t),\qquad
        \hat{H}_t(x)=\partial^2_{xx} f(x;\hat{\theta}_t),
        \]
        and choose
        \[
        x_t \in \arg\max_{x\in\mathcal{X}}
        \left\{ \hat{\mu}_t(x) \;+\; \kappa_1|\hat{g}_t(x)| \;+\; \kappa_2|\hat{H}_t(x)| \right\}.
        \]
        \STATE  Observe $y_t$, update data.
    \ENDIF
\ENDFOR
\end{algorithmic}
\end{algorithm}}

\chg{Note that, in our experiments, slope and curvature terms are computed from the fitted parametric model $f(\cdot;\hat{\theta}_t)$ (analytic derivatives) and used to guide action selection.}

\texttt{ViOlin} is particularly attractive in data-limited settings (e.g., smallholder or resource-constrained agricultural experiments), where minimizing the number of field trials is crucial. 
\texttt{ViOlin}'s greedy exploration is guided by the model fit and its local geometry, without explicit randomized or uncertainty-driven action selection. This enables fast convergence to (local) optima when the model class matches the true reward structure and curvature information is reliable. In low-dimensional agronomic models with smooth response surfaces, \texttt{ViOlin} performs on par with simpler strategies such as $\epsilon$-greedy and UCB. \chg{However, as model complexity (e.g., multi-modal response curves) or the dimensionality of $\mathcal{X}$ increases, geometry-guided strategies such as \texttt{ViOlin} can be especially sample-efficient when the model is well specified and curvature information is reliable.} This makes \texttt{ViOlin} well suited for adaptive experimentation in limited data regimes, such as on-farm field trials or precision agriculture for smallholder contexts, where each data point is costly to obtain. 

\vspace{0.5em}

\noindent
\textbf{Summary of algorithm differences:} \\
\chg{$\epsilon$-greedy explores uniformly at random, UCB targets actions with high uncertainty (optimism-based exploration), and \texttt{ViOlin} is greedy with geometry-guided action selection, using slope/curvature information from the fitted model to promote sample-efficient learning in nonlinear settings.}\\
To illustrate how these algorithms operate and compare in practice, we include a simple, step-by-step example using the Mitscherlich model over three or four rounds in Section \ref{tab:toy-demo} of the Appendix. This concrete illustration should help readers, especially those less familiar with bandit algorithms, understand the exploration-exploitation tradeoff.

While the primary emphasis of this work is not on theoretical development, we present and discuss several well-known theoretical results and highlight their implications for the nonlinear model-based bandit framework.



\section{Theoretical guarantees for nonlinear bandit algorithms} \label{sec: theory}
In this section, we characterize the regret and sample-complexity guarantees of sequential algorithms under increasingly flexible reward models, beginning with classical linear formulations, extending to nonparametric function classes, then focusing on nonlinear parametric (mechanistic) models as used in agronomy, and finally accounting for model misspecification. While the general theory applies broadly, our emphasis is on sample efficiency and robustness in low-data agricultural regimes, where each field trial is costly and decision errors translate to both economic and environmental consequences.
\paragraph{Linear and GLM bandits.}
For linear bandits with $d$-dimensional action/context space $\mathcal{X}$, the minimax expected cumulative regret is known to satisfy
\[
\text{E}[R_T] = O\left( d \sqrt{T} \log T \right),
\]
where $T$ is the time horizon~\citep{abbasi2011improved,dani2008stochastic}. 
Generalized linear bandit models achieve similar rates up to logarithmic factors~\citep{filippi2010parametric,bastani2020mostly}. These results provide a baseline: when yield or profit responds linearly to inputs, bandit algorithms can achieve fast learning rates. However, such linear approximations are rarely biologically realistic in crop response.

\paragraph{Nonparametric bandits.}
For more general nonparametric function classes, such as Lipschitz or RKHS (kernelized) reward functions, the minimax regret scales polynomially with the dimension $d$ of the action/context space $\mathcal{X}$ ~\citep{slivkins2014contextual,srinivas2012information}:
\[
R_T = \widetilde{O}\left( T^{\frac{d+1}{d+2}} \right).
\]
  While attractive for flexibility, these guarantees deteriorate with dimension and are impractical when only a small number of seasons or trials are available, which is a common reality in agriculture.

\paragraph{NeuralUCB and expressive function classes.}
Recent work has extended bandit algorithms to highly expressive nonlinear function classes via neural networks. For instance, NeuralUCB \citep{zhou2020neural} achieves regret rates of order
\[
R_T = \widetilde{O}\!\left(\tilde{d}(\lambda)\sqrt{T}\right),
\]
where $\tilde{d}(\lambda)$ is the effective dimension of the neural tangent kernel associated with the underlying network.
 For example, for a two-layer ReLU neural network with input dimension $d$ and width $m$, it holds that $\tilde{d}(\lambda) \leq d\, \mathrm{poly}(m, \log(T/\lambda))$, so that the regret rate can be written as
\[
R_T = \widetilde{O}\left(d\, \mathrm{poly}(m)\sqrt{T}\right).
\]
Here, $\mathrm{poly}(m, \cdot)$ denotes a polynomial function of the network width and the indicated arguments. Thus, the regret bound increases with both the expressiveness of the neural network (through $m$) and the ambient input dimension $d$, but retains the $\sqrt{T}$ scaling characteristic of parametric (linear) bandit models.  However, these improvements come with heavy computational costs, the need for large data, and limited interpretability.

We summarize the regret bounds for the four classes of bandit problems in Table \ref{tab:regret-comparison}.
While such results highlight the theoretical reach of modern methods, they are poorly suited to agricultural decision-making, where data are scarce, feedback is seasonal, and recommendations must be transparent to practitioners. Our focus, therefore, is on \emph{mechanistic yield-response models} such as Mitscherlich or Michaelis-Menten. These models are parsimonious, interpretable, and biologically grounded, yet flexible enough to capture crop dose-response behavior. This specialization ensures theoretical guarantees translate into actionable insights, directly relevant for resource-efficient agricultural experimentation.

\begin{table}[H]
\centering
\caption{\chg{\textbf{Typical regret rates for common stochastic bandit models (for context).}
Here $T$ is the number of decision rounds (the ``time horizon'') and $d$ is the number of covariates (features). The notation $\tilde{O}(\cdot)$ hides logarithmic factors. The NeuralUCB entry depends on the neural-network width $m$ and other architecture constants; see the cited reference for details.}}
\label{tab:regret-comparison}
\begin{tabular}{|l|l|l|}
\hline
\textbf{Model Class} & \textbf{Regret Bound} & \textbf{Reference} \\
\hline
Linear Bandit & $O(d \sqrt{T} \log T)$ & \citet{abbasi2011improved} \\
GLM Bandit & $\widetilde{O}(d \sqrt{T})$ & \citet{filippi2010parametric} \\
Nonparametric (Lipschitz, RKHS) & $\widetilde{O}(T^{\frac{d+1}{d+2}})$ & \citet{slivkins2014contextual} \\
NeuralUCB (2-layer NN) & $\widetilde{O}(d\, \mathrm{poly}(m) \sqrt{T})$ & \citet{zhou2020neural} \\
\hline
\end{tabular}
\end{table}

\subsection{Model-based nonlinear bandits.} 
Building on these insights, we now focus on model-based nonlinear bandits, a regime particularly relevant in agronomy, where biological processes often admit interpretable, low-dimensional parameterizations. Since one of our motivations for this work is to provide decision-making algorithms in sample-limited scenarios, we focus on characterizing sample complexity more than providing regret guarantees in this section.
 Recent work by~\citet{dong2021provable}  has shown a breakthrough for model-based nonlinear bandits: for any class of reward functions with bounded \emph{sequential Rademacher complexity}, it is possible to find an $\epsilon$-approximate local maximum with sample complexity polynomial in the complexity of the model class, \emph{independent} of the action dimension. This suggests that, unlike classical methods whose sample complexity or regret is exponential in dimension, model-based methods can be vastly more efficient when the reward function admits a suitable low-complexity parameterization.

\paragraph{Sample Complexity and Local Regret for Model-based Nonlinear Bandits.}
In the model-based bandit setting, the functional form of the reward is known, but its parameters are not. The goal is therefore to efficiently identify an \emph{approximate local maximum} of this parametric reward function, rather than a global maximum which is often computationally and statistically intractable. Here, \emph{sample complexity} refers to the minimum number of experimental rounds required to guarantee, with high probability, that the recommended action achieves performance within a specified tolerance of a local optimum. This guarantee is formalized through the notion of \emph{local regret}, which measures the suboptimality of the chosen action relative to the best locally optimal action. In practice, for well-behaved (e.g., unimodal) functions such as those common in agronomic yield response, a local maximum often coincides with the global maximum.

In particular, a point $x$ is said to be an $(\epsilon_g, \epsilon_h)$-approximate local maximum if its gradient is at most $\epsilon_g$ (i.e., $\|\nabla f(x)\|_2 \leq \epsilon_g$) and the Hessian's largest eigenvalue is at most $-\epsilon_h$ (i.e., $f$ is sufficiently concave around $x$). The sample complexity quantifies how quickly an algorithm can find such a point.

In this work, we focus on reward functions (e.g., quadratic-plateau, Michaelis-Menten, Mitscherlich, and logistic) that satisfy the regularity conditions (such as bounded gradients and Hessians) assumed in recent theoretical results~\citep{dong2021provable}. These properties ensure that our algorithms are well-behaved: the models are smooth enough for efficient learning, and the theoretical sample complexity guarantees apply. As a result, we can meaningfully compare algorithms in terms of how many rounds are needed to reach near-optimal fertilizer recommendations with high confidence. We first define the notion of sequential Rademacher complexity and then state Theorem 1.1 from \cite{dong2021provable} in terms of sequential complexity, and then we tailor the result to our specific non-linear model classes.

Intuitively, sequential Rademacher complexity measures how hard it is for a learning algorithm to reliably make good decisions when facing an environment that can adapt to the algorithm's past actions. More formally,  the sequential Rademacher complexity~\citep{rakhlin2015sequential} is defined as follows.
\begin{definition}[Sequential Rademacher Complexity]
    Let $\mathcal{F}$ be a class of real-valued functions defined on $\mathcal{X}$. The \emph{sequential Rademacher complexity} of $\mathcal{F}$ over $T$ rounds is defined as
    \[
    \mathfrak{R}_T^{\mathrm{seq}}(\mathcal{F}) = \sup_{x_1, \ldots, x_T} \mathbb{E}_{\boldsymbol{r}} \left[ \sup_{f \in \mathcal{F}} \frac{1}{T} \sum_{t=1}^T r_t f\left(x_t(r_1, \ldots, r_{t-1})\right) \right]
    \]
    where the supremum is taken over all sequences of functions $x_t: \{-1,+1\}^{t-1} \to \mathcal{X}$, and $r_1,\dots,r_T$ are independent Rademacher random variables, i.e., $\mathbb{P}(r_t = 1) = \mathbb{P}(r_t = -1) = \frac{1}{2}$ for each $t$.
\end{definition}
 Now, we present the result by \cite{dong2021provable} that determines the sample complexity rate for \texttt{ViOlin}.

\begin{theorem}[Sample Complexity for Model-based Nonlinear Bandits {\citep[Theorem 1.1]{dong2021provable}}]
\label{thm:dong}
Suppose the sequential Rademacher complexity of the loss function class $\mathfrak{R}_T^{\mathrm{seq}}(\mathcal{F})$ induced by the reward function class $\{f(\theta, \cdot) : \theta \in \Theta\}$ is bounded by $\sqrt{R(\Theta) T \, \mathrm{polylog}(T)}$ for some complexity parameter $R(\Theta)$.  
Then, there exists an algorithm (\texttt{ViOlin})) (Algorithm~ \ref{alg:ViOlin}) that finds an $\delta$-approximate local maximum with $\widetilde{O}(R(\Theta)\delta^{-8})$ samples, independent of the dimension of the action space.
\end{theorem}

Note that most of the agronomic models assume bounded non-linear functions, such as ones that exhibit plateauing behavior after a certain threshold. Therefore, it is of interest to quantify the  sequential Rademacher complexity for these class of functions. For parametric classes with bounded functions and inputs (such as those considered in this paper), we prove a theoretical  bound on the sequential Rademacher complexity of bounded functions and show that indeed for this class of functions  the sequential Rademacher complexity has an upper bound as required in Theorem \ref{thm:dong}.
\begin{theorem}[Sequential Rademacher Complexity for Bounded Functions] \label{thm: SeqRademacherForBoundedFunctions}
Let $\mathcal{F} \subseteq [-B_\mathcal{F}, B_\mathcal{F}]^{\mathcal{D}}$ be a class of functions uniformly bounded by $B_\mathcal{F} > 0$. Then the sequential Rademacher complexity of $\mathcal{F}$ satisfies
\[
\mathfrak{R}_T^{\mathrm{seq}}(\mathcal{F}) := \sup_{\boldsymbol{x}} \mathbb{E}_{\boldsymbol{r}} \left[ \sup_{f \in \mathcal{F}} \frac{1}{T} \sum_{t=1}^T r_t f(x_t(r_1, \ldots, r_{t-1})) \right]
\leq C \cdot B_\mathcal{F} \cdot \sqrt{\frac{\log T}{T}},
\]
for some universal constant $C > 0$.
\end{theorem}
We defer the proof of Theorem \ref{thm: SeqRademacherForBoundedFunctions} to the Appendix (Section \ref{sec: proof_sample_complexity}) for brevity and discuss its implications from an application viewpoint.
Our theoretical result complement \cite{dong2021provable}'s general framework by concretely instantiating the sample complexity for these practically important nonlinear classes, including quadratic plateau (threshold), Michaelis-Menten, Mitscherlich, and logistic models, the sequential Rademacher complexity $\mathfrak{R}_T^{\mathrm{seq}}(\mathcal{F})$ can be bounded as $\widetilde{O}(1/\sqrt{T})$. \chg{Figure~\ref{fig:seq_rad_complexity} visualizes the bound in Theorem~\ref{thm: SeqRademacherForBoundedFunctions} for the four nonlinear yield-response families in Table~\ref{tab:profit_optima}. In this result, model dependence enters through a uniform magnitude bound $B_{\mathcal F}$ satisfying $\sup_{x\in\mathcal X}|f(x)|\le B_{\mathcal F}$ for all $f\in\mathcal F$. We set $B_{\mathcal F}=\sup_{x\in[0,250]}|f(x)|$ using the parameter settings from Section~\ref{sec: simulation} and plot the resulting (conservative) upper bounds versus $t$; the corresponding $B_{\mathcal F}$ values are reported in Appendix~\ref{app:BF}. This figure is included purely for intuition: it illustrates that the (conservative) complexity bound decreases as the number of rounds increases, and it should not be interpreted as a model-selection criterion.} Therefore, the sample complexity for identifying an $\epsilon$-optimal arm (in the sense of a local maximum) in these bandit settings scales as $\widetilde{O}(1/\epsilon^8)$, with constants depending on the parameter and input bounds. This bound, as established in \cite{dong2021provable}, is general and not specialized to the structure of agronomic models in terms of its polynomial dependence on $\epsilon$.  Deriving sharper rates for the structured nonlinear models considered here is an important direction for future work.
\begin{figure}[t]
\centering

\newcommand{\Bqp}{198.8}
\newcommand{\Bmit}{197.18}
\newcommand{\Blog}{189.77}
\newcommand{\Bmm}{167.14}

\begin{tikzpicture}
\begin{axis}[
    width=0.92\linewidth,
    height=5.6cm,
    xlabel={Round $t$},
    ylabel={Illustrative bound $B_{\mathcal F}\sqrt{\log(t)/t}$},
    xmin=0, xmax=100,
    ymin=0, ymax=120,
    grid=both,
    ticklabel style={font=\small},
    label style={font=\small},
    legend style={
        at={(0.98,1.03)},
        anchor=north east,
        draw=none,
        fill=none,
        font=\small
    },
    legend cell align=left,
    clip=false
]

\addplot+[mark=none, very thick, domain=5:100, samples=250]
    {\Bqp*sqrt(ln(x)/x)};
\addlegendentry{Quadratic-plateau ($B_{\mathcal F}=198.8$)}

\addplot+[mark=none, very thick, dashed, domain=5:100, samples=250]
    {\Bmit*sqrt(ln(x)/x)};
\addlegendentry{Mitscherlich ($B_{\mathcal F}=197.18$)}

\addplot+[mark=none, very thick, dotted, domain=5:100, samples=250]
    {\Blog*sqrt(ln(x)/x)};
\addlegendentry{Logistic ($B_{\mathcal F}=189.77$)}

\addplot+[mark=none, very thick, dashdotted, domain=5:100, samples=250]
    {\Bmm*sqrt(ln(x)/x)};
\addlegendentry{Michaelis--Menten ($B_{\mathcal F}=167.14$)}

\addplot+[black, thick, densely dashed] coordinates {(30,0) (30,120)};
\node[anchor=west, font=\small] at (axis cs:30,86) {$T=30$};

\end{axis}
\end{tikzpicture}

\caption{\chg{\textbf{Illustration of model-dependent upper bounds on sequential Rademacher complexity.}}
We plot $B_{\mathcal F}\sqrt{\log(t)/t}$ (up to a universal constant common to all models) for the four yield-response families used in the simulations, where $B_{\mathcal F}=\sup_{x\in[0,250]}|f(x; \theta_{\text{true}})|$ is computed from the parameter settings in Section~5. The bound decreases with the number of rounds $t$; differences across model families appear through the constant $B_{\mathcal F}$. The dashed vertical line marks the horizon $T=30$ used in the well-specified experiments.}
\label{fig:seq_rad_complexity}
\end{figure}

\begin{corollary}[Sample Complexity for online learning of Bounded Non-linear Reward classes]
Using the bound in Theorem \ref{thm: SeqRademacherForBoundedFunctions} in Theorem \ref{thm:dong}, an $\epsilon$-optimal solution (arm) can be identified with sample complexity (number of time steps), $\widetilde{O}(1/\delta^8)$ for the \texttt{ViOlin} algorithm of \citep{dong2021provable}, independent of the action (arm) space dimension.
\end{corollary}

\paragraph{Interpretation.}
This shows that for commonly used nonlinear agronomic yield response models, bandit algorithms can efficiently identify near-optimal fertilizer rates with a number of field trials that scales polynomially with the desired accuracy and avoids the curse of dimensionality in the action space. Our results instantiate \cite{dong2021provable}'s general theory for these specific models, providing practical sample complexity bounds for adaptive experimentation in agriculture. In Theorem~\ref{thm:dong}, the parameter $\delta > 0$ represents the maximum allowable gap between the algorithm's selected action (e.g., a nitrogen rate recommendation) and the optimal (locally maximizing) action in terms of expected reward. That is, with high probability, the bandit algorithm identifies an action $x$ such that
\[
f^*(x^*) - f^*(x) \leq \delta,
\]
where $f^*$ is the true (unknown) reward or profit function, $x^*$ is a (local) maximizer of $f^*$, and $x$ is the action recommended by the algorithm.


The sample complexity required to achieve an $\epsilon$-optimal recommendation scales polynomially with $1/\delta$ (specifically, as $O(1/\delta^8)$ in Theorem~\ref{thm:dong}), so tighter tolerances (smaller $\delta$) require proportionally more field trials. In practice, $\delta$ should be set according to what constitutes an economically meaningful margin for decision-making.

\citet{dong2021provable} provide broad sample complexity and regret guarantees for nonlinear bandit problems, including highly expressive classes such as 2-layer neural networks. By contrast, our agricultural decision models are low-dimensional and mechanistic (e.g., quadratic plateau), which makes them both interpretable and practically relevant. In this structured setting, algorithms such as UCB and $\epsilon$-greedy are not only straightforward to implement but also expected to admit stronger theoretical guarantees than the general results in \citet{dong2021provable}. We adopt their bounds as a reference baseline, while highlighting the opportunity for future work to establish sharper, model-specific results. This positions our empirical study on firm theoretical footing while opening avenues for more refined analysis.

\begin{remark}
 The proof techniques used here such as bounding regret via sequential Rademacher complexity and the resulting sample-complexity corollary for \texttt{ViOlin} are direct instantiations of general tools from the bandit literature.
What is distinctive in this result is their adaptation to mechanistic nonlinear models (e.g., Mitscherlich, Michaelis-Menten), which integrates domain knowledge with statistical guarantees. This integration ensures that theoretical results are not only mathematically sound but also directly interpretable and actionable in agronomic applications.
\end{remark}

\paragraph{Under Model Misspecification}
In many applications, including agricultural yield optimization, the true reward function may not be perfectly captured by any member of a chosen parametric model class (such as quadratic plateau or Michaelis-Menten). This is known as \emph{model misspecification}, and can have important implications for the performance and guarantees of bandit algorithms. \chg{A concrete example in agronomy is when the true yield response exhibits smooth saturation (well captured by a Mitscherlich curve), but the fitted model class is a quadratic-plateau response. In that case, the quadratic-plateau model may fit the central range reasonably well but can misrepresent the curvature near the shoulder/plateau region, leading to biased estimates of the profit-maximizing nitrogen rate.}

Most bandit algorithms rely on the assumption of a \emph{well-specified model} (or ``realizability''), that is, the existence of a function $f^* \in \mathcal{F}$ such that the mean reward is $\mu(x) = f^*(x)$ for all actions $x$~\citep{chu2011contextual, abbasi2011improved, agarwal2012contextual, foster2018practical}. However, in practice, exact realizability rarely holds, and it is important to consider the effect of model misspecification.

Recent work (see, e.g.,~\citep{greenewald2021adapting, lattimore2020learning, ghosh2017misspecified}) has studied bandits under various notions of misspecification. A common formulation is the \emph{uniform $\varepsilon$-misspecification} setting, where
\begin{equation}
    \inf_{f \in \mathcal{F}} \sup_{x \in \mathcal{X}} \left| \mu(x) - f(x) \right| \leq \chg{\varepsilon_\mathcal{F}},
    \label{eq:uniform_misspec}
\end{equation}
for some $\varepsilon_\mathcal{F} > 0$.

For linear bandits,~\citet{lattimore2020learning} show that the cumulative regret must satisfy
\[
    R_T \gtrsim d\sqrt{T} + \varepsilon_\mathcal{F} \sqrt{d}\,T,
\]
where $d$ is the model dimension and $T$ is the time horizon. The first term corresponds to the minimax regret under well-specified models, while the second is an unavoidable ``price of misspecification''.
In our setting, the models (quadratic plateau, logistic, etc.) may be misspecified relative to the true reward. As such, the regret of our bandit algorithms can be interpreted as
\[
R_T \;=\; O(\text{model-based regret}) \;+\; O(T\,\varepsilon_\mathcal{F}),
\]
where $\varepsilon_\mathcal{F}$ is the uniform misspecification level in~\eqref{eq:uniform_misspec}.
This highlights the need for robust algorithms and motivates future work on model selection or hybrid model-based/nonparametric approaches.

\section{Simulation study: Emulating agricultural decision-making} \label{sec: simulation}


 \chg{\subsection{Experimental setup and evaluation metrics}
Building on the profit objective in Section~\ref{sec:profit_objective}, we evaluate the proposed bandit algorithms in controlled simulations that mimic fertilizer-rate decisions under limited data.
In each experiment, the mean yield response is specified by a chosen model family $f(x;\theta)$ (Mitscherlich, quadratic-plateau, Michaelis--Menten, or logistic), actions are restricted to a discrete grid $\mathcal{X}$ of nitrogen rates, and economic conditions are varied through the price parameters $(p_y,p_x)$.
At any round, if the parameters $\theta$ were known, the profit-maximizing rate would be
\[
x^\star \in \arg\max_{x\in\mathcal{X}} \ \Pi(x)=\arg\max_{x\in\mathcal{X}}\{p_y f(x;\theta)-p_x x\}.
\]
In practice, $\theta$ is unknown and must be learned online from noisy outcomes, so the algorithms repeatedly fit a model and select fertilizer rates based on plug-in profit estimates and exploration bonuses.
For the nonlinear response families considered here, the continuous maximizer admits closed-form expressions (Table~\ref{tab:profit_optima}).
These formulas clarify how the economically optimal decision depends jointly on agronomic parameters $\theta$ and prices $(p_y,p_x)$, motivating regret and profit-based comparisons as data accumulate over rounds.}

We now turn to the problem of sequentially optimizing fertilizer decisions using model-based bandit algorithms that account for uncertainty and limited data. At each round, the algorithm fits a nonlinear yield response model to the observed data and selects the fertilizer rate that maximizes the estimated economic profit, optionally incorporating an uncertainty-based exploration bonus.\chg{We investigate three nonlinear model-based strategies below. To avoid confusion with the generic {reward} (yield) formulation in Section~3, Algorithms~\ref{alg:eps}--\ref{alg:ViOlin_profit} in this section are {profit-based variants} of Algorithms~\ref{alg:eps_greedy}--\ref{alg:ViOlin} (and the baselines), obtained by replacing the reward with the profit objective $\Pi(x)=p_yY(x)-p_xx$. Algorithms~\ref{alg:eps_greedy}--\ref{alg:ViOlin} provide the core reward-based pseudocode templates, while Algorithms~\ref{alg:eps}--\ref{alg:ViOlin_profit} give the corresponding simulation-ready, implementation-level versions used in Section~\ref{sec: simulation}.} 
\begin{enumerate}
    \item \emph{Model-based $\epsilon$-greedy algorithm (Algorithm~\ref{alg:eps})}: This method selects a random fertilizer rate with probability $\epsilon_t = t^{-a}$, with $a$ chosen from $\{0.5, 1, 1.5\}$, for exploration. With the remaining $(1-\epsilon_t)$ probability, it chooses the rate closest to the current profit-maximizing estimate $x^*$, computed using the fitted model and closed-form expressions where available.
    \item \emph{Model-based UCB algorithm (Algorithm~\ref{alg:ucb})}: This approach augments the estimated profit at each arm with a model-based uncertainty term, derived from the parameter covariance of the fitted nonlinear model, and selects the rate with the highest upper confidence bound. \chg{Model based nonlinear-UCB uses the delta-method uncertainty proxy \eqref{eq:uncertainty} with warm-start $n_0$.}
    \item \emph{\texttt{ViOlin}) algorithm (Algorithm~\ref{alg:ViOlin})}: A curvature-aware strategy that exploits local second-order structure in the profit function to accelerate convergence, particularly under smooth nonlinearity.
\end{enumerate}
 Although these methods differ in their exploration strategies, they all share the same structure: fitting a nonlinear model to guide adaptive decision-making under uncertainty.  Accurate initialization of nonlinear model parameters is crucial to the success of these algorithms, especially in small-sample settings.  To benchmark the performance of nonlinear model-based strategies, we compare them against two widely used baselines: LinUCB \citep{chu2011contextual} and $k$-Nearest Neighbor-UCB (kNN-UCB) \citep{reeve2018k}. 
\begin{enumerate}
    \item[4.] \emph{LinUCB algorithm  (Algorithm~\ref{alg:linucb})}: A linear bandit method that models yield as a linear function of fertilizer rate and selects actions using upper confidence bounds derived from online regression.
    \item[5.] \emph{kNN-UCB algorithm (Algorithm~\ref{alg:knnucb})}: A nonparametric method that estimates yield using the $k$ nearest neighbors of each candidate arm and quantifies uncertainty via local sample variance.
\end{enumerate}
Further descriptions of these algorithms can be found in the Appendix. 
These baselines span the spectrum from parametric simplicity to nonparametric flexibility and help highlight the advantages of domain-informed nonlinear models in sample-constrained settings. 
\begin{algorithm}[H]
\caption{Model-Based $\epsilon$-Greedy for Economic Profit Maximization}
\label{alg:eps}
\begin{algorithmic}[1]
\REQUIRE Fertilizer levels $\mathcal{X}$, time horizon $T$, prices $p_y, p_x$,  exploration rate $\epsilon_t$
\STATE Initialize dataset $\mathcal{D} \leftarrow \emptyset$
\FOR{$t = 1$ to $T$}
    \IF{$u \sim \text{Uniform}(0, 1) < \epsilon_t$ or $|\mathcal{D}| < n_0$}
        \STATE Select fertilizer level $x_t \sim \text{Uniform}(\mathcal{X})$ \hfill // Explore
    \ELSE
        \STATE Fit model $f(x; \hat{\theta})$ to data in $\mathcal{D}$
        \STATE Compute $x^* = \arg\max_{x \in \mathcal{X}}\, \left[ p_y f(x; \hat{\theta}) - p_x x \right]$
        \STATE Select $x_t \in \mathcal{X}$ closest to $x^*$
    \ENDIF
    \STATE Apply $x_t$, observe yield $Y_t$
    \STATE Compute profit: $\Pi_t = p_y \cdot Y_t - p_x \cdot x_t$
    \STATE Add $(x_t, Y_t)$ to dataset $\mathcal{D}$
\ENDFOR
\end{algorithmic}
\end{algorithm}
\begin{algorithm}[H]
\caption{Model-Based UCB for Economic Profit Maximization}
\label{alg:ucb}
\begin{algorithmic}[1]
\REQUIRE Fertilizer levels $\mathcal{X}$, horizon $T$, prices $p_y, p_x$, UCB constant $\alpha$, \chg{warm-start length $n_0$}
\STATE Initialize dataset $\mathcal{D} \leftarrow \emptyset$
\FOR{$t = 1$ to $T$}
    \IF{$|\mathcal{D}| < n_0$}
        \STATE Select $x_t \sim \mathrm{Uniform}(\mathcal{X})$ \hfill // warm-start exploration
    \ELSE
        \STATE Fit reward model $f(x; \hat{\theta})$ to data $\mathcal{D}$
        \STATE Estimate parameter covariance $\widehat{\mathrm{Cov}}(\hat{\theta})$ (e.g., from nonlinear least squares)
        \FOR{each $x \in \mathcal{X}$}
            \STATE Predict profit: $\hat{\Pi}(x) = p_y \cdot f(x; \hat{\theta}) - p_x \cdot x$
            \STATE Compute uncertainty proxy (delta-method, cf.\ \eqref{eq:uncertainty}):
            \[
                \mathrm{Unc}(x) = p_y \cdot \sqrt{ \nabla_{\theta} f(x; \hat{\theta})^\top \widehat{\mathrm{Cov}}(\hat{\theta}) \nabla_{\theta} f(x; \hat{\theta}) }
            \]
            \STATE Compute UCB score: $UCB(x) = \hat{\Pi}(x) + \alpha \cdot \mathrm{Unc}(x)$
        \ENDFOR
        \STATE Select $x_t = \arg\max_{x \in \mathcal{X}} UCB(x)$
    \ENDIF
    \STATE Apply $x_t$, observe yield $Y_t$
    \STATE Compute profit: $\Pi_t = p_y \cdot Y_t - p_x \cdot x_t$
    \STATE Update $\mathcal{D} \leftarrow \mathcal{D} \cup \{(x_t, Y_t)\}$
\ENDFOR
\end{algorithmic}
\end{algorithm}

\begin{table}[ht]
\centering
\caption{\chg{\textbf{Closed-form fertilizer rate that maximizes \emph{profit} under four yield--nitrogen response models.}
We maximize per-round profit $\Pi(x)=p_y\,Y(x)-p_x x$, where $x$ is the nitrogen rate (lb N/ac), $Y(x)$ is yield, $p_y$ is grain price (\$/bu), and $p_x$ is nitrogen price (\$/lb N). Model parameters are defined in Section~5. In the simulations, actions are chosen from a discrete grid $\mathcal{X}$ (e.g., $\{0,50,\ldots,250\}$ lb N/ac), so the implemented decision is the nearest grid value to the continuous maximizer reported here.}}
\begin{tabular}{||l|l|l||}
\hline
\toprule
Model & Yield Function $Y(x)$ & Profit-maximizing $x^*$ \\
\hline
Mitscherlich &
$A(1 - e^{-b x})$ &
$-\frac{1}{b} \ln\left(\frac{p_x}{p_y A b}\right)$ \\[1em]
\hline
Quadratic (threshold) &
$\begin{array}{l}
  a + b x + c x^2,\ x \leq x_0 \\
  a + b x_0 + c x_0^2,\ x > x_0
\end{array}$ &
$\min\left\{ x_0,\ \max\left\{ 0,\ \frac{1}{2c} \left(\frac{p_x}{p_y} - b \right) \right\} \right\}$ \\[1em]
\hline
Michaelis-Menten &
$\dfrac{a x}{b + x}$ &
$\sqrt{ \dfrac{a b p_y}{p_x} } - b$ \\[1em]
\hline
Logistic &
$\dfrac{A}{1 + e^{-B(x - C)}}$ &
$\begin{array}{l}
  u^* = \dfrac{ \gamma - 2 - \sqrt{ (\gamma - 2)^2 - 4 } }{2 }, \\
  \gamma = \dfrac{B p_y A}{p_x},\quad
  x^* = C - \dfrac{1}{B} \ln u^*
\end{array}$ \\
\bottomrule
\hline
\end{tabular}
\label{tab:profit_optima}
\end{table}

\vspace{1em}

\begin{algorithm}[H]
\caption{\chg{\texttt{ViOlin} for Economic Profit Maximization}}
\label{alg:ViOlin_profit}
\begin{algorithmic}[1]
\REQUIRE Fertilizer levels $\mathcal{X}$, horizon $T$, prices $p_y,p_x$, curvature weights $\kappa_1,\kappa_2\ge 0$, warm-start length $n_0$, learner/estimator $\mathcal{O}$
\STATE Initialize dataset $\mathcal{D}\leftarrow \emptyset$
\FOR{$t=1$ to $T$}
    \IF{$|\mathcal{D}| < n_0$}
        \STATE Select $x_t \sim \mathrm{Uniform}(\mathcal{X})$ \hfill // warm-start exploration
    \ELSE
        \STATE Fit/update $\hat{\theta}$ using $\mathcal{D}$ via $\mathcal{O}$
        \FOR{each $x\in\mathcal{X}$}
            \STATE Predict profit: $\hat{\Pi}(x)=p_y f(x;\hat{\theta})-p_x x$
            \STATE Compute profit-gradient and curvature terms:
            \[
            \widehat{g}_{\Pi}(x)=\partial_x \hat{\Pi}(x)=p_y\,\partial_x f(x;\hat{\theta})-p_x,
            \qquad
            \widehat{H}_{\Pi}(x)=\partial^2_{xx}\hat{\Pi}(x)=p_y\,\partial^2_{xx} f(x;\hat{\theta})
            \]
            \STATE Compute \texttt{ViOlin} score:
            \[
            S_t(x)=\hat{\Pi}(x)\;+\;\kappa_1\big|\widehat{g}_{\Pi}(x)\big|\;+\;\kappa_2\big|\widehat{H}_{\Pi}(x)\big|
            \]
        \ENDFOR
        \STATE Select $x_t \in \arg\max_{x\in\mathcal{X}} S_t(x)$
    \ENDIF
    \STATE Apply $x_t$, observe yield $Y_t$
    \STATE Compute profit: $\Pi_t=p_y Y_t - p_x x_t$
    \STATE Update $\mathcal{D}\leftarrow \mathcal{D}\cup\{(x_t,Y_t)\}$
\ENDFOR
\end{algorithmic}
\end{algorithm}
\noindent

In the next section, we present simulation studies under both well-specified and misspecified conditions to compare the performance of these five bandit algorithms in terms of profit and cumulative regret.


\subsection{Well-specified setting}
\paragraph{Simulation setup.}
In the well-specified setting, data were generated according to a quadratic plateau model of the form:
\[
Y(x) = 
\begin{cases}
a + b x + c x^2, & \text{if } x \leq x_0 \\
a + b x_0 + c x_0^2, & \text{if } x > x_0
\end{cases},
\]
where $Y(x)$ denotes crop yield (bu/ac) at fertilizer rate $x$ (lb N/ac), and $a=80$, $b=1.2$, $c=-0.003$, $x_0=180$ were chosen to reflect realistic agronomic response curves. The profit at each round is computed as $\Pi(x) = p_y Y(x) - p_x x$, with grain price $p_y = \$5.00$ per bushel and fertilizer price $p_x \in \{0.3, 0.5, 0.7\}$ \$/lb N to mimic realistic prices per unit nitrogen fertilizer (such as Urea) in the Midwest US. Fertilizer rates are restricted to the discrete grid $\mathcal{X} = \{0, 50, 100, \ldots, 250\}$ lb N/ac, and yields are corrupted by i.i.d.\ Gaussian noise with standard deviation $\sigma = 0.5$.
\begin{algorithm}[H]
\caption{LinUCB for Economic Profit Maximization}
\label{alg:linucb}
\begin{algorithmic}[1]
\REQUIRE Fertilizer levels $\mathcal{X}$, time horizon $T$, prices $p_y, p_x$, UCB parameter $\alpha$
\STATE Initialize dataset $\mathcal{D} \leftarrow \emptyset$
\STATE Define feature mapping $\varphi(x) = [1,\, x]^\top$
\FOR{$t = 1$ to $T$}
    \IF{not enough data}
        \STATE Choose $x_t$ uniformly at random from $\mathcal{X}$
    \ELSE
        \STATE Fit linear model to data: $Y_i \approx \varphi(x_i)^\top \hat{\beta}$
        \STATE For each $x \in \mathcal{X}$:
            \begin{itemize}
                \item Predict yield: $\widehat{Y}(x) = \varphi(x)^\top \hat{\beta}$
                \item Compute profit: $\hat{\Pi}(x) = p_y \cdot \widehat{Y}(x) - p_x \cdot x$
                \item Estimate variance: $s^2(x) = \varphi(x)^\top (\mathbf{V}^{-1}) \varphi(x)$, where $\mathbf{V}$ is the feature covariance matrix
                \item Compute UCB score: $UCB(x) = \hat{\Pi}(x) + \alpha \cdot s(x)$
            \end{itemize}
        \STATE Select $x_t = \arg\max_{x \in \mathcal{X}} UCB(x)$
    \ENDIF
    \STATE Apply $x_t$, observe $Y_t$
    \STATE Compute profit: $\Pi_t = p_y \cdot Y_t - p_x \cdot x_t$
    \STATE Append $(x_t, Y_t)$ to $\mathcal{D}$
\ENDFOR
\end{algorithmic}
\end{algorithm}

\begin{algorithm}[H]
\caption{kNN-UCB for Economic Profit Maximization}
\label{alg:knnucb}
\begin{algorithmic}[1]
\REQUIRE Fertilizer levels $\mathcal{X}$, time horizon $T$, prices $p_y, p_x$, UCB parameter $\alpha$, no. of neighbors $k$
\STATE Initialize dataset $\mathcal{D} \leftarrow \emptyset$
\FOR{$t = 1$ to $T$}
    \IF{not enough data}
        \STATE Choose $x_t$ uniformly at random from $\mathcal{X}$
    \ELSE
        \FOR{each $x \in \mathcal{X}$}
            \STATE Find $k$ nearest previously tried fertilizer rates to $x$ in $\mathcal{D}$
            \STATE Compute average yield: $\widehat{Y}_k(x) = \frac{1}{k} \sum_{j=1}^k Y_j$
            \STATE Compute sample standard deviation: $s_k(x)$ of yields among neighbors
            \STATE Compute profit: $\hat{\Pi}_k(x) = p_y \cdot \widehat{Y}_k(x) - p_x \cdot x$
            \STATE Compute UCB score: $UCB(x) = \hat{\Pi}_k(x) + \alpha \cdot \frac{s_k(x)}{\sqrt{k}}$
        \ENDFOR
        \STATE Select $x_t = \arg\max_{x \in \mathcal{X}} UCB(x)$
    \ENDIF
    \STATE Apply $x_t$, observe $Y_t$
    \STATE Compute profit: $\Pi_t = p_y \cdot Y_t - p_x \cdot x_t$
    \STATE Append $(x_t, Y_t)$ to $\mathcal{D}$
\ENDFOR
\end{algorithmic}
\end{algorithm}
We evaluated five algorithms: model-based $\epsilon$-greedy, nonlinear UCB, \texttt{ViOlin} (curvature-matched nonlinear bandit), LinUCB (linear model-based UCB), and kNN-UCB (nonparametric UCB using $k=3$ nearest neighbors). Hyperparameters for each of the algorithms are chosen based on theoretical guidelines and grid searches over multiple runs. We choose $\epsilon_t = t^{-1.5}$ for $\epsilon$-greedy in this setting, $\alpha = 1$ in UCB, linUCB and kNN-UCB for uniformity and fairness in comparison, $\alpha_1=2.0,\alpha_2=640$ for the \texttt{ViOlin} algorithm based on guidelines in \cite{dong2021provable}. For each algorithm and parameter configuration, we performed $10$ independent simulation replicates of $T=30$ rounds. Each algorithm fits its specified model to the observed data and selects actions to maximize estimated profit. 
The regret at each round is defined as the difference between the average profit for the optimal arm (computed using knowledge of the true reward function) and the average profit realized by the arm selected by the algorithm, that is,
\[
\mathrm{Regret}_t = \mathbb{E}\left[ \Pi(x^*; \theta^*) - \Pi(x_t; \theta^*) \right],
\]
where $x^*$ denotes the arm maximizing the true profit function, $\theta^*$ the true model parameters and $x_t$ is the arm selected at round $t$ by the bandit algorithm employed to make decisions. This definition of regret translates directly into wasted profit and fertilizer input, making it a practical metric for resource-efficient agricultural decision-making.
For each run, we record cumulative regret, average profit, and select fertilizer rates. In order to quantify the uncertainty, we also plot the distributions of cumulative regret and average profit over 10 replications at round 10, 20, and 30, respectively. Note, profit, regret and all other monetary quantities are reported in \$/ac (yield in bu/ac and nitrogen in lb/ac) unless stated otherwise.
\paragraph{Results and Interpretation.}
Figure~\ref{fig:all_algs_compare} presents the average cumulative regret and the average profit per round for each algorithm. Figure \ref{fig:all_algs_compare_UQ} presents boxplots of the cumulative regret and average profit for each algorithm, evaluated at rounds 10, 20, and 30, respectively, across 10 independent replications. Several trends are immediately apparent:
\begin{itemize}
    \item \textbf{Superiority of nonlinear model-based methods:} Both nonparametric (kNN-UCB) and linear (LinUCB) methods incur substantially higher regret than the nonlinear model-based approaches. In particular, kNN-UCB and LinUCB accumulate regret rapidly and fail to converge to near-optimal recommendations within the small sample budget. LinUCB, although it does not experience a sharp initial dip in profit due to exploration, tends to settle on a suboptimal arm and consistently incurs losses relative to the optimal strategy.
    \begin{figure}[h!]
    \centering
    \includegraphics[width=0.45\linewidth]{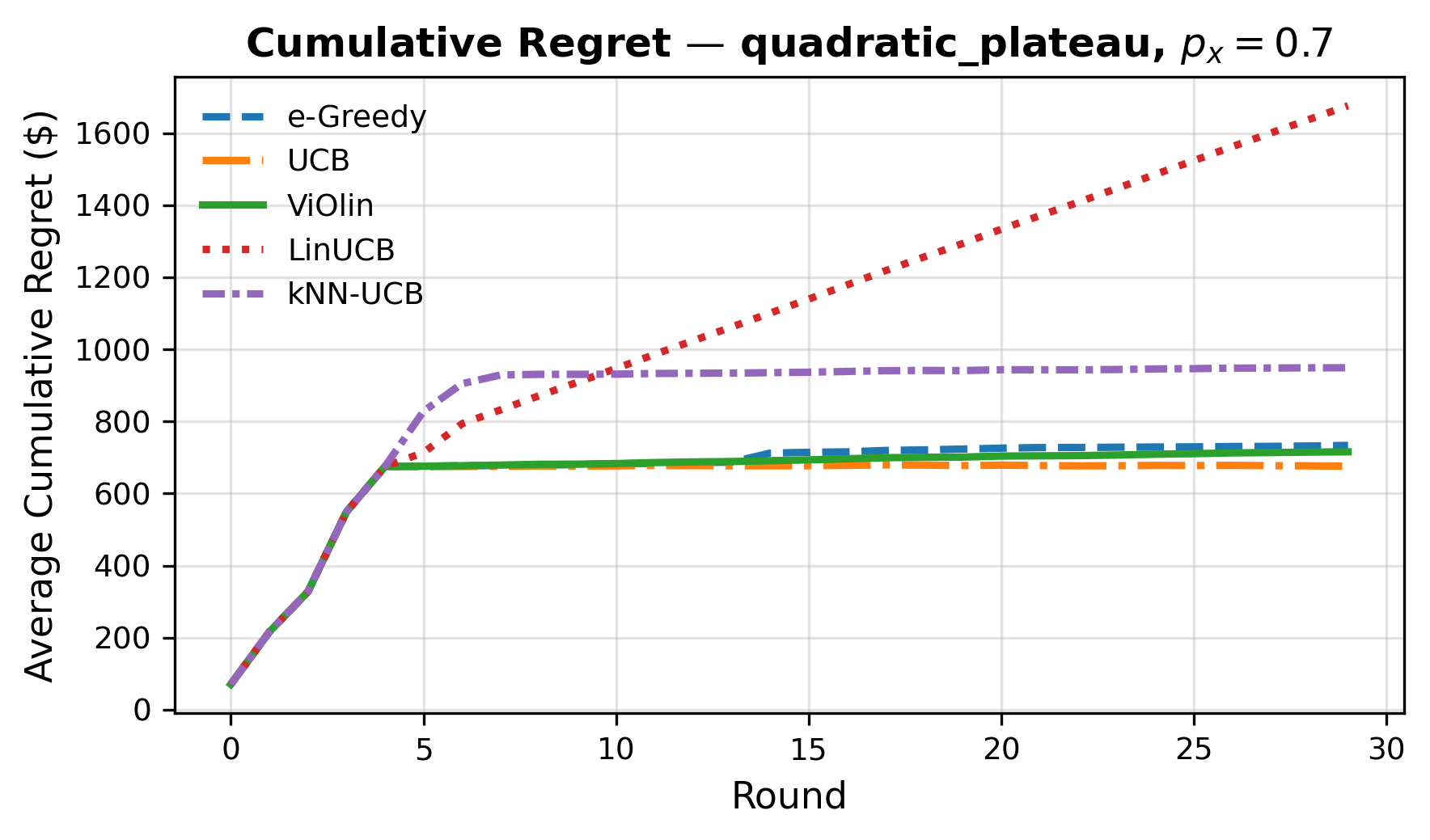}
    \includegraphics[width=0.45\linewidth]{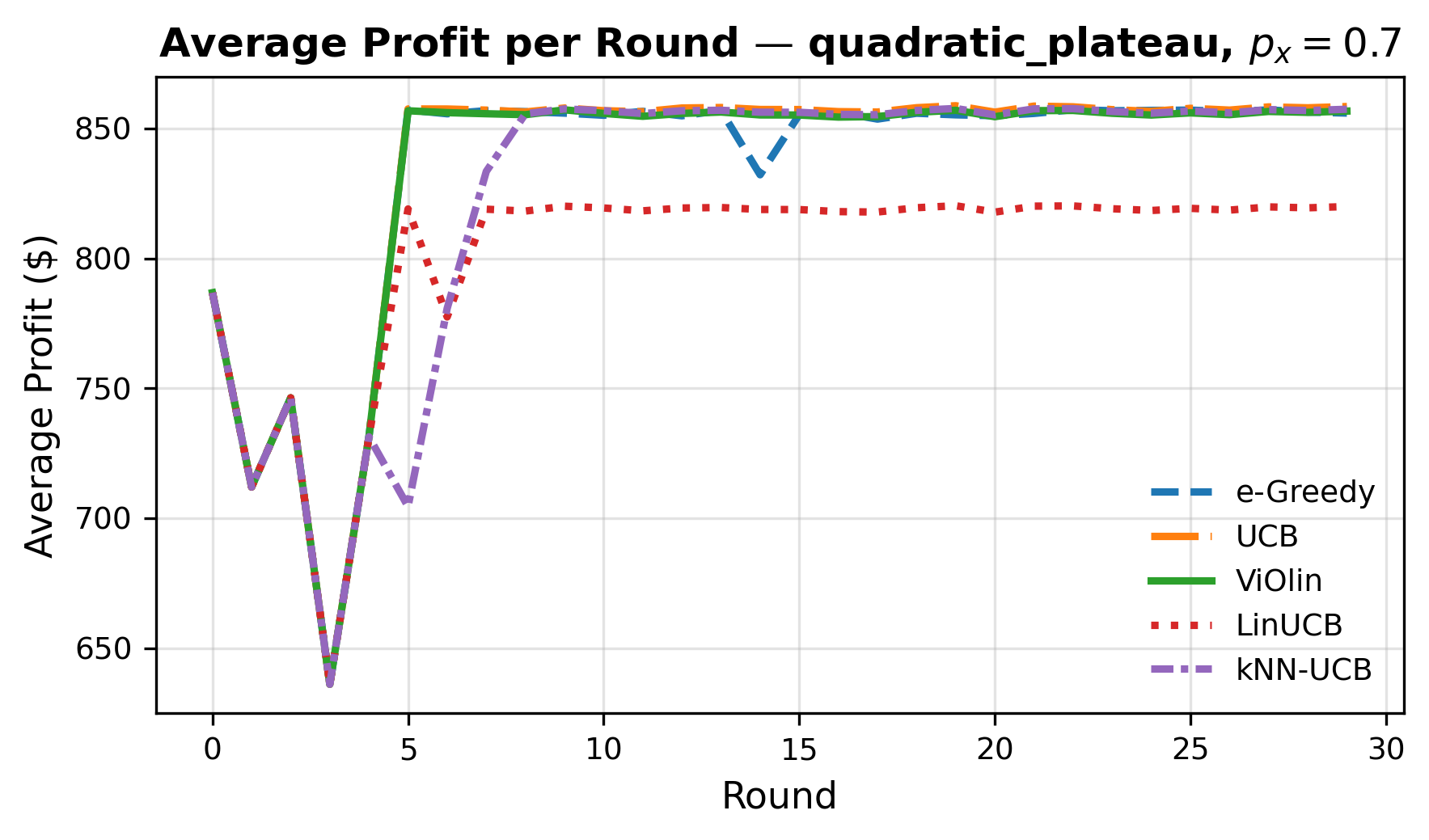}
  \caption{\chg{\textbf{Well-specified quadratic-plateau profit experiment.}
Actions lie in $\mathcal{X}=\{0,50,\ldots,250\}$ lb N/ac with $p_y=\$5$/bu, $p_x=\$0.7$/lb N, $T=30$, and $\sigma=0.5$.
\textbf{Left:} cumulative profit regret (\$/ac). \textbf{Right:} average profit (\$/ac).
Curves show means over 10 replicates for $\epsilon$-greedy, nonlinear-UCB, ViOlin, LinUCB, and kNN-UCB; ViOlin and nonlinear-UCB reduce regret fastest.}}
    \label{fig:all_algs_compare}
\end{figure}
    \item \textbf{Exploration versus exploitation trade-off:} All algorithms that actively explore---notably $\epsilon$-greedy and UCB---show a pronounced dip in profit at early rounds, reflecting the cost of exploratory actions. In contrast, the \texttt{ViOlin} algorithm, which is designed to be more exploitative (greedily maximizing predicted profit with curvature matching), avoids this initial dip and achieves high profits almost immediately. However, both $\epsilon$-greedy and UCB eventually recover and converge to policies with low cumulative regret, validating the effectiveness of their exploration in learning the optimal fertilizer rate.
   \item \textbf{Uncertainty Quantification: }  Figure \ref{fig:all_algs_compare_UQ} highlights important differences between model-based approaches (e.g., \texttt{ViOlin}, model based UCB) against linear and nonparametric benchmarks such as lin-UCB and kNN-UCB, particularly in the small-sample regime relevant to agricultural field trials.
 While kNN-UCB exhibits relatively low variability across replicates, suggesting stable short-term performance, this stability arises primarily from smoothing rather than from capturing the underlying input-response mechanism. As a result, kNN-UCB can suffer from persistent bias, leading to consistently suboptimal mean performance. This is evident in the profit distributions, i.e., the mean profit for kNN-UCB consistently lies below the lower quartile of the nonlinear model-based approaches. In other words, the algorithm produces profits that are reliably clustered, but around a lower and biased center. This suggests that while its variability is small, its systematic bias leads to consistently suboptimal performance.
   \begin{figure}[h!]
    \centering
    \includegraphics[width=0.75\linewidth]{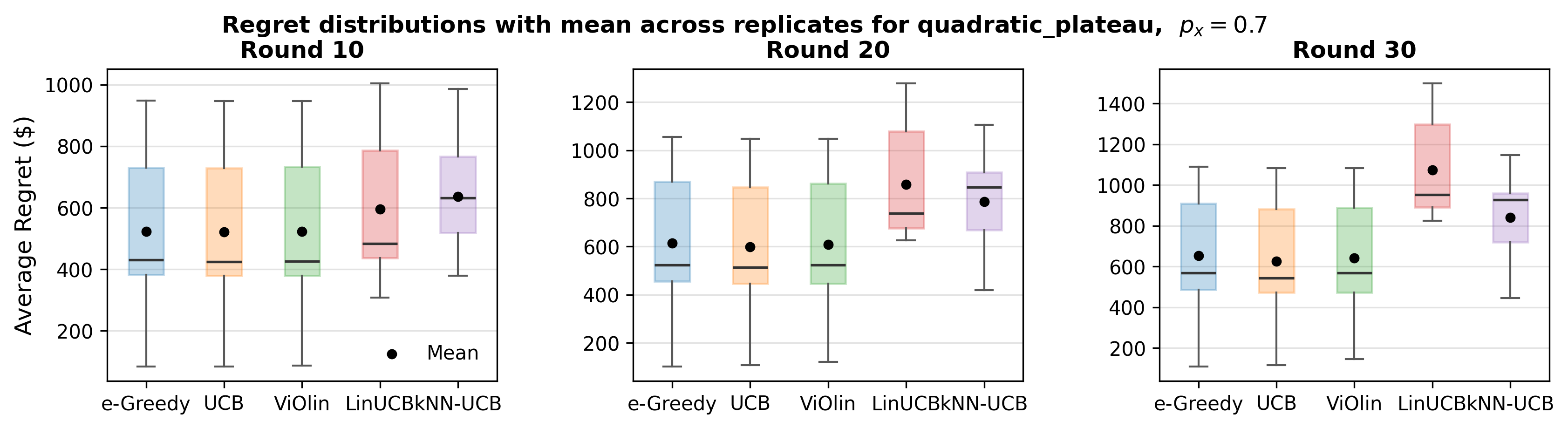}
    \includegraphics[width=0.75\linewidth]{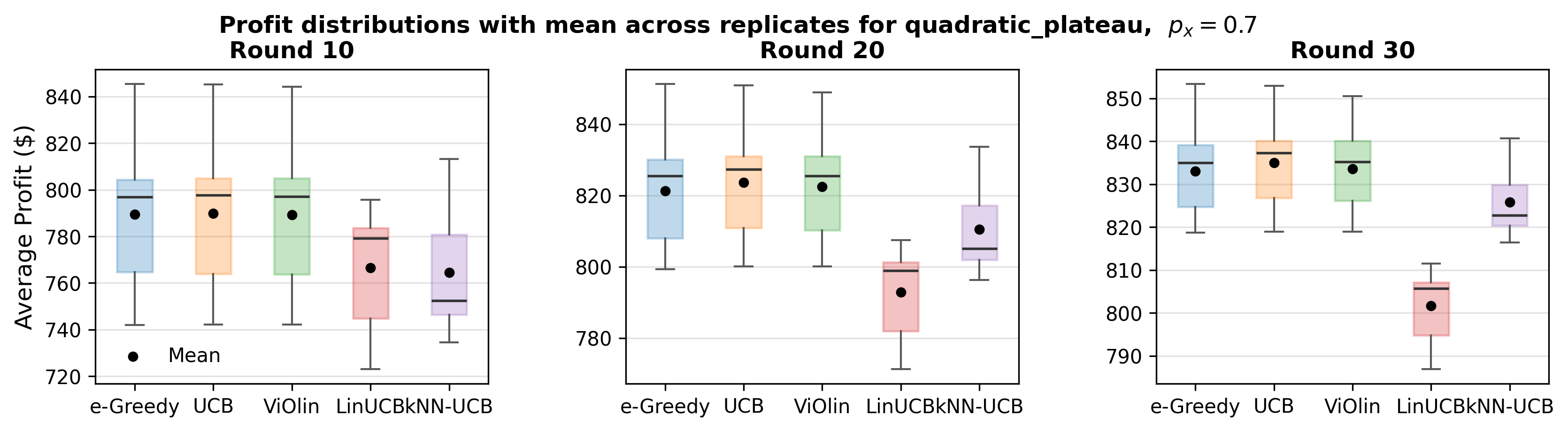}
   \caption{\chg{\textbf{Between-run variability for the well-specified quadratic-plateau experiment.}
Boxplots summarize results across 10 independent simulation replicates at rounds $t\in\{10,20,30\}$.
\textbf{Top row:} cumulative regret (in \$/ac).
\textbf{Bottom row:} average per-round profit (in \$/ac).
Each box shows the interquartile range (25th--75th percentiles) with the median; whiskers indicate the spread; the dot marks the mean.}}
    \label{fig:all_algs_compare_UQ}
\end{figure}
 
In contrast, nonlinear model-based algorithms exhibit higher run-to-run variability due to parameter estimation uncertainty, particularly in early rounds when data are scarce. However, these algorithms leverage structural assumptions that align with biological processes, enabling faster convergence to the true optimum as more data accrue. 




\end{itemize}

Similar trends were observed across other values of the fertilizer price $p_x$ and for all nonlinear response models considered. Taken together, these simulation studies demonstrate a consistent advantage of nonlinear model-based bandit algorithms over both nonparametric and linear parametric alternatives in small-sample regimes. Based on these findings, we subsequently focus our comparisons on the family of nonlinear model-based algorithms, investigating their relative strengths under a range of simulation scenarios to clarify which approaches are most advantageous in which settings.




In this next part of the simulation study, we focus on illustrating the implications of using the three non-linear model-based bandit algorithms for learning fertilizer rates sequentially with the goal of profit maximization.  Again in a small-sample regime, we run each of our algorithms for $T = 30$ rounds and replicate each run 10 times. Since our algorithms choose arms over time, it is important to visualize how the arm choices evolve over time for each of the algorithms. In Figure  \ref{fig:arms_bandit_WS}(a),  we plot the running proportion of fertilizer rate ranging from $\{0,50,\hdots,250\}$ lb N/ac, selected at each round by each of the algorithm over the 10 replicates for the quadratic plateau model for $p_x = \$ 0.5$/lb N to mimic realistic prices per unit nitrogen fertilizer (Urea) in the Midwest US. Note that in the beginning, the UCB and $\epsilon$-greedy ($\epsilon_t = t^{-1.5}$) explore other arms such as $x = 100, 150$ lb N/ac, but just after a few rounds of exploration, learn the better arm to be $x = 200$ lb N/ac. \texttt{ViOlin} is more greedy from the beginning itself as can be seen by the proportions of other arms (other than purple) selected in the initial rounds. In Figure \ref{fig:arms_bandit_WS}(b), in order to assess the affect of increasing fertilizer prices on decision-making for profit maximization, we also plotted the average selected fertilizer rate over time for the three algorithms. We note that the optimum arm (fertilizer rate) choice decreases as $p_x$ (price per unit) increases from \$0.3 to \$0.7/lb N. 
\begin{figure}[H] 
\centering
 {\resizebox{.97\textwidth}{!}{
\begin{tabular}{c}
{(a) \includegraphics[width=0.85\linewidth]{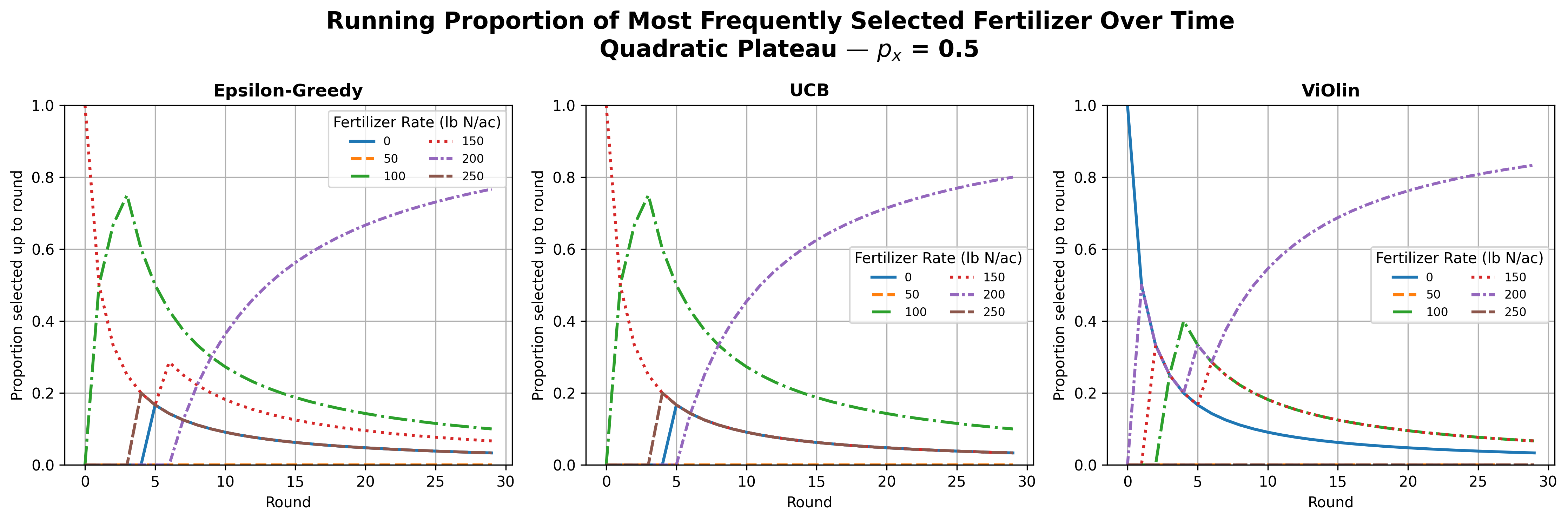}}\\
{(b)\includegraphics[width=0.85\textwidth]{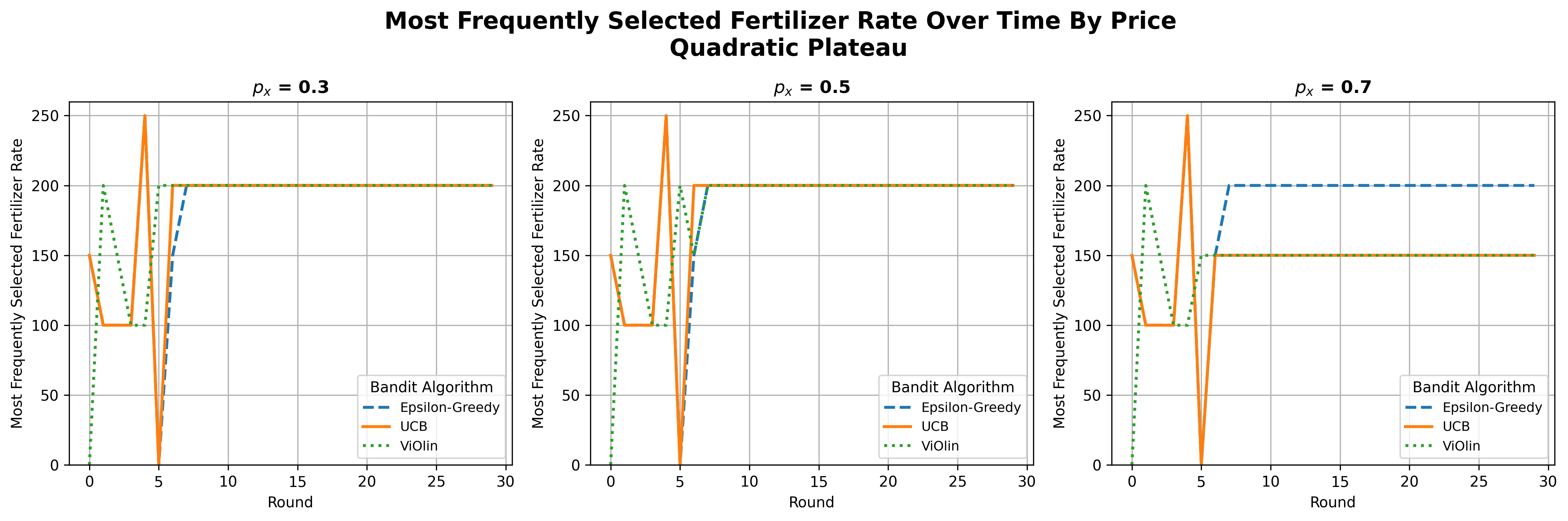}}
\end{tabular}}}
   \caption{\chg{\textbf{How the algorithms' fertilizer-rate choices evolve over time (well-specified quadratic plateau).}
Actions are chosen from $\mathcal{X}=\{0,50,\ldots,250\}$ lb N/ac.
\textbf{(a)} For $p_x=\$0.5$/lb N, the running proportion of times each nitrogen rate has been selected up to round $t$ is shown for $\epsilon$-greedy, nonlinear-UCB, and ViOlin (averaged over 10 replicates).
\textbf{(b)} The most frequently selected nitrogen rate at each round is shown for fertilizer prices $p_x\in\{0.3,0.5,0.7\}$ \$/lb N, illustrating how higher fertilizer cost shifts the learned decision toward lower nitrogen rates.}}
    \label{fig:arms_bandit_WS}
\end{figure}
Additionally, since our decisions depend significantly on how well our parameters for the non-linear model are estimated over time, in Figure \ref{fig:parameter_trajectories_WS}, we also plot the parameter trajectories over time for the quadratic plateau model with $p_x = \$0.5$/lb N. Note that all the three algorithms, as data accumulates, the parameter estimates stabilize close to the true parameter values that were used to generate the data. In terms of interpretability, one can, for example, look at the parameter estimates for $x_0$ (bottom right), which denotes the threshold value at which the quadratic model plateaus, thus reflecting the saturating point over rounds (or seasons), beyond which additional fertilizer amount provide little benefit. In our simulations, we also examined the sensitivity of parameter estimation to initialization. Because nonlinear least-squares fitting is iterative, the choice of starting values can influence convergence, particularly in small-sample regimes. To reflect realistic practice, we initialized parameters at values close to but not equal to the true generating parameters (see Section \ref{sec: additional_sims} in the Appendix for explicit values used in each model). We observed that biologically plausible initializations led to faster and more stable convergence, while poor starting values could slow parameter recovery. In practice, we recommend using domain knowledge (e.g., agronomic ranges for maximum yield, response rates, or plateau points) to set initial values when fitting nonlinear models.
We conducted similar experiments for the other non-linear models and  those results are presented in the Appendix.

\begin{figure}[h!]
    \centering
    \includegraphics[width=0.65\linewidth]{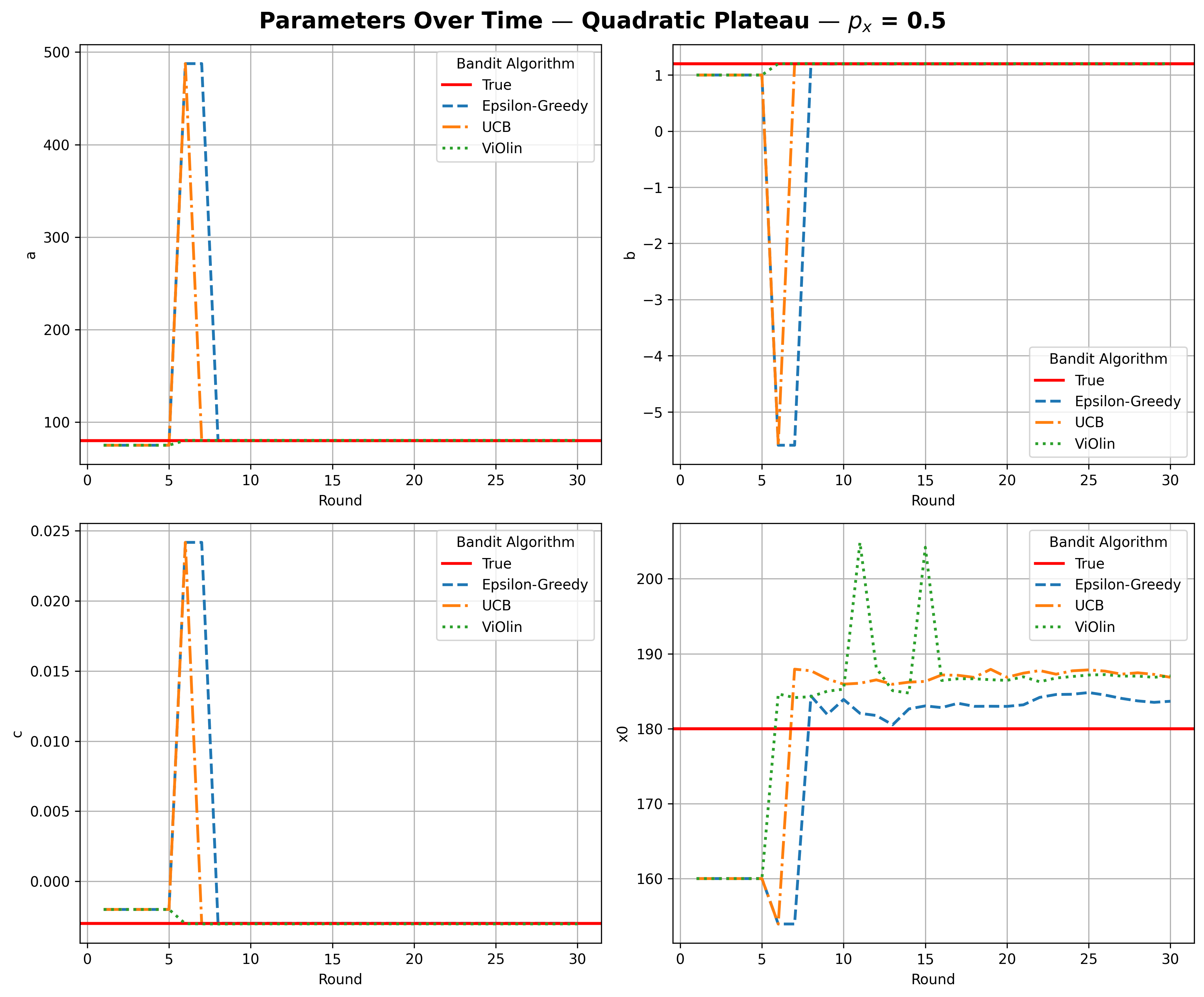}
    \caption{\chg{\textbf{Model-parameter learning over time (well-specified quadratic-plateau).}
For $p_x=\$0.5$/lb N and $T=30$, we plot estimated $(a,b,c,x_0)$ for $\epsilon$-greedy, nonlinear-UCB, and ViOlin; The  horizontal red line denotes the true value. Parameter estimates stabilize over time, indicating that the response curve can be learned from sequential data within the horizon considered.
(One representative replicate is shown; similar behavior occurs across runs. )}}
    \label{fig:parameter_trajectories_WS}
\end{figure}

\subsection{Misspecified setting}\label{sec:sim_misspec}

\paragraph{Simulation Setup.}
To mimic a realistic situation in which the decision-maker fits an \emph{approximate} but shape-compatible model, we generated data from the Mitscherlich response
\[
    Y_{\mathrm{true}}(x)=A\!\left(1-e^{-b(x-d)}\right),
    \qquad 
    A=120,\; b=0.015,\; d=80,
\]
yet \textit{fitted} a quadratic-plateau curve
\[
Y_{\mathrm{fit}}(x)=a+bx+cx^{2}\;\text{for }x\le x_{0},  
\]
truncated at $x_{0}=180$ and takes the same value $a+bx_0+cx_0^{2}$ for $x > x_0$.  All other ingredients are identical to the well-specified experiment except that we increase the time horizon to $T = 100$, specifically, price grid
$p_{x}\in\{0.3,0.5,0.7\}$ \$\,lb$^{-1}$ N, yield price
$p_{y}= \$5.00$ bu$^{-1}$, action set
$\mathcal X=\{0,50,100,\ldots,250\}$ lb N /ac,  noise $\sigma=0.5$, and  the 5 algorithms: $\epsilon$-greedy, nonlinear-UCB, \texttt{ViOlin}, LinUCB, kNN-UCB.  
10 independent replicates were run for each algorithm and price.
\begin{figure}[H]
    \centering
    \includegraphics[width=0.45\linewidth]{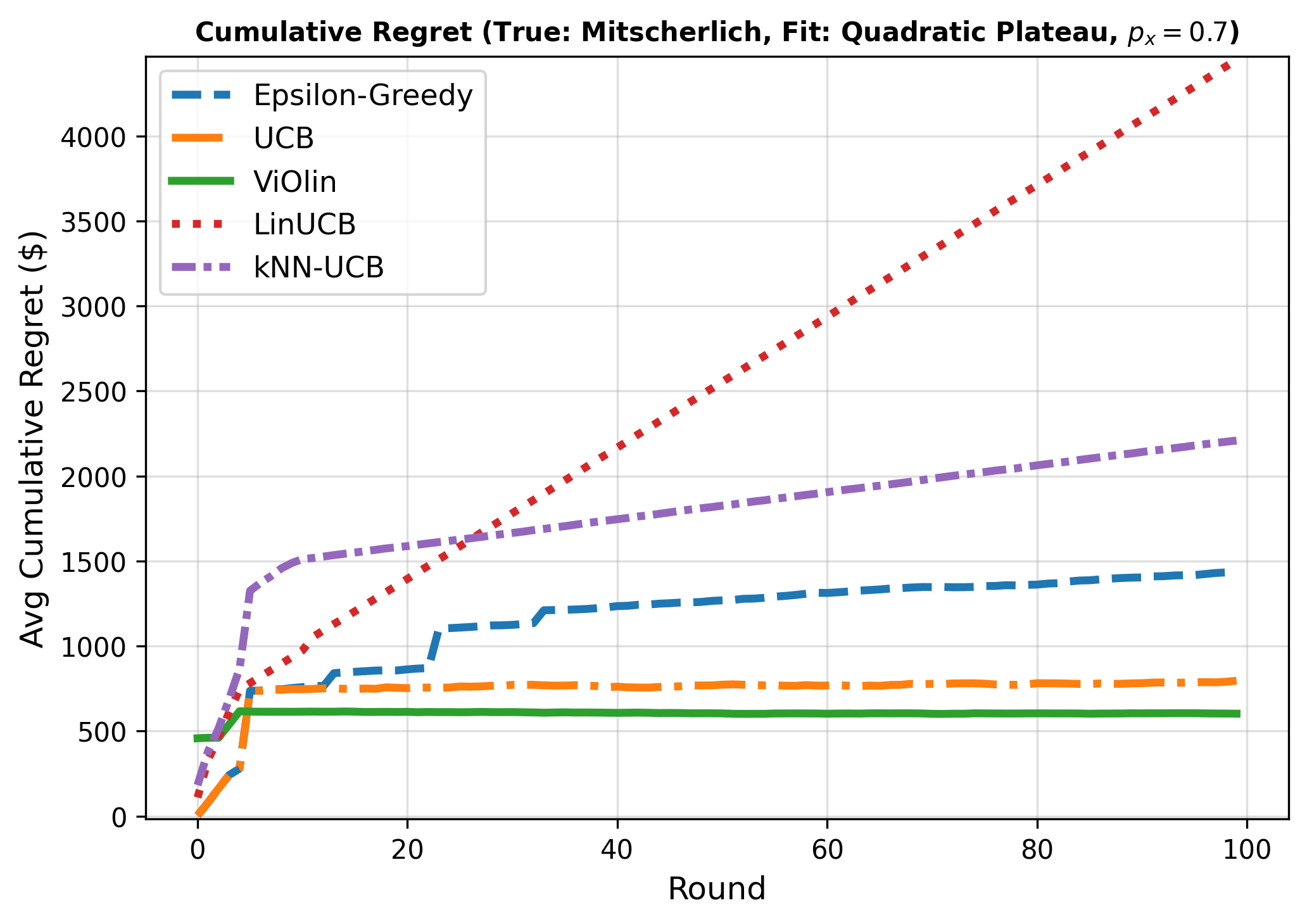}
    \includegraphics[width=0.45\linewidth]{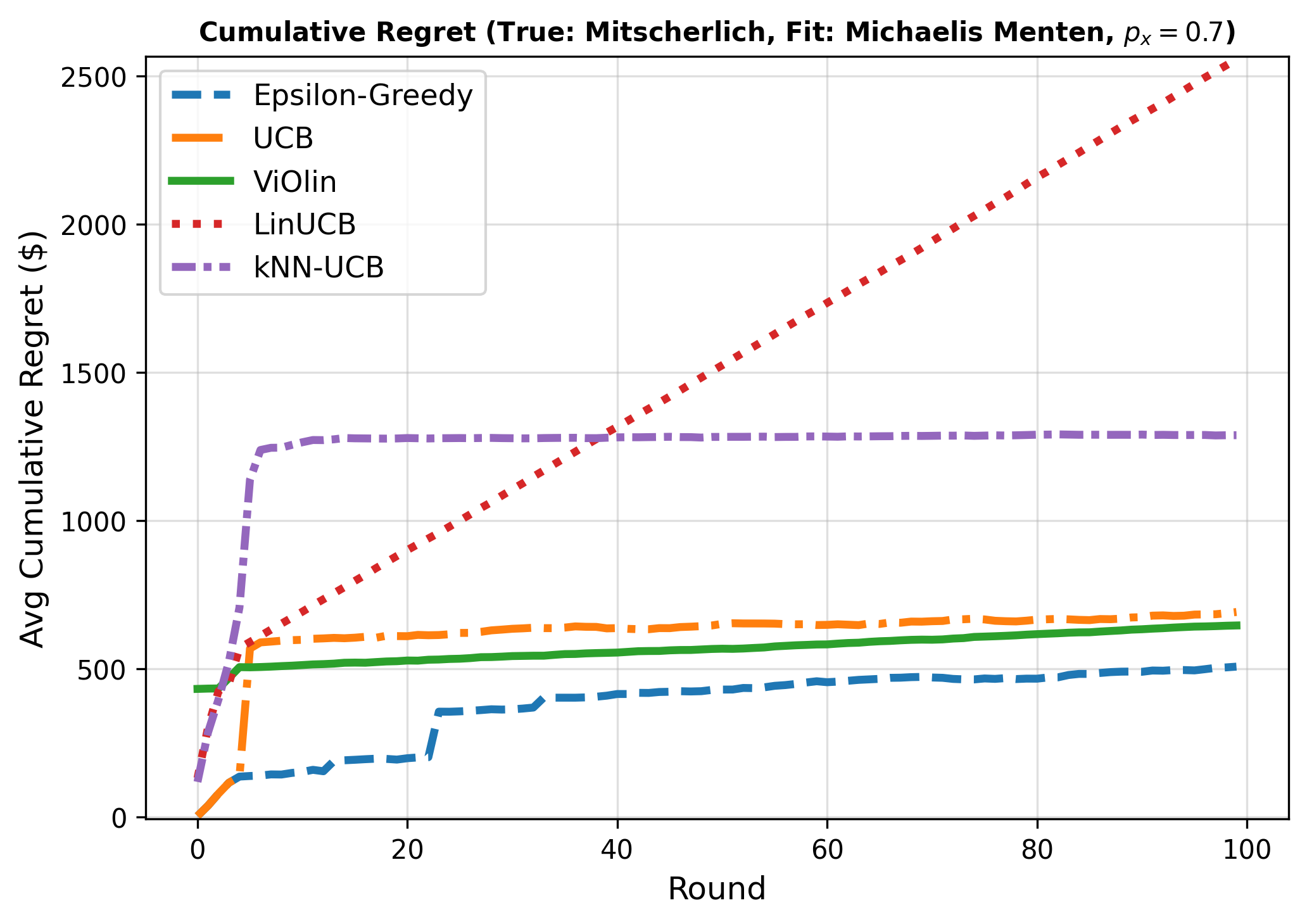}
  \caption{\chg{\textbf{Profit regret under model misspecification.}
Data are generated from a Mitscherlich (truth) yield curve, but the learner fits a different parametric family.
We use $p_y=\$5$/bu, $p_x=\$0.7$/lb N, $\mathcal{X}=\{0,50,\ldots,250\}$ lb N /ac, $\sigma=0.5$, and $T=100$.
\textbf{Left:} quadratic-plateau fit; \textbf{Right:} Michaelis--Menten fit.
Curves show mean cumulative profit regret (in \$/ac) over 10 replicates for $\epsilon$-greedy, nonlinear-UCB, ViOlin, LinUCB, and kNN-UCB. 
Under misspecification, regret increases for all methods, but nonlinear model-based policies remain competitive.}}
    \label{fig:all_algs_compare_MS}
\end{figure}
\vspace{-0.5cm}
\begin{figure}[H]
    \centering
    \includegraphics[width=0.6\linewidth]{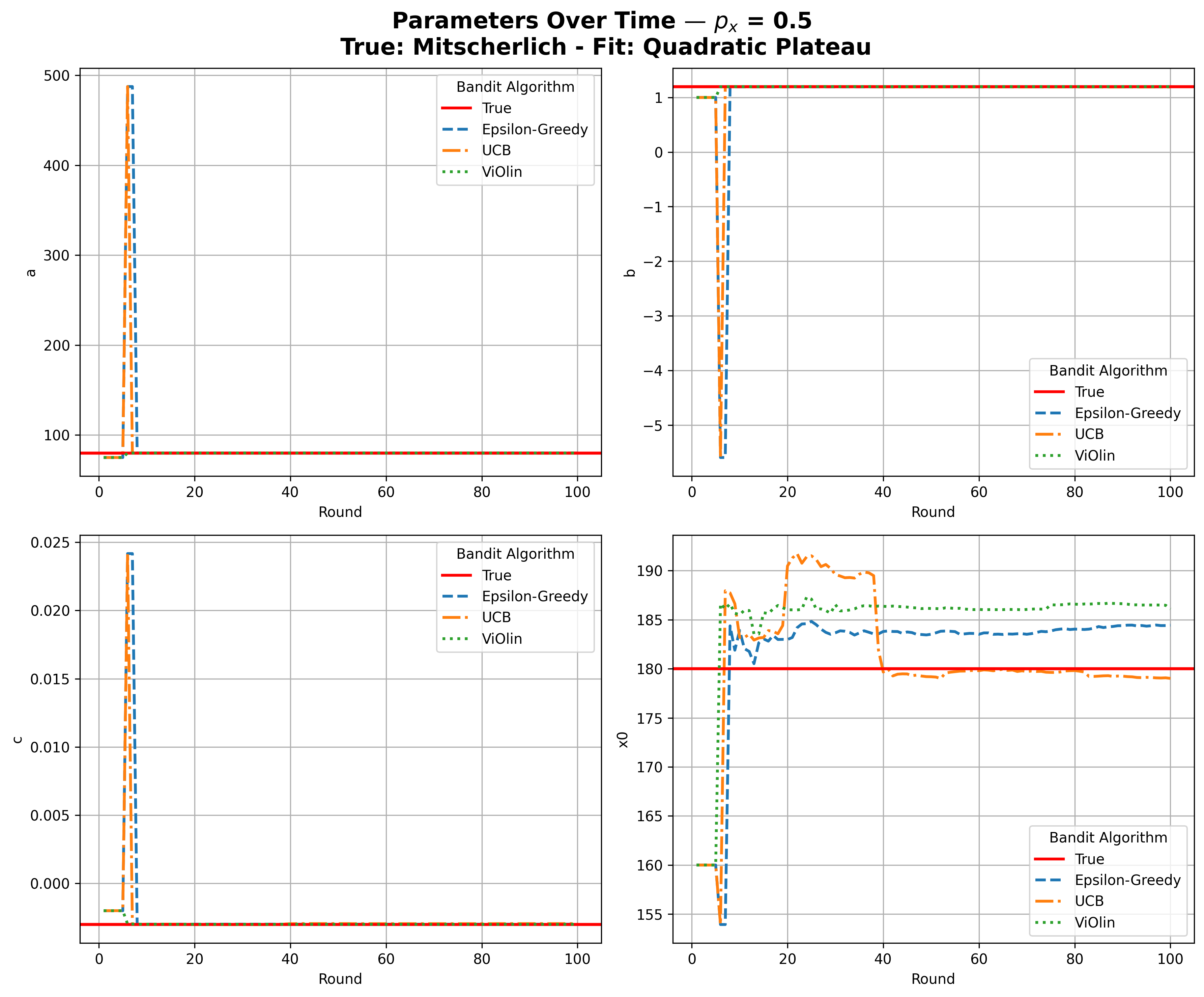}
    \caption{\chg{\textbf{Estimated parameter trajectories under misspecification (true Mitscherlich; fitted quadratic-plateau).}
Shown for $p_x=\$0.5$/lb N over $T=100$ rounds.
Panels track fitted quadratic-plateau parameters $(a,b,c,x_0)$ for $\epsilon$-greedy, nonlinear-UCB, and ViOlin.
The red reference line denotes an \emph{oracle (pseudo-true)} parameter value obtained by fitting the quadratic-plateau form to the noiseless Mitscherlich mean response on the same nitrogen grid $\mathcal{X}$ (used only as a benchmark).}}
    \label{fig:param_trajectories_MS}
\end{figure}
\begin{figure}[h!] 
\centering
 {\resizebox{.97\textwidth}{!}{
\begin{tabular}{c}
{(a) \includegraphics[width=0.95\linewidth]{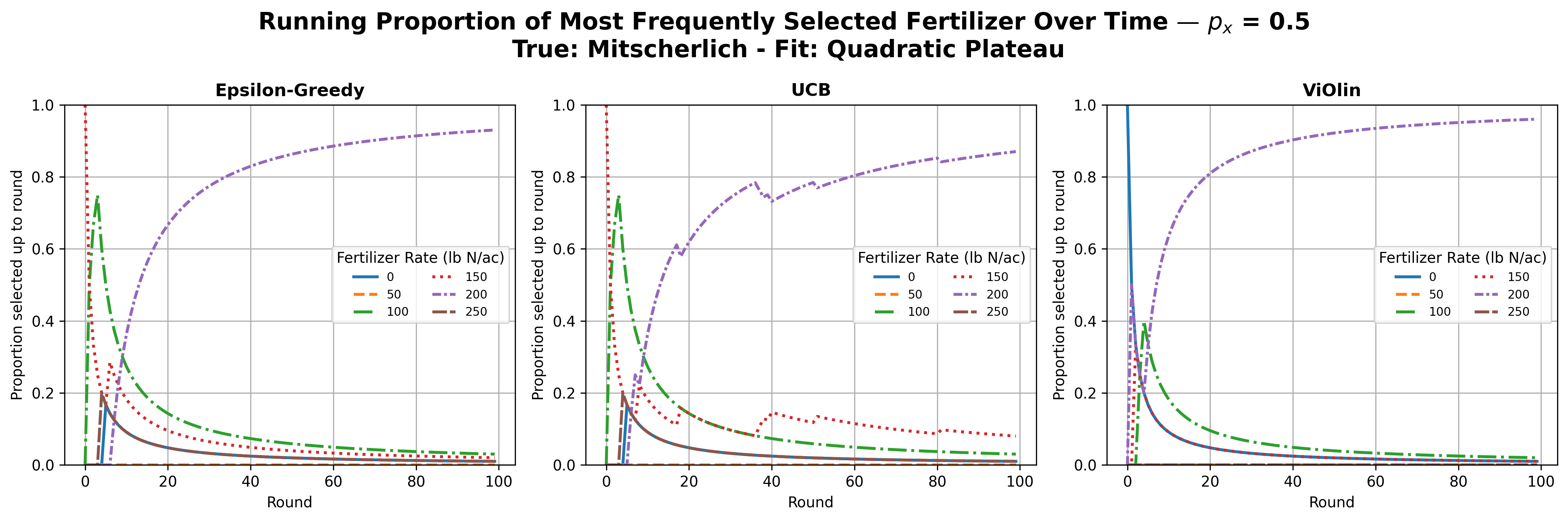}}\\
{(b)\includegraphics[width=0.95\textwidth]{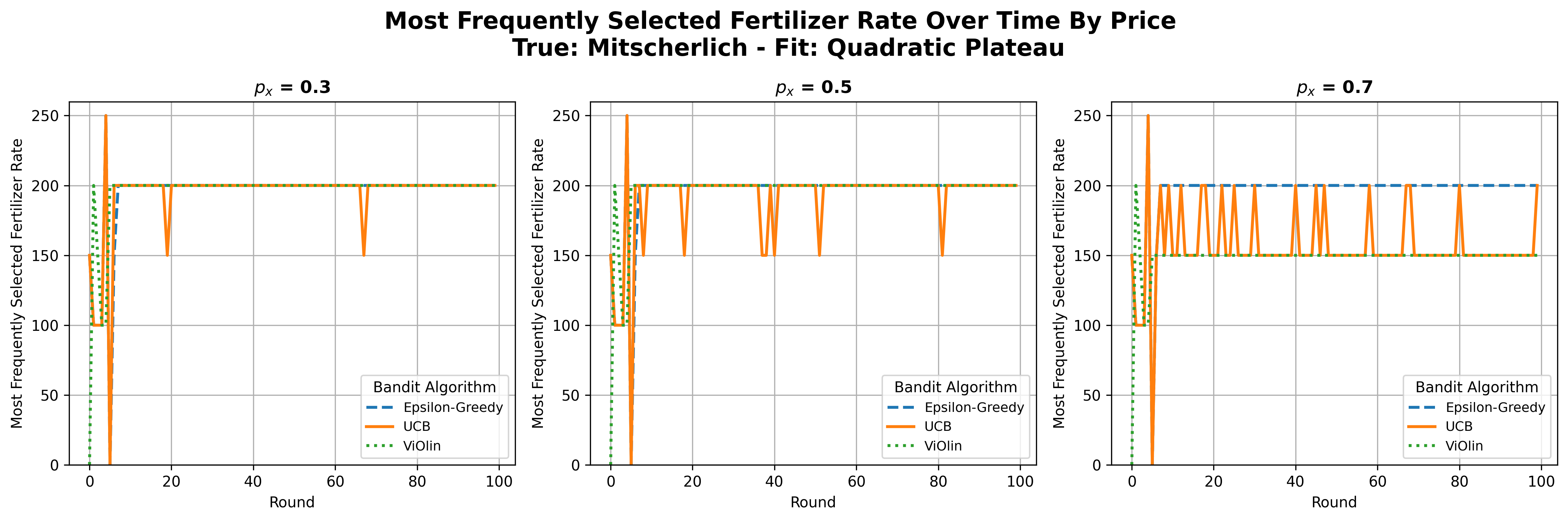}}
\end{tabular}}}
     \caption{\chg{\textbf{Fertilizer-rate choice dynamics under misspecification (true Mitscherlich; fitted quadratic-plateau).}
Actions are chosen from $\mathcal{X}=\{0,50,\ldots,250\}$ lb N/ac over $T=100$ rounds.
\textbf{(a)} For $p_x=\$0.5$/lb N, running selection proportions for each nitrogen rate are shown for $\epsilon$-greedy, nonlinear-UCB, and ViOlin (averaged over 10 replicates).
\textbf{(b)} The most frequently selected nitrogen rate is shown for $p_x\in\{0.3,0.5,0.7\}$ \$/lb N, highlighting how fertilizer price affects the learned policy when the fitted model is imperfect.}}
     \label{fig:arms_bandit_MS}
\end{figure}
Figure~\ref{fig:all_algs_compare_MS} plots the average cumulative regret when the fitted quadratic-plateau and Michaelis Menten model is misspecified with respect to the true Mitscherlich process. Relative to the well-specified case, the model-based regret curves shift upwards, reflecting the price of fitting an approximate response surface.  
Yet their slopes remain moderate, indicating that the fitted quadratic-plateau still guides the search toward profitable regions quickly.  Note, that similar to the well-specified setting, the model-based non-linear bandits perform better than the linear and non-parametric baselines. 
Although the quadratic-plateau form cannot capture the exponential approach to an asymptote exactly, its concave-then-flat shape matches the broad geometry of the Mitscherlich curve.  As a consequence, all three nonlinear model-based algorithms, $\epsilon$-greedy, UCB, and \texttt{ViOlin} continue to accumulate substantially \emph{less} regret than the linear (LinUCB) or non-parametric (kNN-UCB) baselines. These results highlight the resilience of shape-compatible models under misspecification and point to rich opportunities for deeper theoretical and empirical study.  Although \texttt{ViOlin}'s greedier, curvature-penalized score again avoids the large initial dip seen in UCB and $\epsilon$-greedy, note that UCB and $\epsilon$-greedy seem to perform comparably, especially in the Michaelis Menten fit, perhaps suggesting that exploration might help in misspecified settings.  Other details can be found in the Appendix. 

Similar to the well-specified setting, we also plot the running proportion of arms over time and the most frequently selected arm trajectories in Figure~\ref{fig:arms_bandit_MS} (a) and (b) respectively. Although the arm selections remain mostly similar as in the well-specified settings, we do notice that UCB tends to explore more in the misspecified setting (with the same exploration parameter $\alpha = 1$ as in the well-specified setting), which also results in better parameter estimates for UCB as is seen in Figure \ref{fig:param_trajectories_MS}.

These findings underline a pragmatic principle for on-farm experimentation: when the \emph{shape} of the assumed response is qualitatively correct, e.g.\ monotone-increasing with a plateau, model-based bandits remain sample efficient even if the functional form is wrong.  
Conversely, generic linear or purely non-parametric methods may demand far more data before converging, a luxury seldom available to smallholders.  Hence, coupling modest domain knowledge with sequential decision-making promises a robust path toward Sustainable Development Goal~2 by enabling data-scarce producers to improve input efficiency without costly large-scale trials.

\chg{\section{Real data analysis: Optimizing nitrogen rates for corn in the U.S.\ Midwest}
\label{sec:realdata}}
\chg{We evaluate the proposed bandit algorithms on publicly available multi-site corn nitrogen field-trial data from a public--industry partnership spanning the 2014--2016 growing seasons, collected under standardized protocols across U.S.\ Midwest institutions \citep{ransom_data}. 
For reproducibility, we construct an analysis-ready subset by retaining only the variables needed for sequential decision-making (state, site, year, block, nitrogen rate, and yield), removing incomplete records, and harmonizing formats; the processed files  are provided in the supplement and and variable descriptions in Table ~\ref{tab:realdata_columns} of the Appendix.}

\chg{\subsection{Offline replay protocol and profit objective}
The goal is to assess whether bandit-style decision rules can adaptively choose nitrogen rates to improve an economic objective. Because the trials were not run sequentially for online learning, we evaluate policies using an \emph{offline replay} procedure: we treat each \emph{round} as a grouped set of plots sharing the same environment and management conditions, and within that round the available actions are the nitrogen rates observed in that group. These rounds should be interpreted as \emph{decision instances} rather than a true temporal sequence.

\smallskip
\noindent\textbf{Two complementary regimes.}
We report results in two settings that mirror data availability in practice:
(i) a {data-limited case study} at a single site, and (ii) a {pooled multi-site} evaluation that aggregates across similar environments.
Specifically:
\begin{enumerate}
    \item \textbf{Single-site case study (short horizon):} Urbana, IL (2014--2016), with rounds defined as \texttt{(Year, Block)}. With blocks $\in\{1,2,3,4\}$ and years $\in\{2014,2015,2016\}$, this yields $T=12$ rounds. 
In other words, we fix \texttt{State} = IL and \texttt{Site} = Urbana, and treat \texttt{Block} $\in \{1,2,3,4\}$ as geographically similar sub-units within the site observed over \texttt{Year} $\in \{2014,2015,2016\}$, giving $T = 3 \times 4 = 12$ rounds.
    \item \textbf{Pooled low-productivity evaluation (longer horizon):}
We pool all observations labeled \texttt{Site\_Prod = low}, where \texttt{Site\_Prod} is the dataset's within-state productivity label indicating that a site-year falls into the lower-yielding group relative to other sites in the same state (based on the study's baseline/expected-yield information).
Rounds are defined as \texttt{(State, Site, Year, Block)}, yielding a modestly larger horizon (e.g., $T\approx 56$ in our filtered sample; see Table~\ref{tab:lowprod_state_summary} in the Appendix for locations and counts).
Restricting to the low-productivity subset reduces cross-round heterogeneity so that a single \emph{non-contextual} policy is meaningful; otherwise, the profit-maximizing nitrogen rate can vary substantially across environments and would naturally call for a contextual (covariate-dependent) model.
\end{enumerate}

\smallskip
\noindent\textbf{Profit objective.}
In both regimes we optimize
$
\Pi(x) \;=\; p_y\,Y(x) \;-\; p_x\,x,
$
where $Y(x)$ is yield (bu/ac) at nitrogen rate $x$ (lb N/ac), $p_y$ is the corn price (\$/bu), and $p_x$ is the nitrogen cost (\$/lb N). 
We use year-specific corn prices from publicly available summary data \citep{USDA_NASS_CropValues_2016_AnnualSummary} and compute $p_x$ from December Midwest urea prices using urea's 46\% N analysis \citep{Yara_Urea_46_0_0_Label}; the resulting year-by-year values are summarized in Table~\ref{tab:realdata_prices_urea} in the Appendix.

\noindent\textbf{Bandit feedback, oracle benchmark, and regret.}
Within each round, multiple nitrogen-rate treatments are observed (with replicated plots), but the policy selects a {single} rate $x_t$ and the replay reveals only the corresponding realized outcome (mean over replicates at $x_t$, i.e., {mean-reward} mode), thereby emulating bandit feedback.
The underlying experiment that generated the original data uses a randomized complete block design (RCBD), so outcomes for all treatments are available within each round for evaluation \citep{ransom_data}. 
We define the within-round oracle action (best mean profit among the observed rates in that round)
\[
x_t^\star \in \arg\max_{x \in \mathcal{X}_t}\ \overline{\Pi}_t(x),
\]
where $\mathcal{X}_t$ is the set of observed nitrogen rates in round $t$ and $\overline{\Pi}_t(x)$ is mean profit across replicates at rate $x$.
We then compute per-round regret as $r_t=\overline{\Pi}_t(x_t^\star)-\Pi_t(x_t)$ (unit in \$/ac) and cumulative regret $R_t=\sum_{s=1}^t r_s$ (units in \$/ac), reporting both $R_t$ and $R_t/t$.

To reduce sensitivity to a single ordering, we repeat the offline replay over 100 randomized orderings; for the pooled analysis we sort rounds by year and shuffle within year, then average curves across replays. }

\chg{\subsection{Results}
Figure~\ref{fig:offline_realdata_profit_urea_quadplateau}(a)-(d) summarize offline replay performance in terms of \emph{cumulative regret} $R_t=\sum_{s=1}^t r_s$ and \emph{average regret} $R_t/t$, where $r_t=\overline{\Pi}_t(x_t^\star)-\Pi_t(x_t)$ is the within-round profit gap to the oracle benchmark as defined above. 
We report mean trajectories over 100 replay replications, with pointwise 95\% confidence bands.

\noindent\textbf{Algorithms compared.}
In the single-site case study, we compare the three nonlinear model-based strategies with profit maximization objective, i.e., model-based $\epsilon$-greedy (Alg \ref{alg:eps}), nonlinear-UCB (Alg. \ref{alg:ucb}), and \texttt{ViOlin} (Alg. \ref{alg:ViOlin_profit}), and nonparametric k-NN UCB (Alg. \ref{alg:knnucb}) against a {random} baseline that selects a nitrogen rate uniformly at random from the rates available in each round.
For the pooled low-productivity analysis, we additionally include LinUCB (Alg. \ref{alg:linucb}) and kNN-UCB (Alg. \ref{alg:knnucb}) as parametric and nonparametric benchmarks, respectively.
In all model-based methods, the fitted reward model is the quadratic-plateau response with the profit objective $\Pi(x)=p_yY(x)-p_xx$ (Section~\ref{sec:profit_objective}), and regret is measured in {profit units} (\$/ac).

\paragraph{Urbana, IL (short horizon; $T=12$ rounds).}
In the data-limited Urbana, IL case study (rounds defined by year $\times$ block), the nonlinear model-based algorithms consistently reduce regret relative to the random baseline (Figure~\ref{fig:offline_realdata_profit_urea_quadplateau}(a)-(b)). 
Both the cumulative regret curves (left) and the corresponding average regret $R_t/t$ curves (right) show that leveraging agronomically motivated nonlinear structure can be substantially more sample-efficient when only a handful of decisions are available.

\paragraph{Low-productivity sites across states (longer horizon; $T\approx 56$ rounds).}
When pooling all low-productivity site-years across states (rounds defined by state $\times$ site $\times$ year $\times$ block), we observe a different pattern (Figure~\ref{fig:offline_realdata_profit_urea_quadplateau}(c)-(d)). 
After an initial transient, the nonparametric kNN-UCB becomes increasingly competitive and can outperform the parametric baselines, while LinUCB degrades under cross-site heterogeneity. 
This is consistent with the fact that a single global linear approximation is poorly suited to multi-environment (different locations) response variability, whereas local smoothing in action space can provide robustness once sufficient history accumulates.

\noindent
 Taken together, these results mirror the main message from the simulations while highlighting a practical distinction between regimes: under short horizons, nonlinear model-based bandits provide clear benefits over uninformed baselines, whereas under longer horizons with heterogeneous pooled environments, flexible nonparametric benchmarks can become competitive. Extending to mixed productivity regimes is natural future work using contextual bandits with site and soil covariates.}
\begin{figure}[t]
\centering

\begin{subfigure}[t]{0.48\linewidth}
\centering
\includegraphics[width=\linewidth]{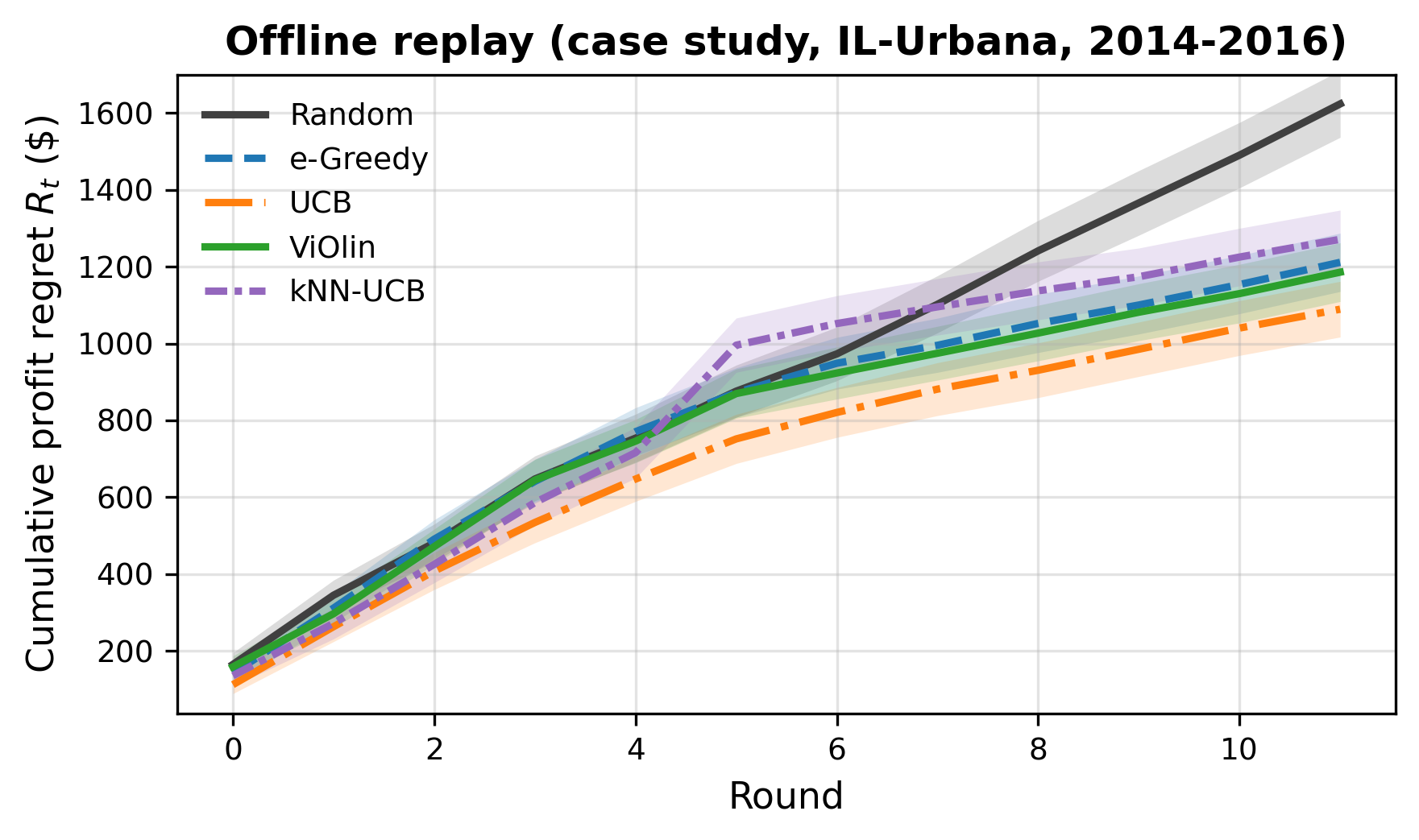}
\caption{Urbana, IL (2014--2016): cumulative regret $R_t$}
\label{fig:realdata_urbana_Rt}
\end{subfigure}\hfill
\begin{subfigure}[t]{0.48\linewidth}
\centering
\includegraphics[width=\linewidth]{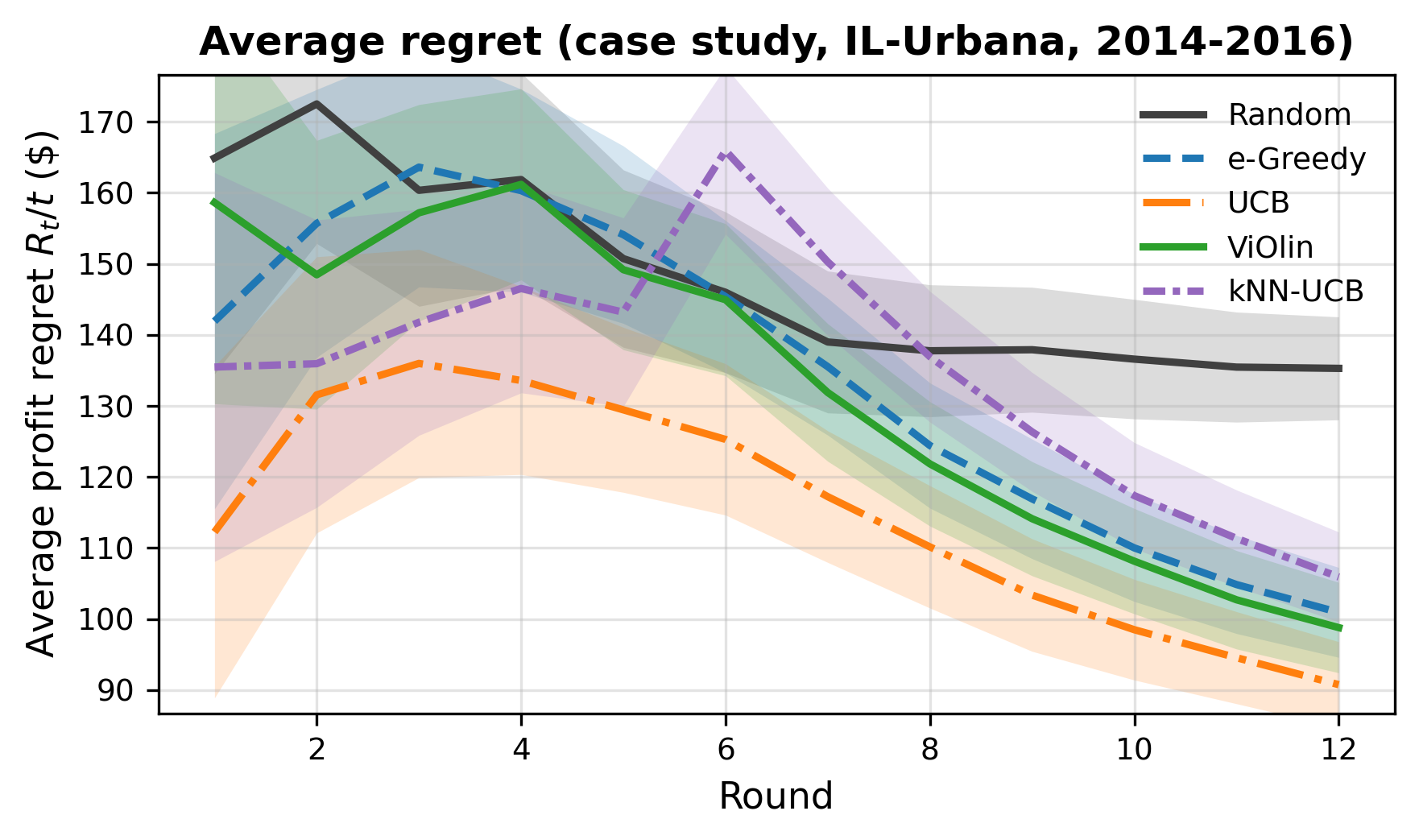}
\caption{Urbana, IL (2014--2016): average regret $R_t/t$}
\label{fig:realdata_urbana_Rt_over_t}
\end{subfigure}

\vspace{0.4em}

\begin{subfigure}[t]{0.48\linewidth}
\centering
\includegraphics[width=\linewidth]{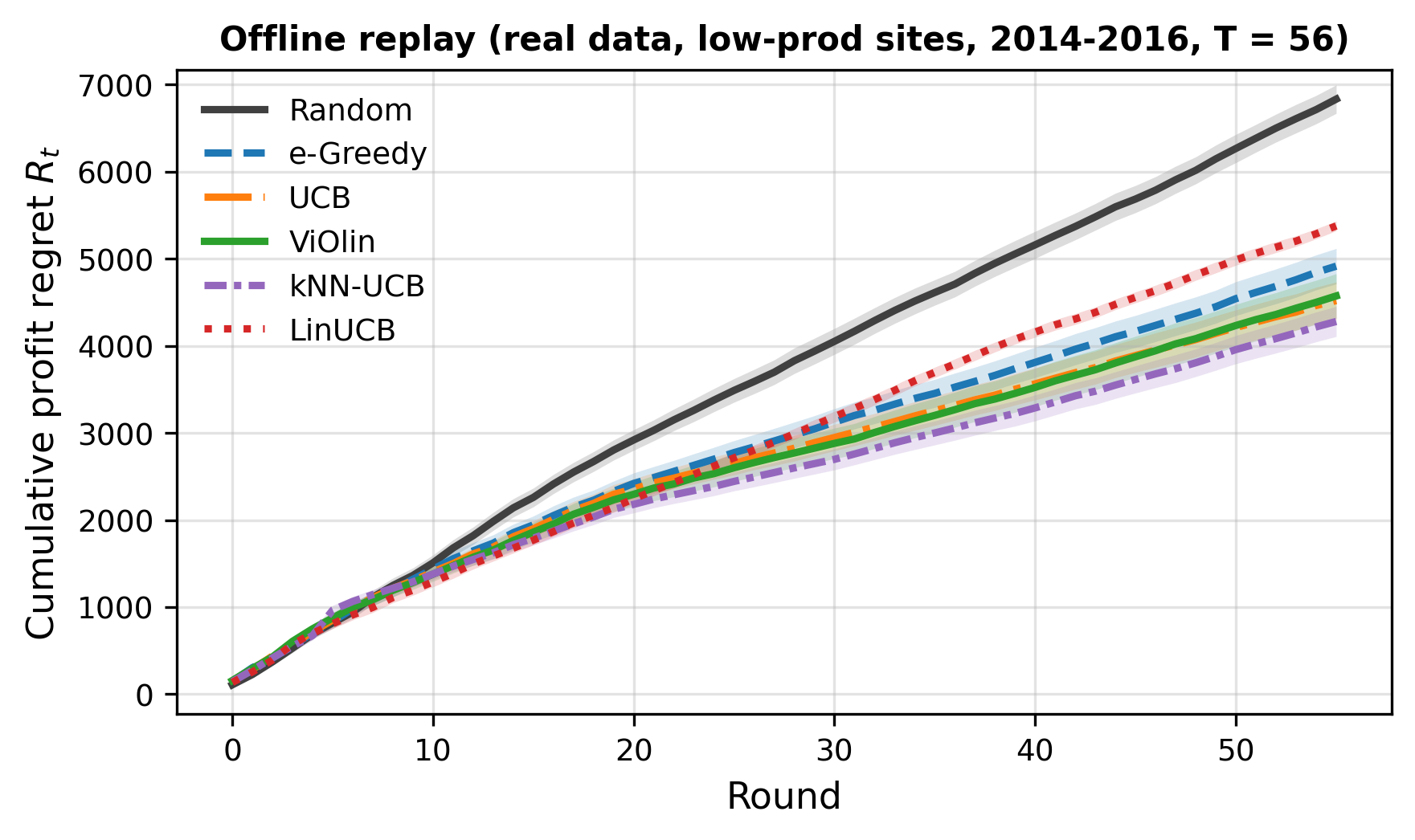}
\caption{Low-productivity sites: cumulative regret $R_t$}
\label{fig:realdata_lowprod_Rt}
\end{subfigure}\hfill
\begin{subfigure}[t]{0.48\linewidth}
\centering
\includegraphics[width=\linewidth]{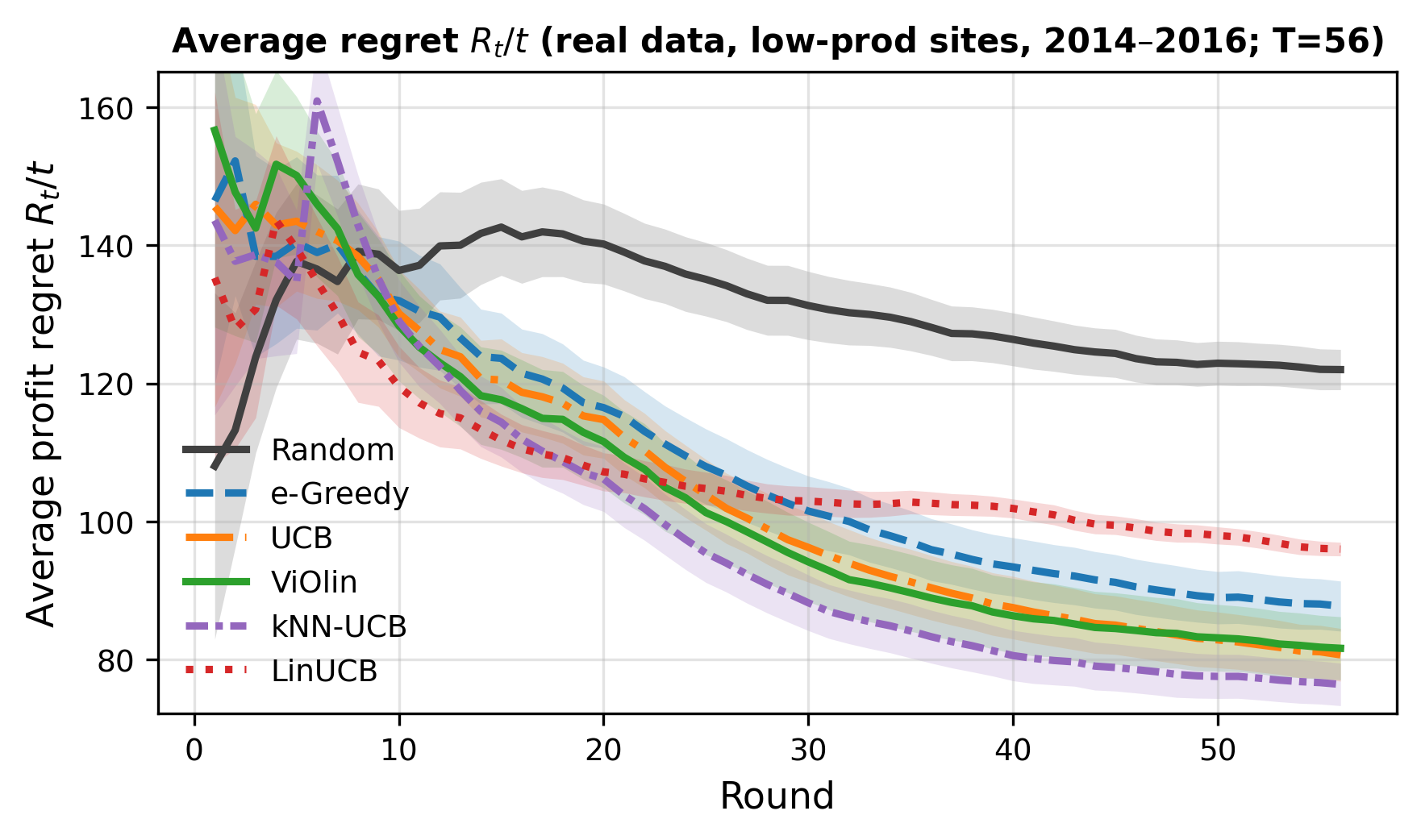}
\caption{Low-productivity sites: average regret $R_t/t$}
\label{fig:realdata_lowprod_Rt_over_t}
\end{subfigure}

\caption{\chg{\textbf{Offline replay on real multi-site corn nitrogen trials (2014--2016): profit regret under a quadratic-plateau model.}
Reward is per-acre profit $\Pi(x)=p_y Y(x)-p_x x$, where $p_y$ is the corn price (USDA--NASS Crop Values) and $p_x$ is the nitrogen cost as in Table \ref{tab:realdata_prices_urea}. Panels (a)--(b) show the data-limited Urbana, IL case study with rounds defined by (Year, Block) ($T=12$). Panels (c)--(d) show the pooled low-productivity evaluation with rounds defined by (State, Site, Year, Block) (longer horizon). Curves are means over replay replications; shaded bands are pointwise 95\% confidence intervals.}}
\label{fig:offline_realdata_profit_urea_quadplateau}
\end{figure}

\chg{\section{Conclusion and discussion}\label{sec:conclusion}
We developed and evaluated nonlinear, model-based bandit algorithms for sequential fertilizer decision-making using agronomy-standard mechanistic yield--nitrogen response families (Mitscherlich, quadratic-plateau, Michaelis--Menten, and logistic). By framing nitrogen-rate selection as an online learning problem, these models become interpretable decision rules that update as data accrue. Our primary objective is economic: we optimize profit (revenue minus fertilizer cost) rather than yield alone, so recommendations directly reflect realistic input-cost trade-offs. To our knowledge, our study provides one of the first systematic comparisons in the data-limited regime demonstrating when agronomy-standard nonlinear response modeling embedded in bandit decision rules yields clear advantages over linear and model-free baselines for profit-oriented nitrogen recommendations made sequentially across seasons.

In simulations, nonlinear model-based bandits achieved substantially lower profit regret than linear and nonparametric baselines in well-specified settings, highlighting the value of incorporating domain structure in small-sample regimes. Under misspecification, performance degrades as expected, but mechanistic models remain competitive when they provide a reasonable approximation and data are limited. We also included an offline replay case study on publicly available multi-site corn nitrogen trials (\citet{ransom_data}), which corroborates the practical message: in short horizons, nonlinear model-based strategies can be more sample-efficient than uninformed baselines, while in longer pooled (multiple locations) regimes, nonparametric methods can become competitive.

Several directions merit further work: (i) principled model selection/averaging within the bandit loop to improve robustness to misspecification; (ii) extensions to spatially varying or multi-dimensional decisions (e.g., variable-rate application); (iii) contextual and nonstationary formulations that incorporate soil/weather covariates and season-to-season shifts; and (iv) tighter regret guarantees under structured misspecification and extensions to delayed or partially observed outcomes.}

\section*{Reproducibility statement}
All simulations were implemented in \texttt{Python 3.11} using standard scientific libraries (\texttt{NumPy}, \texttt{SciPy}, \texttt{Matplotlib}). We report the data-generating processes, parameter values, and hyperparameters in full within the paper and Appendix, including initial values used for nonlinear least-squares estimation, fertilizer and grain price settings, and noise levels. Each experiment was repeated independently across 10 replicates with fixed randomization schemes to assess variability.  The code to reproduce all simulation experiments and figures, together with the analysis-ready files used for the offline real-data replay evaluation, are provided in the Supplementary Material and will be made publicly available on GitHub upon publication.

\section*{Data availability statement}
No new data were collected for this study. The real-data analysis uses the publicly available multi-site corn nitrogen trial dataset of \citet{ransom_data}. 

\section*{Ethics statement}
This study does not involve human participants, human data, or animals, and therefore ethics approval was not required.

\section*{Funding statement}
This research received no specific grant from any funding agency in the public, commercial, or not-for-profit sectors.

\section*{Conflict of interest}
The authors declare that there are no conflicts of interest regarding the publication of this paper.
\bibliography{main}
\section*{Appendix}
\appendix
\section{Rademacher complexity for parametric nonlinear function classes} \label{sec: proof_sample_complexity}
\begin{proof}[Proof of Theorem \ref{thm: SeqRademacherForBoundedFunctions}]
We follow the approach of \cite{rakhlin2015sequential}. Let $N_1(\alpha, \mathcal{F}, T)$ denote the sequential covering number in the $\ell_1$-norm at scale $\alpha$. From equation (9) of their paper, we have:
\begin{align}
    \mathfrak{R}_T(\mathcal{F}) \leq \alpha + \sqrt{\frac{2 \log N_1(\alpha, \mathcal{F}, T)}{T}}. \label{eq: seq_Rademacher}
\end{align}

Using Corollary 1 from the same paper, the $\ell_\infty$-sequential covering number satisfies:
\[
N_\infty(\alpha, \mathcal{F}, T) \leq \left( \frac{2eT}{\alpha} \right)^{\mathrm{fat}_\alpha(\mathcal{F})},
\]
and since $N_1(\alpha, \mathcal{F}, T) \leq N_\infty(\alpha, \mathcal{F}, T)$, we obtain:
\begin{align}
   \log N_1(\alpha, \mathcal{F}, T) \leq \mathrm{fat}_\alpha(\mathcal{F}) \cdot \log \left( \frac{2eT}{\alpha} \right). \label{eq: N1_bound} 
\end{align}

For any class $\mathcal{F} \subseteq [-B_\mathcal{F}, B_\mathcal{F}]^{\mathcal{X}}$, we have from Theorem 10 of \cite{rakhlin2015online}:
\begin{align}
 \mathrm{fat}_\alpha(\mathcal{F}) \leq \left\lceil \frac{2B_\mathcal{F}}{\alpha} \right\rceil. \label{eq: fatF}   
\end{align}

Now, using the bounds in \eqref{eq: N1_bound} and \eqref{eq: fatF} in \eqref{eq: seq_Rademacher} and choosing $\alpha = \frac{B_\mathcal{F}}{\sqrt{T}}$, we get:
\[
\left\lceil \frac{2B_\mathcal{F}}{\alpha} \right\rceil = \left\lceil 2\sqrt{T} \right\rceil \leq 3\sqrt{T},
\quad
\log \left( \frac{2eT}{\alpha} \right) = \log \left( \frac{2e T^{3/2}}{B_\mathcal{F}} \right) = \mathcal{O}(\log T).
\]

Thus,
\[
\mathfrak{R}_T(\mathcal{F}) \leq \frac{B_\mathcal{F}}{\sqrt{T}} + \sqrt{ \frac{6 \sqrt{T} \cdot \log T}{T} } = \frac{B_\mathcal{F}}{\sqrt{T}} + \sqrt{ \frac{6 \log T}{\sqrt{T}} }.
\]

The dominant term is $\frac{B_\mathcal{F} \sqrt{\log T}}{\sqrt{T}}$, so we conclude:
\[
\mathfrak{R}_T(\mathcal{F}) \leq C \cdot B_\mathcal{F} \cdot \sqrt{ \frac{\log T}{T} }
\]
for some universal constant $C$.
\end{proof}

\chg{\subsection{Computation of the model-dependent bound $B_{\mathcal F}$}\label{app:BF}

Figure~\ref{fig:seq_rad_complexity} visualizes the upper bound in Theorem~2 using a model-dependent magnitude bound $B_{\mathcal F}$.
For each yield-response family $\mathcal F$, we define
\[
B_{\mathcal F} := \sup_{x\in[0,250]} |f(x; \theta)|,
\]
where the domain $[0,250]$ matches the fertilizer range used in the simulation study (Section~5) and $f(\cdot; \theta)$ is evaluated under the parameter settings used in that section.

\paragraph{Example: quadratic-plateau.}
For the quadratic-plateau model
\[
Y(x)=
\begin{cases}
a+bx+cx^2, & x\le x_0,\\
a+bx_0+cx_0^2, & x>x_0,
\end{cases}
\]
with $(a,b,c,x_0)=(80,1.2,-0.003,180)$ and $c<0$, the quadratic component is concave on $[0,x_0]$ and the function is constant for $x>x_0$. Hence the maximum on $[0,250]$ occurs at $x=x_0$, and
\[
B_{\mathrm{QP}}=\sup_{x\in[0,250]}f(x)=f(x_0)
=80+1.2(180)-0.003(180)^2=198.8\ \text{bu/ac}.
\]

\paragraph{Values used in Figure~\ref{fig:seq_rad_complexity}.}
For the other model families (Mitscherlich, Logistic, and Michaelis--Menten), $f(x)$ is monotone increasing on $[0,250]$ under the parameters used in Section~5, so $B_{\mathcal F}$ is attained at $x=250$. Table~\ref{tab:BF_values} lists the resulting values.

\begin{table}[h]
\centering
\caption{Model-dependent bounds $B_{\mathcal F}=\sup_{x\in[0,250]}|f(x)|$ (bu/ac) used in Figure~\ref{fig:seq_rad_complexity}, computed under the parameter settings in Section~5.}
\label{tab:BF_values}
\begin{tabular}{l c}
\toprule
Model family $\mathcal F$ & $B_{\mathcal F}$ (bu/ac)\\
\midrule
Quadratic-plateau & 198.80 \\
Mitscherlich & 197.18 \\
Logistic & 189.77 \\
Michaelis--Menten & 167.14 \\
\bottomrule
\end{tabular}
\end{table}}

\section{Finding optimal profit maximizing arms for different non-linear models}
Recall that $\Pi(x)$ expresses the trade-off between revenue from yield and the cost of input.
To find the economically optimal fertilizer dose $x^*$, we differentiate the profit function with respect to $x$:
Below, we derive $x^*$ for each model considered in this work.

\paragraph{Mitscherlich Model: }
The Mitscherlich yield response is
\[
    Y(x) = A(1 - e^{-b x}),
\]
so the profit function is
\[
    \Pi(x) = p_y\, A(1 - e^{-b x}) - p_x x.
\]
The optimal dose $x^*$ solves
\[
    \frac{d\Pi}{dx} = p_y\, A\, b\, e^{-b x} - p_x = 0,
\]
giving
\[
    x^* = -\frac{1}{b} \ln\left( \frac{p_x}{p_y A b} \right).
\]
\textit{Note:} If $\frac{p_x}{p_y A b} \geq 1$, then $x^* = 0$.

\paragraph{Quadratic Model with Threshold}
Suppose
\[
    Y(x) = \begin{cases}
        a + b x + c x^2, & x \leq x_0 \\
        a + b x_0 + c x_0^2, & x > x_0
    \end{cases}
\]
with threshold $x_0$. Then
\[
    \Pi(x) = p_y Y(x) - p_x x.
\]
For $x \leq x_0$, set the derivative to zero:
\[
    \frac{d\Pi}{dx} = p_y (b + 2c x) - p_x = 0
    \implies x^* = \frac{1}{2c} \left( \frac{p_x}{p_y} - b \right).
\]
Thus,
\[
    x^* = \min\left\{ x_0,\, \max\left\{ 0,\, \frac{1}{2c} \left( \frac{p_x}{p_y} - b \right) \right\} \right\}.
\]
If $x > x_0$, profit decreases with increasing $x$.

\paragraph{Michaelis-Menten (MM) Model}
With
\[
    Y(x) = \frac{a x}{b + x},
\]
the profit is
\[
    \Pi(x) = p_y\, \frac{a x}{b + x} - p_x x.
\]
Setting the derivative to zero,
\[
    \frac{d\Pi}{dx} = p_y\, \frac{a b}{(b + x)^2} - p_x = 0
    \implies (b + x)^2 = \frac{a b p_y}{p_x}
\]
\[
    x^* = \sqrt{ \frac{a b p_y}{p_x} } - b.
\]
We set $x^* = 0$ if the right side is negative.

\paragraph{Logistic Model}
For the logistic response
\[
    Y(x) = \frac{A}{1 + e^{-B(x - C)}},
\]
profit is
\[
    \Pi(x) = p_y\, \frac{A}{1 + e^{-B(x - C)}} - p_x x.
\]
Let $u = e^{-B(x - C)}$, so $x = C - \frac{1}{B} \ln u$. Setting the derivative to zero,
\[
    p_y\, A\, \frac{B u}{(1 + u)^2} = p_x,
\]
which yields a quadratic in $u$:
\[
    u^2 + \left(2 - \frac{B p_y A}{p_x}\right) u + 1 = 0.
\]
Let
\[
    \gamma = \frac{B p_y A}{p_x}.
\]
The positive root is
\[
    u^* = \frac{ \gamma - 2 - \sqrt{ (\gamma - 2)^2 - 4 } }{2 },
\]
and the optimal dose is
\[
    x^* = C - \frac{1}{B} \ln u^*.
\]
If $u^* \leq 0$ or $x^* < 0$, set $x^* = 0$.

\section{Baseline Algorithm descriptions}
We compare the non-linear model based algorithms to the nonparametric and linear baselines. Specifically, we compare these with the LinUCB and $k$-Nearest Neighbor UCB algorithm. For the sake of completeness, we describe these algorithms below. 

The \textbf{LinUCB algorithm} (Algorithm~\ref{alg:linucb}) is a classic model-based bandit method that assumes a linear relationship between the reward and the action (or its features). At each round, LinUCB fits a linear regression model to the observed data (line~7) and predicts the expected yield for each fertilizer rate (line~8a). It computes an upper confidence bound (UCB) for each action by adding a model-based uncertainty term, proportional to the standard error of the prediction, to the profit estimate (lines~8b-d). The fertilizer rate with the highest UCB is selected for the next trial (line~9), balancing exploration of uncertain arms and exploitation of high-yield arms. While LinUCB is efficient and easy to implement, its performance can degrade if the true reward function is nonlinear or misspecified.

\vspace{1em}

The \textbf{kNN-UCB algorithm} (Algorithm~\ref{alg:knnucb}) is a nonparametric bandit method that does not assume a specific model for the yield response. Instead, for each candidate fertilizer rate, it estimates the expected yield by averaging the observed yields of the $k$ most similar (nearest) rates previously tried (line~8). The uncertainty in this prediction is quantified by the sample variance among these $k$ neighbors (line~9). As with LinUCB, an upper confidence bound is constructed for each candidate rate (line~11), and the rate with the highest UCB is selected (line~13). Although kNN-UCB is robust to model misspecification and can capture complex nonlinearities, it generally requires more data to achieve accurate estimates, particularly when the action space is large or sparsely explored.

\subsection{Toy Illustration of Algorithm Behavior} \label{sec:toy_illustration}

To make the exploration--exploitation logic more tangible, we provide a simple
illustration using a Mitscherlich yield function
\[
  Y(x) \;=\; 80 + 120 \, \big( 1 - e^{-0.015x} \big),
\]
with crop price $p_y = \$5$ per unit yield (bu$^{-1}$), fertilizer cost $p_x = \$0.7$ per unit (lb$^{-1}$ N),
and action grid $\{0, 50, 100, 150, 200\}$ lb N /ac. Table~\ref{tab:toy-demo} shows
four rounds of decision making under three algorithms: $\epsilon$-Greedy,
UCB, and \texttt{ViOlin}.

\begin{table}[h!]
\centering
\caption{Toy demonstration of algorithm behavior (profit in \$).}
\label{tab:toy-demo}
\begin{tabular}{c|cc|cc|cc}
\hline
Round &
\multicolumn{2}{c|}{$\varepsilon$-Greedy} &
\multicolumn{2}{c|}{UCB} &
\multicolumn{2}{c}{ViOlin} \\
& Action & Profit & Action & Profit & Action & Profit \\
\hline
1 & Explore 50  & 541 & Explore 100 & 610 & Explore 50  & 540 \\
2 & Explore 150 & 650 & Explore 150 & 648 & Exploit 150 & 649 \\
3 & Exploit 150 & 648 & Exploit 150 & 647 & Exploit 150 & 648 \\
4 & Explore 50  & 541 & Explore 200 & 645 & Exploit 150 & 647 \\
\hline
\end{tabular}
\end{table}

This toy comparison highlights the distinct decision logics:
\begin{itemize}
  \item {$\epsilon$-Greedy} alternates between random exploration and
  exploiting the currently estimated best dose.
  \item {UCB} incorporates an explicit uncertainty bonus, occasionally
  sampling higher but uncertain doses (round~4).
  \item \textbf{\texttt{ViOlin}} emphasizes curvature and stability, quickly converging
  to the plateau dose and avoiding unnecessary exploration.
\end{itemize}

Even in this simple 4-round setting, the three strategies exhibit clearly
different exploration styles. Such illustrations help convey to applied readers
that algorithm choice affects not only long-run performance, but also the
sequence of recommendations farmers may observe in practice.
\section{Additional Simulation Results}
\label{sec: additional_sims}
In this section, we present the simulation results for other non-linear models that we considered in the simulation experiment of Section \ref{sec: simulation}. First for the well-specified setting, we present the results for the Logistic and Mitscherlich model that was fit instead of quadratic plateau model. 
For reproducibility, Table~\ref{tab:init_values} reports the true parameter values used to generate data in our simulations, alongside the initial values supplied to the nonlinear least-squares estimation routines. These initialization choices were selected to be reasonably close to the true parameters, reflecting what might be obtained from domain knowledge or prior agronomic experiments. As discussed in Section~\ref{sec: simulation}, good initialization improves convergence in small-sample regimes, though all algorithms eventually stabilize near the true values as data accumulate.

\begin{table}[h!]
\centering
\caption{True and initial parameter values used in simulation studies.}
\label{tab:init_values}
{\resizebox{.97\textwidth}{!}{\begin{tabular}{|l|l|l|}
\hline
\textbf{Model} & \textbf{True Parameters} & \textbf{Initial Values} \\
\hline
Mitscherlich & $A=120,\ b=0.015,\ d=80$ & $A=100,\ b=0.01,\ d=75$ \\
Quadratic Plateau & $a=80,\ b=1.2,\ c=-0.003,\ x_0=180$ & $a=75,\ b=1.0,\ c=-0.002,\ x_0=160$ \\
Michaelis--Menten & $a=150,\ b=100,\ d=60$ & $a=120,\ b=80,\ d=50$ \\
Logistic & $A=120,\ B=0.05,\ C=125,\ d=70$ & $A=100,\ B=0.03,\ C=100,\ d=65$ \\
\hline
\end{tabular}}}
\end{table}
\paragraph{Well-specified setting: Additional results}
The same hyperparameter choices were made for the algorithms as in the well-specified setting in Section \ref{sec: simulation}. Namely, we use $\epsilon_t = t^{-1.5}$ to encourage exploitation sooner than later, $\alpha = 1$ for UCB and $\alpha_1 = 2, \alpha_2 = 640$ for \texttt{ViOlin}. Each algorithm was run 10 times and $T = 30$ to emulate data-limited setting. 
\begin{figure}[H]
    \centering
    \includegraphics[width=0.4\linewidth]{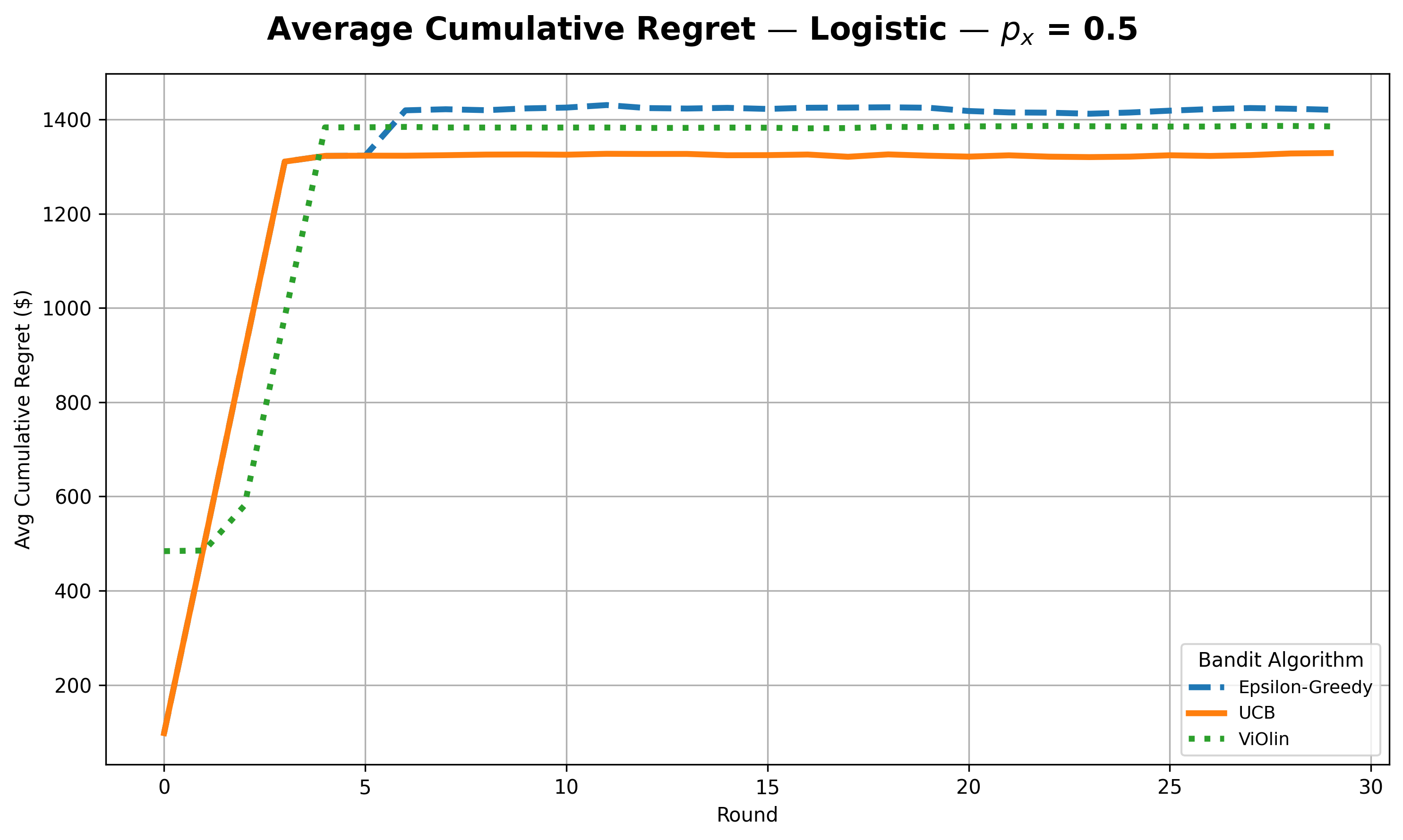}
      \includegraphics[width=0.4\linewidth]{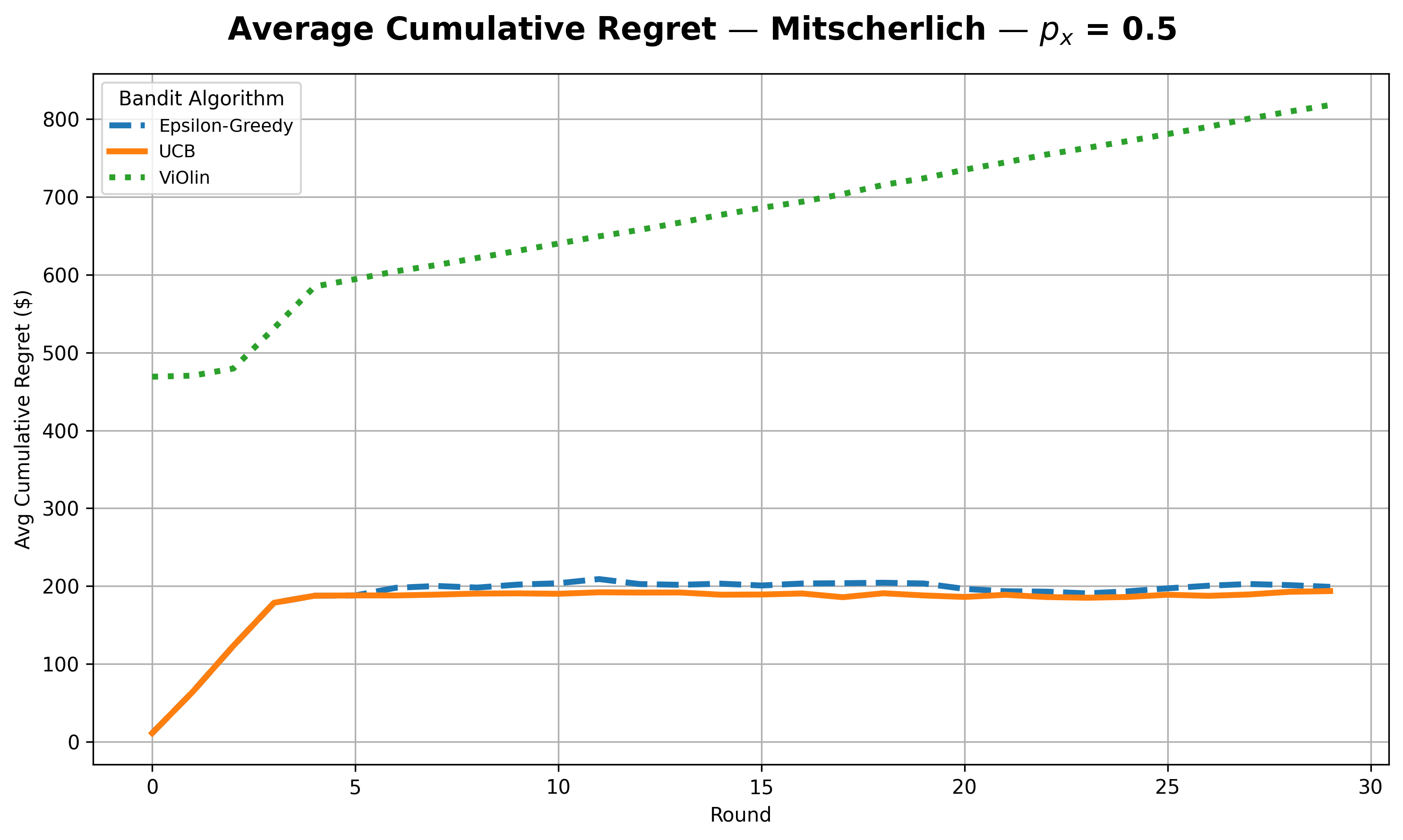}
\caption{\chg{\textbf{Additional well-specified results ($p_x=\$0.5$/lb N).}
Average cumulative profit regret (\$/ac) over $T=30$ rounds for the logistic (left) and Mitscherlich (right) response models, comparing model-based $\epsilon$-greedy ($\epsilon_t=t^{-1.5}$), nonlinear-UCB ($\alpha=1$), and \texttt{ViOlin} ($\alpha_1=2,\alpha_2=640$). Curves are averaged over 10 independent simulation runs.}}
    \label{fig:regret_WS_appendix}
\end{figure}
\begin{figure}[h!] 
\centering
 {\resizebox{.97\textwidth}{!}{
\begin{tabular}{c}
{(a) \includegraphics[width=0.85\linewidth]{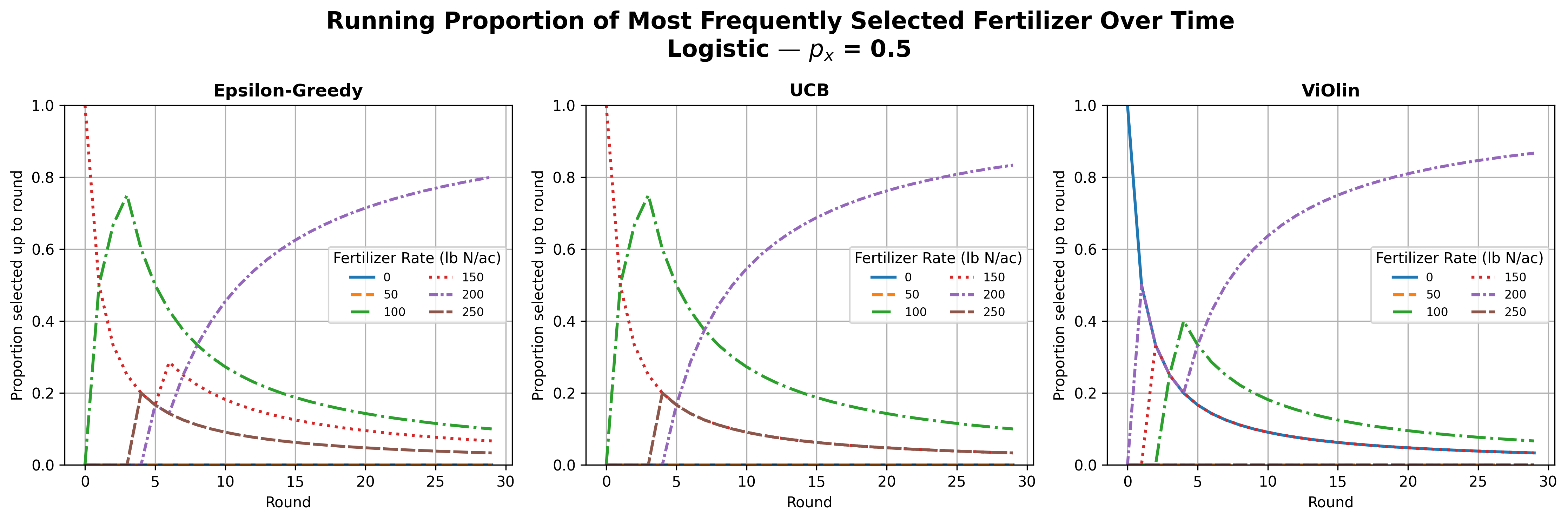}}\\
{(b)\includegraphics[width=0.85\textwidth]{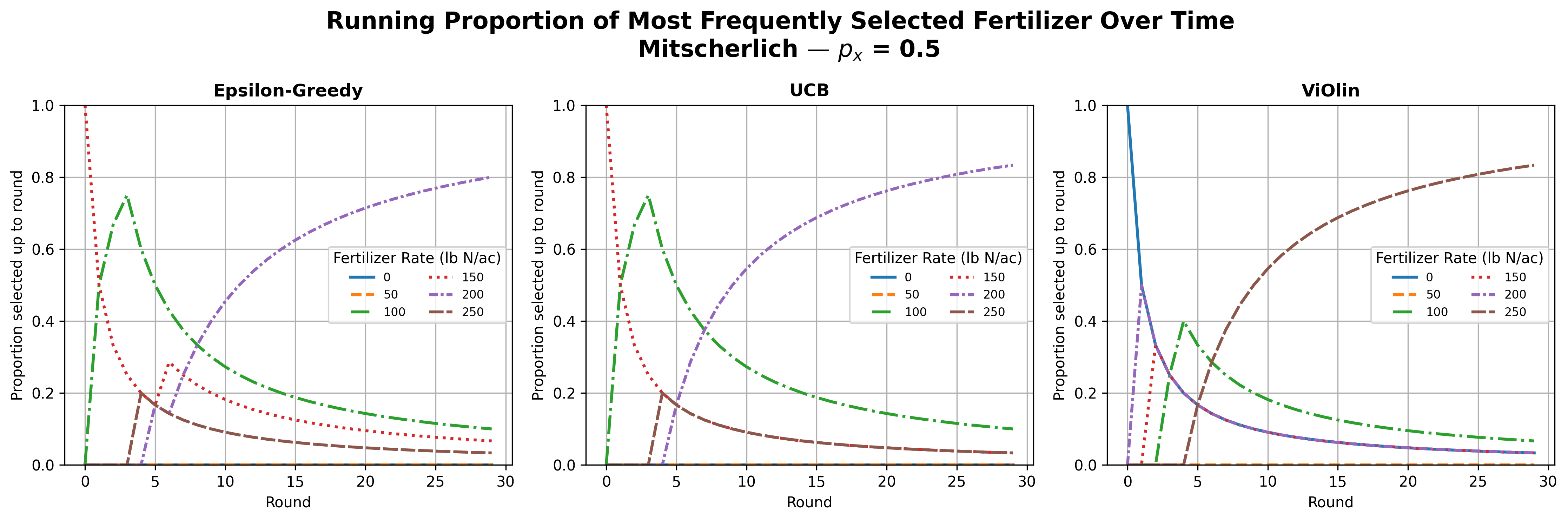}}
\end{tabular}}}
\caption{\chg{\textbf{Fertilizer-rate choice dynamics in the well-specified setting ($p_x=\$0.5$/lb N).}
Running selection proportions over $T=30$ rounds for the most frequently selected nitrogen rates (lb N/ac) under $\epsilon$-greedy, nonlinear-UCB, and \texttt{ViOlin}. 
\textbf{(a)} Logistic response; \textbf{(b)} Mitscherlich response.}}
         \label{fig:arms_bandit_WS_appendix}
\end{figure}
It can be noted from Figure \ref{fig:arms_bandit_WS_appendix} that while the arm choices remain mostly similar in the Logistic model to that of Quadratic plateau model, for the Mitscherlich model the \texttt{ViOlin} algorithm seems to choose $x = 250$ ln N/ac more than other choices, which seems also reflect in the worsened regret for \texttt{ViOlin} in Figure \ref{fig:regret_WS_appendix}. We hypothesize that this is because, as the yield curve quickly saturates, the local curvature used for exploration in \texttt{ViOlin} becomes very small, causing the algorithm to underestimate uncertainty and prematurely stop exploring. Also, we notice greater variability in the estimation of parameters especially in the Logistic model as can be seen in Figure \ref{fig:param_trajectories_WS_logistic}.

\begin{figure}[H]
    \centering
    \includegraphics[width=0.55\linewidth]{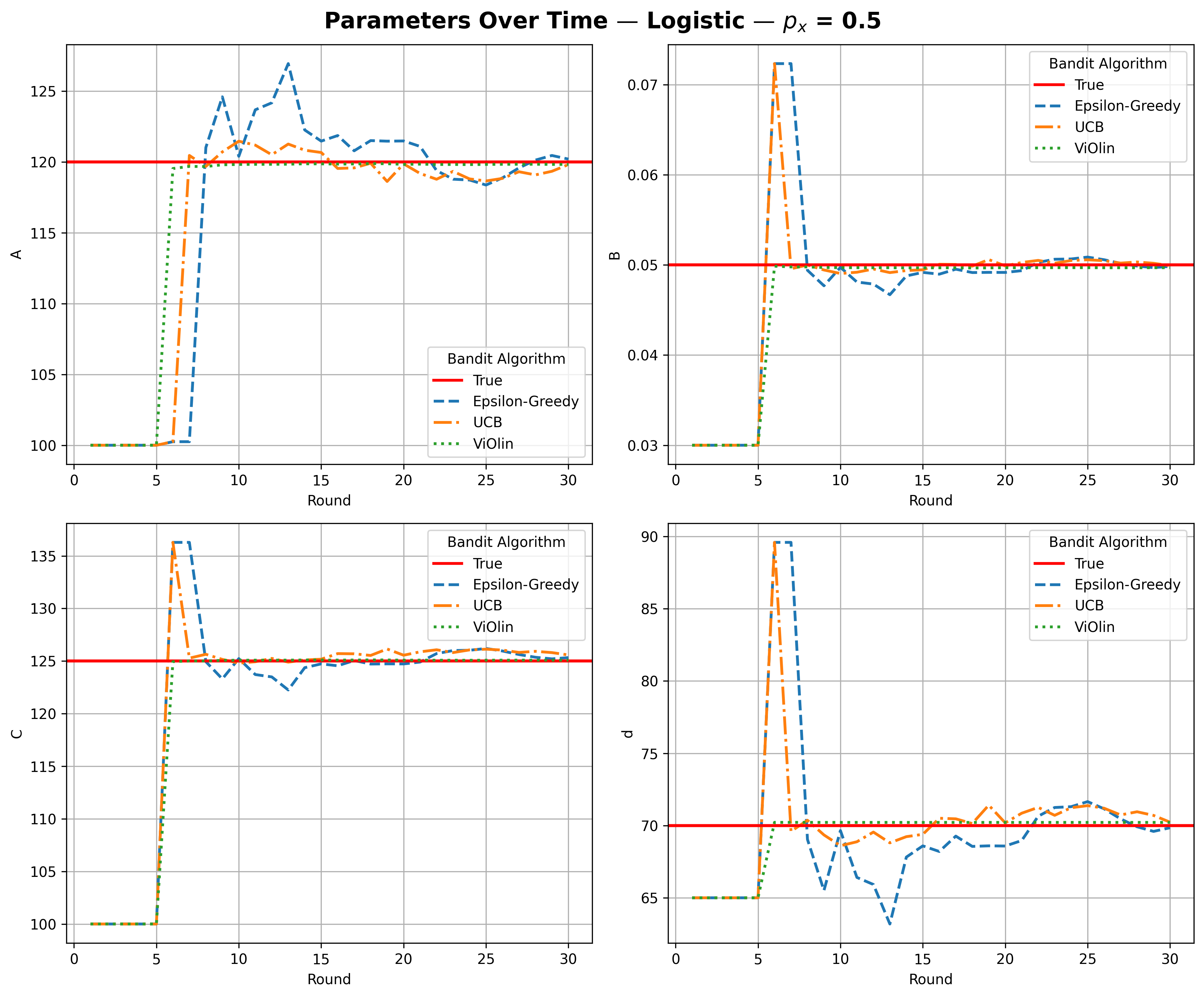}
    \caption{\chg{\textbf{Parameter-learning trajectories (well-specified logistic model; $p_x=\$0.5$/lb N).}
Estimated logistic parameters over $T=30$ rounds for model-based $\epsilon$-greedy ($\epsilon_t=t^{-1.5}$), nonlinear-UCB ($\alpha=1$), and \texttt{ViOlin} ($\alpha_1=2,\alpha_2=640$). Red horizontal lines denote the true parameter values; one representative run is shown.}}
    \label{fig:param_trajectories_WS_logistic}
\end{figure}

\begin{figure}[h!] 
\centering
 {\resizebox{.87\textwidth}{!}{
\begin{tabular}{c}
{(a) \includegraphics[width=0.95\linewidth]{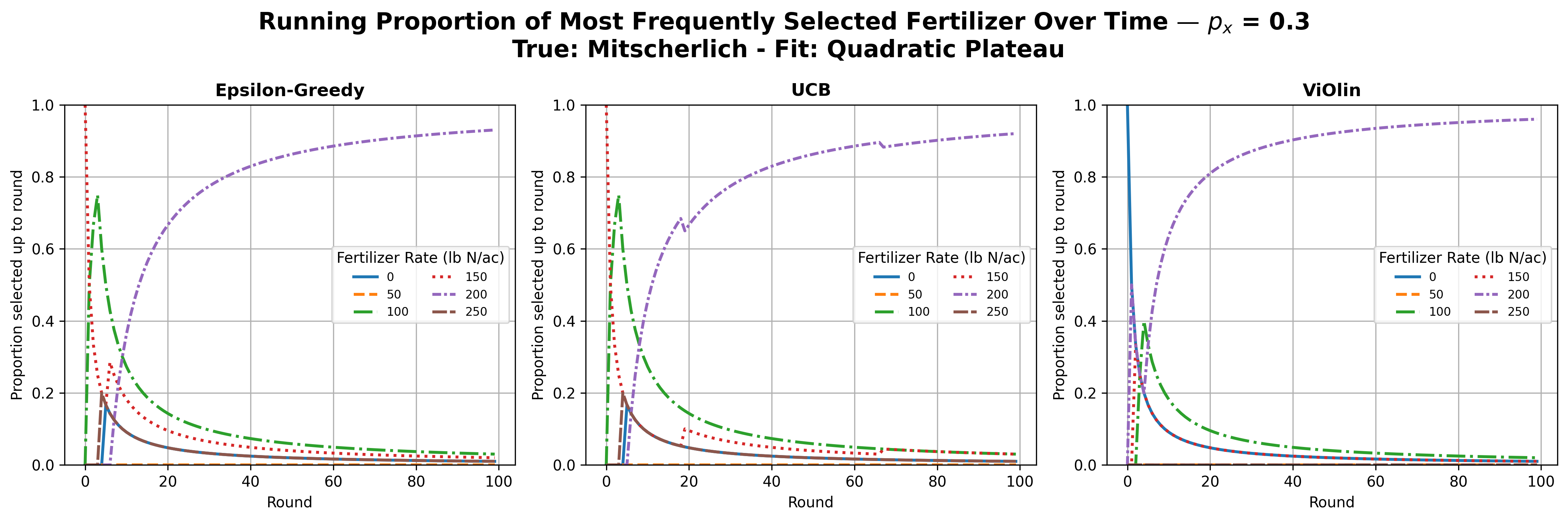}}\\
{(b)\includegraphics[width=0.95\textwidth]{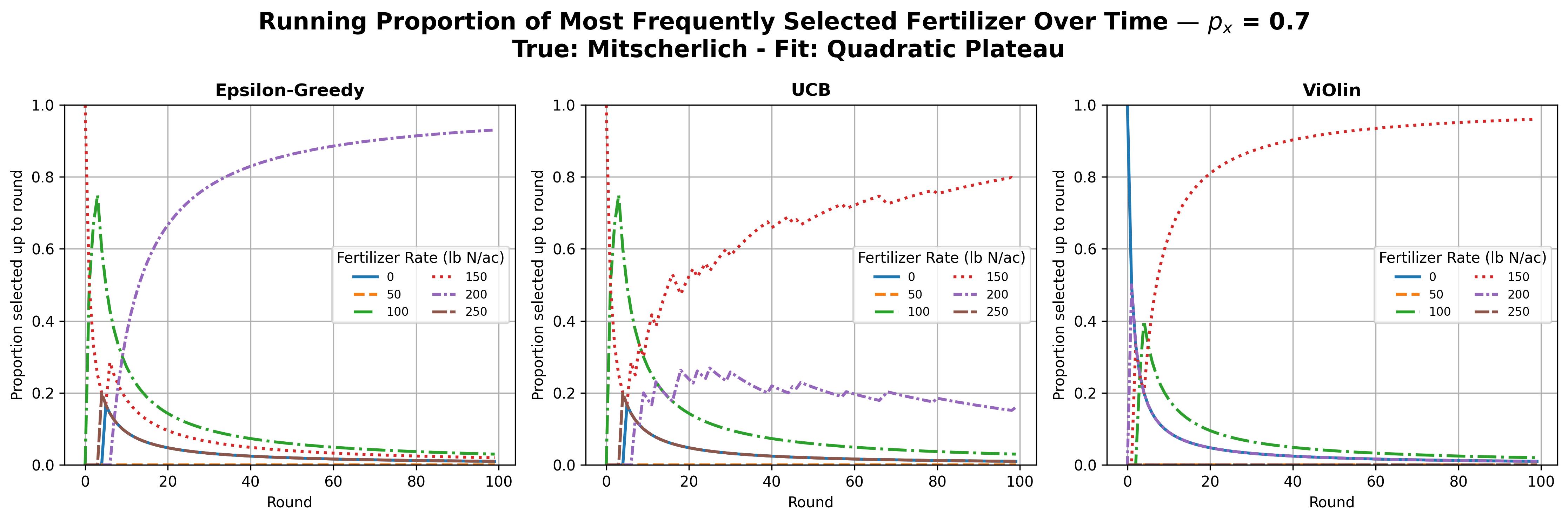}}
\end{tabular}}}
    \caption{\chg{\textbf{Fertilizer-rate choice dynamics under misspecification (true Mitscherlich; fitted quadratic-plateau).}
Running selection proportions for each nitrogen rate in $\mathcal{X}=\{0,50,\ldots,250\}$ lb N/ac over $T=100$ rounds (averaged over 10 replicates) for model-based $\epsilon$-greedy ($\epsilon_t=t^{-1.5}$), nonlinear-UCB ($\alpha=1$), and \texttt{ViOlin} ($\alpha_1=2,\alpha_2=640$).
Panels compare two fertilizer-price regimes: \textbf{(a)} $p_x=\$0.3$/lb N and \textbf{(b)} $p_x=\$0.7$/lb N. Higher fertilizer cost shifts the learned policies toward lower nitrogen rates.}}
      \label{fig:arms_bandit_MS_appendix}
\end{figure}
\paragraph{Misspecified setting: Additional results}
For the misspecified setting, for the same combination with the true model being Mitscherlich and fitted model being Quadratic plateau, we illustrate the effect of changing the level of fertilizer price from \$0.3 to \$0.7 per unit lb N. The same hyperparameter choices are made as in the misspecified setting of Section \ref{sec: simulation}. 

\begin{figure}[h!]
    \centering
    \includegraphics[width=0.4\linewidth]{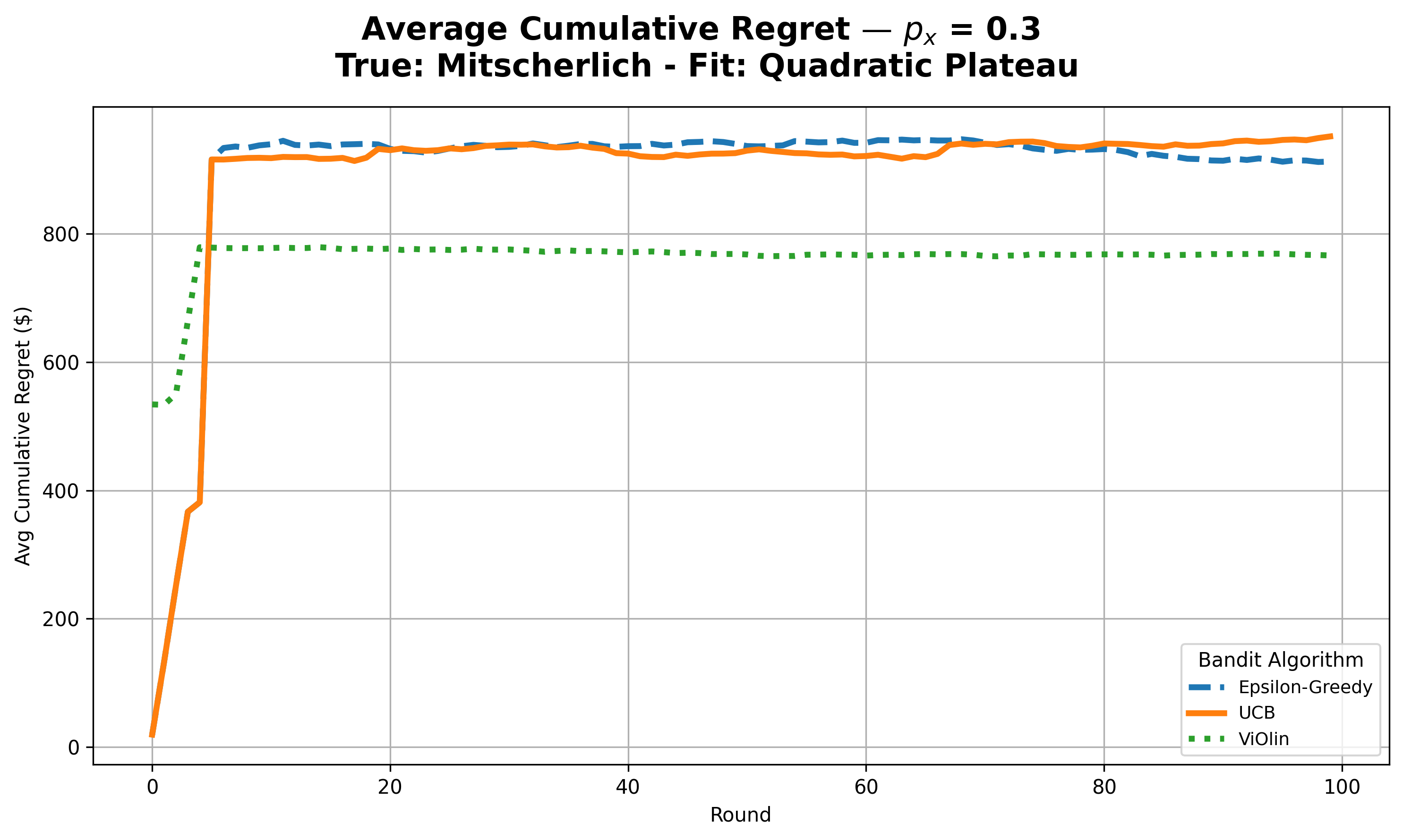}
       \includegraphics[width=0.4\linewidth]{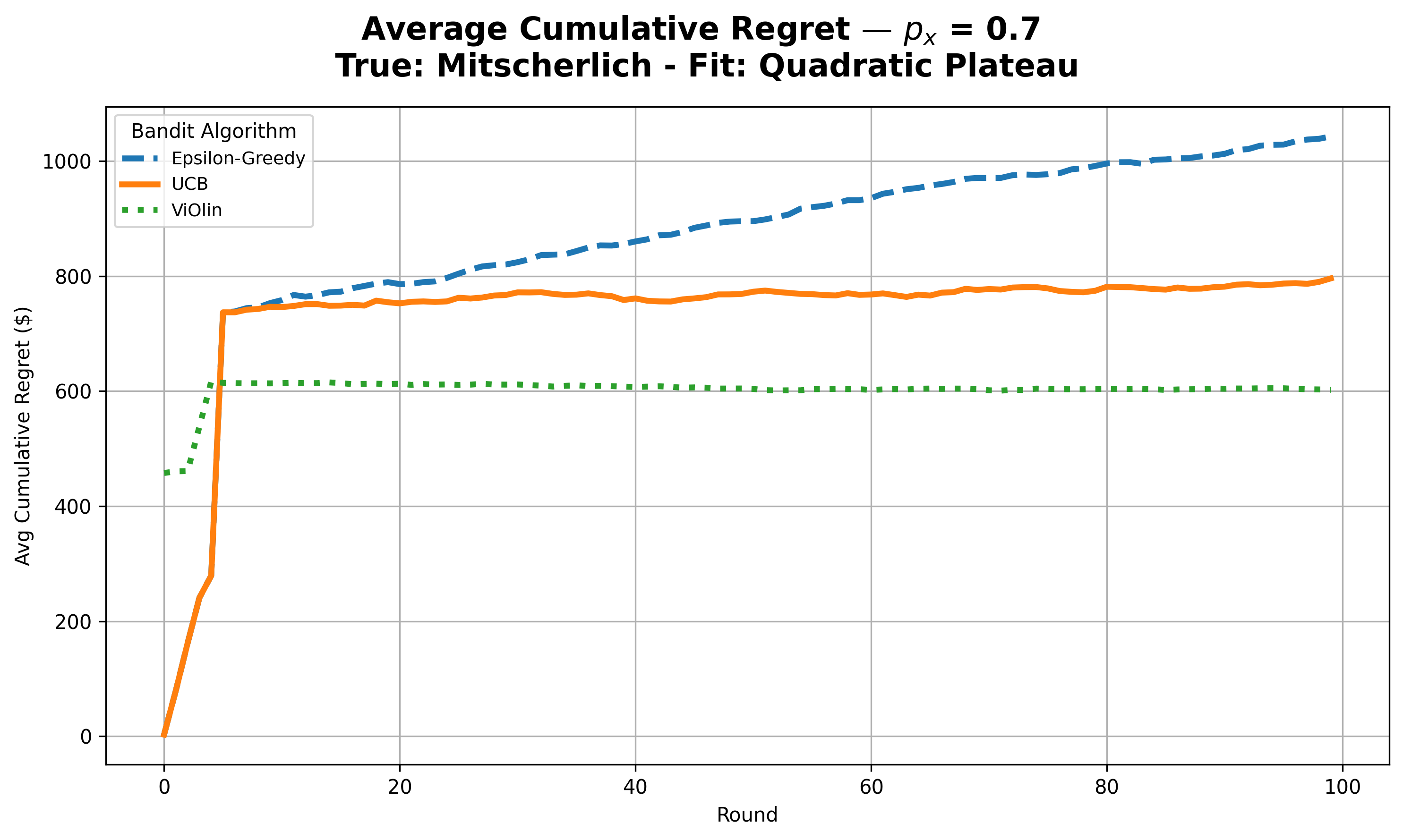}
    \caption{\chg{\textbf{Profit cumulative regret under model misspecification (true Mitscherlich; fitted quadratic-plateau).}
Average cumulative regret (in \$/ac) over $T=100$ rounds for model-based $\epsilon$-greedy ($\epsilon_t=t^{-1.5}$), nonlinear-UCB ($\alpha=1$), and \texttt{ViOlin} ($\alpha_1=2,\alpha_2=640$), averaged over 10 replicates.
Panels compare two fertilizer-price regimes: \textbf{left} $p_x=\$0.3$/lb N and \textbf{right} $p_x=\$0.7$/lb N.}}
    \label{fig:regret_MS_appendix}
\end{figure}

Figure \ref{fig:arms_bandit_MS_appendix} (a) corresponds to $p_X = 0.3$ while Figure \ref{fig:arms_bandit_MS_appendix} (b) corresponds to $p_X = 0.7$. Note that, both UCB and \texttt{ViOlin} capture that as the price increases the most frequently chosen arm over time becomes $x =150$ (brown dotted line) instead of $x=200$ (purple dashed dotted line). This sort of adaptive learning is a robustness check on how effective the decision-making is towards the goal of profit maximization. However, $\epsilon$-greedy failed to adapt to this change, perhaps due to a fast decaying $\epsilon_t$ which made it get stuck on a sub-optimal arm in the initial rounds. This is also reflected in the poor regret performance by $\epsilon$-greedy in  Figure \ref{fig:regret_MS_appendix}. 

\section{\chg{Data description}}
\label{app:data}

We complement the simulation study with an offline evaluation on a multi-site nitrogen-rate trial dataset for corn production in the U.S.\ Midwest \citep{ransom_data}. 
Our analysis uses the processed file \texttt{expanded\_data.csv}, which aggregates all available trials across sites and states over 2014--2016. We also use restricted subsets of this file for illustrative case studies, e.g., Urbana, IL.

\smallskip
\noindent\textbf{Experimental layout and key variables.}
The underlying field experiments follow a randomized complete block design (RCBD) within each site-year, with four spatial blocks (\texttt{Block} $\in \{1,2,3,4\}$) that are geographically close. 
Each plot receives a planting nitrogen rate \texttt{Plant\_N} (lb N/ac) taking values in $\{0,40,80,120,160,200,240,280\}$, and the response is recorded as yield \texttt{Yield\_Bu} (bu/ac). 
The dataset also includes an agronomically motivated baseline covariate \texttt{ExpectYield1}, representing the expected yield for that site based on prior yield history and growing conditions, and a site-level productivity label \texttt{Site\_Prod} indicating relative productivity (high vs.\ low) within each state. Table~\ref{tab:realdata_columns} lists the variables retained in \texttt{expanded\_data.csv} for the offline replay analysis.

\begin{table}[H]
\centering
\caption{\textbf{Column names in the processed field-trial dataset used for offline replay.}}
\label{tab:realdata_columns}
\begin{tabular}{ll}
\toprule
\textbf{Column} & \textbf{Description (brief)} \\
\midrule
\texttt{Trial\#}        & Trial identifier \\
\texttt{Year}           & Growing season year \\
\texttt{State}          & U.S. state code \\
\texttt{Site}           & Site/location name \\
\texttt{Site\_Prod}     & Site productivity class (e.g., low/high) \\
\texttt{Block}          & Block index within site-year (RCBD block) \\
\texttt{Plant\_N}       & Nitrogen application rate (lb N/ac) \\
\texttt{Yield\_Bu}      & Observed corn yield (bu/ac) \\
\texttt{ExpectYield1}   & Expected yield covariate (as provided in source data) \\
\bottomrule
\end{tabular}
\end{table}

\smallskip
\noindent\textbf{Low-productivity pooled subset (used in Section~\ref{sec:realdata}).}
To reduce heterogeneity and keep the evaluation aligned with a non-contextual bandit model, our pooled analysis restricts attention to site-years labeled \texttt{Site\_Prod = low}. 
Table~\ref{tab:lowprod_state_summary} summarizes the resulting subset by state. Here \#sites counts unique \texttt{(State, Site)} pairs, \#rounds counts unique \texttt{(State, Site, Year, Block)} decision rounds, and \#rows is the number of plot-level observations.

\begin{table}[H]
\centering
\caption{\textbf{Summary of the pooled low-productivity real-data subset (2014--2016).}
The pooled subset totals 7 states, 11 sites, 56 rounds, and 442 rows.}
\label{tab:lowprod_state_summary}
\vspace{0.25em}
\begin{tabular}{lrrr}
\toprule
State & \# Sites & \# Rounds & \# Rows \\
\midrule
IL & 2 & 12 & 93 \\
MN & 2 & 12 & 94 \\
IA & 2 & 8  & 64 \\
IN & 1 & 8  & 64 \\
WI & 2 & 8  & 63 \\
MO & 1 & 4  & 32 \\
NE & 1 & 4  & 32 \\
\midrule
\textbf{Total} & \textbf{11} & \textbf{56} & \textbf{442} \\
\bottomrule
\end{tabular}
\end{table}

\smallskip
\noindent\textbf{Price series for the profit objective.}
For the profit objective $\Pi(x)=p_yY(x)-p_xx$ used in the real-data replay, Table~\ref{tab:realdata_prices_urea} reports the year-specific corn price $p_y$ and nitrogen cost $p_x$. 
Corn price $p_y$ is the U.S.\ annual average price received (\$/bu) from USDA--NASS Crop Values (2014--2016). 
Urea prices are December Midwest retail prices (\$/ton); nitrogen cost is computed as
$p_x=(\$/\text{ton})/(2000\times 0.46)$ using urea's 46\% N analysis.

\begin{table}[ht]
\centering
\caption{\textbf{Price series used for profit in the real-data offline replay (urea).}}
\label{tab:realdata_prices_urea}
\begin{tabular}{lccc}
\toprule
Year & Corn price $p_y$ (\$/bu) & Urea price (\$/ton) & Nitrogen cost $p_x$ (\$/lb N) \\
\midrule
2014 & 3.70 & 485 & 0.527 \\
2015 & 3.61 & 431 & 0.468 \\
2016 & 3.40 & 343 & 0.373 \\
\bottomrule
\end{tabular}
\end{table}

\end{document}